\documentclass[12pt]{article}

\usepackage[title]{appendix}
\usepackage[english]{babel}
\usepackage[utf8x]{inputenc}
\usepackage[T1]{fontenc}
\usepackage{pdfpages}
\usepackage[authoryear]{natbib}
\usepackage{caption}
\usepackage{breqn}
\usepackage{subcaption}
\usepackage{amsmath, amssymb}
\usepackage{amsfonts, bm}
\usepackage{amsthm}
\usepackage{bbold}
\usepackage{graphicx}
\usepackage{booktabs}
 \usepackage{multirow}
\usepackage{enumitem}
\usepackage{soul}
\usepackage[linesnumbered,ruled,vlined]{algorithm2e}
\usepackage[colorinlistoftodos]{todonotes}
\usepackage[colorlinks=true, allcolors=blue]{hyperref}
\usepackage{comment}
\usepackage{url}

\newtheorem{theorem}{Theorem}[section]
\newtheorem{corollary}{Corollary}[theorem]
\newtheorem{proposition}{Proposition}[section]
\newtheorem{lemma}[theorem]{Lemma}
\newtheorem{remark}{Remark}
\newtheorem{assumption}{Assumption}
 
\newcommand{\diag}{\operatorname{diag}}
\newcommand{\tr}{\operatorname{tr}}
\newcommand{\mb}{\mathbf}
 
\newcommand{\bb}[1]{\mathbb{#1}}

\newcommand{\nrm}[2]{\|#1\|_{#2}}
\newcommand{\lrnrm}[2]{\left\|#1\right\|_{#2}}
\newcommand{\lrp}[1]{\left(#1 \right)}
\newcommand{\brcs}[1]{\left\{#1 \right\}}
\newcommand{\bxs}[1]{\left[#1 \right]}
\newcommand{\vrt}[1]{\left|#1 \right|}
\newcommand{\argmin}{{\arg\min}}
\newcommand{\argmax}{{\arg\max}}

\newcommand{\oraet}{\b\tht^o}
\newcommand{\Oraet}{\Tht^v}

\def\b{\bm}
\def\r{\rm}
\def\h{\hat}
\def\t{\tilde}
\def\c{\mathcal}
\def\cO{\c O}
\def\P{\bb P}
\def\E{\bb E}
\def\R{\bb R}
 
\def\eps{\varepsilon}
\def\vrp{\varphi}
\def\dlt{\delta}
\def\Dlt{\Delta}
\def\tht{\theta}
\def\Tht{\Theta}
 
\def\blt{\bullet}
\def\Vol{\operatorname{Vol}}
\def\Var{\operatorname{Var}}

\def\diag{\operatorname{diag}}
\def\tr{\operatorname{tr}}

\def\un{\underline{n}}
\def\bn{\bar n}
\def\us{\underline{s}}
\def\bs{\bar s}
 
\def\GSk{{k,G}}
\def\HSk{{k,H}}

\def\GoSk{{k,G_0}}
 
\newcommand{\bx}{\h{\b\tht}}
\newcommand{\tx}{\underline{\bx}}

\newcommand{\anon}{1}
 
\begin{document}
 
\def\spacingset#1{\renewcommand{\baselinestretch}%
{#1}\small\normalsize} \spacingset{1}

\if1\anon
{
  \title{\bf Personalized Federated Learning via Variance-Aware Nonparametric Empirical Bayes}
  \author{\\Jae Ho Chang,  Arnab Auddy, and Subhadeep Paul \\ \\
    \textit{Department of Statistics, The Ohio State University}
}
    \date{}
  \maketitle
} \fi
 
\if0\anon
{
  \bigskip
  \bigskip
  \bigskip
  \begin{center}
    {\LARGE\bf Personalized Federated Learning via Variance-Aware Nonparametric Empirical Bayes}
  \end{center}
   \date{}
  \medskip
} \fi
 
\medskip
\begin{abstract}
We develop a new approach to Personalized Federated Learning across heterogeneous clients using Nonparametric Empirical Bayes (NPEB). Leveraging the asymptotic normality of local parameter estimates obtained from Empirical Risk Minimization or M-estimation, our method formulates these estimates as noisy observations to estimate an unknown shared prior via Nonparametric Maximum Likelihood. A key challenge in applying NPEB in this setting is that existing approaches assume known fixed variances, which is not true in practice. To address this, we introduce a Variance-Aware Nonparametric Empirical Bayes (VANEB) framework that leverages the parameter-dependent asymptotic variance of local M-estimators. A key technical contribution is a generalized Tweedie's formula for this heteroskedastic setting. We then establish non-asymptotic error rates for density estimation in the average squared Hellinger distance and derive an oracle denoising inequality that provides error bounds for our estimator. While our theoretical guarantees are rooted in the asymptotic regime of M-estimators, we empirically explore heuristic extensions of VANEB to modern federated learning settings involving Deep Neural Networks (DNNs). For DNNs, we propose VANEB-head and VANEB-FT, which personalize the last fully connected layer via an NPEB step using an approximate diagonal variance estimator. We show that our method has strong performance on popular vision datasets MNIST and CIFAR-10, using a convolutional neural network architecture.
\end{abstract}
\noindent%
 
\spacingset{1.4}

\section{Introduction}

Modern machine learning and artificial intelligence applications increasingly rely on data distributed across large numbers of devices or institutions, where ownership, privacy, or communication constraints may preclude centralized data collection. Federated Learning (FL) provides a framework for training models collaboratively in such settings by keeping raw data local to participating ``edge'' databases or devices and exchanging only information required for model learning \citep{mcmahan2017learning,mcmahan2017communication,wei2020federated,karimireddy2020scaffold}. The seminal work of \cite{mcmahan2017communication} introduced the Federated Averaging (FedAvg) algorithm, which periodically aggregates locally trained model updates at a central server. Since then, extensive research has sought to improve optimization efficiency and robustness under communication and resource constraints, as well as data and system heterogeneity \citep{kairouz2021advances, li2020federated, karimireddy2020scaffold, auddy2026statistical}. More recently, federated learning methods have been developed for Large Language Models (LLMs) \citep{yao2024federated, zhuang2023foundation}, including parameter-efficient fine-tuning approaches based on the Low-Rank Adaptation (LoRA) framework \citep{hu2022lora,babakniya2023slora}.

In real-world FL deployments, clients often differ substantially in their underlying data distributions, and this statistical heterogeneity can significantly impair the performance of a single globally shared model \citep{li2020federated,karimireddy2020scaffold}. To address this, Personalized Federated Learning (PFL) seeks to train distinct, customized models for each client while still leveraging the shared knowledge of the federation \citep{karimireddy2020scaffold,tan2022towards,fallah2020personalized,jiang2019improving,deng2020adaptive,mansour2020three,ozkara2023statistical,kotelevskii2022fedpop,dinh2020personalized,collins2021exploiting,oh2021fedbabu}. These methods seek to improve predictive performance for each client under its local data distribution, while exploiting shared structure across clients to facilitate knowledge transfer among clients.
Existing personalization strategies include meta-learning approaches \citep{jiang2019improving,fallah2020personalized}, multi-task learning formulations \citep{smith2017federated,dinh2020personalized}, representation learning based methods \citep{collins2021exploiting,oh2021fedbabu,liang2020think,arivazhagan2019federated}, and Bayesian methods \citep{ozkara2023statistical,kotelevskii2022fedpop,makhija2024bayesian}, among others.

In this paper, we cast personalized federated learning as a compound decision problem governed by empirical Bayes (EB) principles \citep{robbins1956empirical}. Instead of transmitting raw data, each client transmits a local estimate to the central server. The server models these local summary statistics as noisy realizations of latent client-specific parameters, and the shared prior is estimated from the collection of local estimators. Then, our goal is to estimate the prior and denoise the local client estimates by EB \citep{efron2014two}. We propose to employ Nonparametric empirical Bayes (NPEB) to avoid misspecification of the prior by learning a prior directly from the data via nonparametric maximum likelihood (NPMLE) \citep{saha2020nonparametric, soloff2025multivariate}.

A fundamental limitation of existing NPEB theory and algorithms \citep{jiang2009general, saha2020nonparametric, soloff2025multivariate} is their reliance on the assumption of \textit{known} and parameter-independent covariance structures. In the context of FL, this premise is frequently violated even in the asymptotic regime. For example, for M-estimators and generalized linear models (GLMs), the covariance of a client’s local estimator is unknown and is often a function of the client parameter that is being estimated (e.g., the inverse of the Fisher information or a ``sandwich'' formula). Standard NPEB methods in \cite{saha2020nonparametric} and \cite{soloff2025multivariate}, which require known covariance matrices, are therefore not applicable in Federated Learning. 
Parameter-dependent covariance structures also change the structure of the NPMLE. The problem remains convex in the mixing distribution, but the stationarity condition for the support points requires additional terms involving the gradients of the variance functions (see Section \ref{sec:vaneb}). As a consequence, the support no longer lies on the ridgeline manifold of \cite{ray2005topography}, and the closed-form updates available for fixed covariances do not apply.

We address this gap with the \emph{Variance-Aware Nonparametric EB} (VANEB) method. We incorporate the parameter-dependent covariance matrix into the likelihood model and derive the corresponding \textit{generalized Tweedie's formula}. We derive oracle denoising inequalities that bound the excess risk of the VANEB estimator relative to the oracle posterior mean, which requires knowledge of the true unknown prior. For computation, we propose a pseudo-EM algorithm that alternates between solving fixed-covariance NPMLE subproblems and updating the covariances at the newly estimated support points.

Our framework considers two distinct sources of heterogeneity. The first is cross-client heterogeneity, which is the variation in the latent parameters $\{\b\tht_k\}_{k=1}^K$ across clients, encoded in the unknown prior $G_0$. This is the usual FL heterogeneity that motivates personalization: different clients have different true parameters, and a single global model is suboptimal. Yet, information sharing is possible if one can identify other similar clients, which requires knowing $G_0$.
The second is heteroskedasticity within each client's model: the local covariance $\Sigma_k(\b\tht)$ depends on the parameter $\b\tht$ itself.  

VANEB addresses both sources of heterogeneity simultaneously by learning $G_0$ via the NPMLE and using it to estimate client-specific posterior means. Our theory takes the Gaussian approximation of the M-estimators as the working model, which is appropriate when the local sample sizes $n_k$ are moderate to large. Within this regime, the main results are proven under a diagonal heteroskedastic Gaussian likelihood, which keeps both the variance-weighted Tweedie representation and the Hellinger convergence analysis tractable. Algorithmically, however, the same pseudo-EM template extends to a full covariance structure (not necessarily diagonal), and we study that broader setting empirically through simulation and data applications.

We have two main contributions, one each to the literature on empirical Bayes methods and personalized federated learning.
\begin{enumerate}[leftmargin=*]
    \item \textbf{Unknown Heteroskedastic Variance Empirical Bayes formulation.} We derive a generalized Tweedie formula for parameter-dependent covariance and identify the corresponding oracle Bayes rule through variance-weighted marginal densities. We establish non-asymptotic convergence rates for the NPMLE in the average squared Hellinger distance. Furthermore, we derive an oracle denoising inequality, which provides a bound on the excess mean squared error of our estimator relative to the infeasible oracle Bayes estimator that has access to the true latent prior $G_0$. We develop a pseudo-EM algorithm that iteratively solves fixed-covariance NPMLE subproblems while updating covariances at the new support locations. We give conditions for the convergence of its iterates and bound the distance between its fixed points and the critical points of the full heteroskedastic score (Lemma \ref{lem-vaneb-convergence}). Our results therefore relax assumptions and make NPEB methods applicable to a broad range of new problems.
    
\item \textbf{Personalized federated learning for modern AI models.}
We translate the VANEB compound-decision framework into practical personalization methods for deep neural networks. For vision models, we introduce \emph{VANEB-Head} and \emph{VANEB-FT}, which iteratively combines a FedAvg-trained body representation with variance-aware NPEB shrinkage of client-specific classification heads. These constructions preserve the body--head decompositions used by several modern personalized FL methods \citep{collins2021exploiting,arivazhagan2019federated}. However, they replace either independent local adaptation (fine-tuning) or complete parameter averaging across clients with uncertainty-aware, data-adaptive sharing across clients.

We evaluate VANEB against a wide range of state of the art methods including, (1) FedAvg \citep{mcmahan2017communication}, (2) FedAvg followed by finetuning, (3) FedRep \citep{collins2021exploiting}, (4) FedBABU \citep{oh2021fedbabu}, (5) FedPer \citep{arivazhagan2019federated}, (6) LG-FedAvg \citep{liang2020think}, (7) Per-FedAvg \citep{fallah2020personalized}, (8) SCAFFOLD \citep{karimireddy2020scaffold}, (9) Ditto \citep{li2021ditto}, and (10) local-only training, under label and covariate heterogeneity on MNIST and CIFAR-10. For evaluation, we consider both client-matched tests, which draw the test set to match the client's observed label distribution, and balanced tests, which evaluate whether a personalized model retains predictive ability across the complete label space. VANEB-FT is particularly strong under balanced evaluation: it achieves the highest accuracy across all 13 reported MNIST and CIFAR-10 settings. Furthermore, it is the only personalized method that outperforms FedAvg across all balanced and matched MNIST and CIFAR-10 experiments. These results indicate that empirical-Bayes sharing can provide personalization without sacrificing the broad predictive coverage learned from the federation.
\end{enumerate}

\subsection{Comparisons to Related Work}
Beyond the NPEB literature discussed above, VANEB is related to
parametric Bayesian approaches to PFL,
distributed statistical inference, and representation-based PFL methods.

\paragraph{Bayesian Personalized Learning.}
Recent Bayesian PFL methods also formulate personalization as a hierarchical inference problem. \cite{ozkara2023statistical} studies a homoskedastic Gaussian observation model with a single Gaussian prior, while FedPop \citep{kotelevskii2022fedpop} and FedGMM \citep{wu2023personalized} use more flexible parametric priors, including Gaussian mixtures and energy-based models. These approaches generally assume parameter-independent noise and require a parametric specification of the client population, such as the number and form of mixture components. VANEB instead estimates the latent prior nonparametrically and explicitly permits the covariance of each local estimator to depend on the unknown client parameter. It therefore accommodates multimodal or irregular client populations without prespecifying their structure. Our results in Theorem \ref{thm-denoising} show that this flexibility retains the benefit of pooling across clients: the excess risk relative to the oracle Bayes rule is of order $(nK)^{-1}$ up to polylogarithmic factors in $K$.

\paragraph{Distributed Statistical Inference.} Related work on distributed statistical inference studies $M$-estimation under
heterogeneous client distributions while targeting a common parameter \citep{gu2023distributed,gu2024statistical}. VANEB instead targets \textit{personalized} inference under heterogeneity. It treats the client-specific parameters as heterogeneous draws from an unknown prior $G_0$, uses the collection of noisy local M-estimators $\bx_k$ to learn $G_0$ and
estimate each client-specific parameter. Thus, our target is
personalization instead of efficient distributed estimation of a single consensus parameter.

\paragraph{Personalized federated learning for deep learning.}
VANEB is closely related to representation-based PFL methods that divide a neural network into shared and client-specific components. FedPer \citep{arivazhagan2019federated}, FedRep \citep{collins2021exploiting}, LG-FedAvg \citep{liang2020think}, and FedBABU \citep{oh2021fedbabu} differ in which body or head parameters are shared and whether the personalized head is trained during federation or adapted locally afterward. VANEB uses the same body-head decomposition but introduces a different sharing mechanism: it treats locally estimated heads as noisy observations from an unknown population distribution and replaces either independent local fitting or direct averaging with variance-aware NPEB shrinkage.
We empirically compare against these representation-based methods, as well as several other PFL methods, and show strong performance of VANEB in our experiments. 

\section{Framework and Model}\label{sec:model}
In many federated systems, the server cannot access raw client data and instead receives only low-dimensional summaries or estimators computed locally. This makes personalization a natural distributed inference problem. The server should combine ``noisy'' client estimates in a way that borrows information across clients without forcing them onto a single global estimate. Empirical Bayes provides exactly this viewpoint. It treats latent client parameters as draws from a common prior and interprets the server's task as estimating that prior from noisy observations, i.e., a compound decision problem in the sense of \cite{robbins1956empirical, efron2014two}.

In EB, the amount of shrinkage is automatically matched to each client's uncertainty. If the prior is nearly degenerate, the personalized estimator collapses toward a common global signal; if the prior is diffuse, shrinkage disappears, and the method behaves like local training. In contrast to parametric hierarchical models, NPEB allows the data to determine the shape of latent heterogeneity, which is important in federated settings where client populations may be multi-modal, skewed, or concentrated near lower-dimensional structures.

To connect this perspective to FL, we consider the following setup. 
Suppose client $k$ observes i.i.d. local data $\{\mb z_{k,i}\}_{i=1}^{n_k}$ from a marginal law $F_{\mb z_k}(\mb z)$ and and define a local parameter as $\b\tht_k:=\arg\min_{\mb u\in\R^d}\E_{\mb z_{k,1}}\rho(\mb z_{k,1};\mb u)$ for some function $\rho:\R^d\times\R^d\to\R$.
Then, $k$ computes an $M$-estimator
\[
\bx_k \in \argmin_{\b\tht\in\R^d} \frac{1}{n_k}\sum_{i=1}^{n_k} \rho(\mb z_{k,i}; \b\tht).
\]
Under standard regularity conditions, local empirical processes yield asymptotic normality \citep{van2000asymptotic,shao2022berry}:
\[
\sqrt{n_k}(\bx_k-\b\tht_k)\overset{n_k\to\infty}{\Rightarrow}{\c N}(\b0_d,\Sigma_k(\b\tht_k))
\]
with asymptotic covariance
\[
\Sigma_k(\b\tht) := \mb H_k(\b\tht)^{-1} \mb J_k(\b\tht) \mb H_k(\b\tht)^{-1},
\]
with Hessian $\mb H_k(\b\tht) = \E[\nabla^2\rho(\mb z_{k,1};\b\tht)]$ and $\mb J_k(\b\tht) = \E[\nabla\rho(\mb z_{k,1};\b\tht)\nabla\rho(\mb z_{k,1};\b\tht)^\top]$.
Berry-Esseen-type bounds for M-estimators show that the Gaussian approximation error decays at the near-$n_k^{-1/2}$ rate under moment conditions \citep{shao2022berry}. We therefore treat $\bx_k$ as approximately ${\c N}(\b\tht_k,\frac{\Sigma_k(\b\tht_k)}{n_k})$, a noisy observation of the client-specific parameter $\b\tht_k$. This characterization connects FL with the EB framework.

However, standard NPEB methods assume $\Sigma_k$ is known and independent of the parameter \citep{soloff2025multivariate}. Neither condition holds here: the covariance matrices are unknown, and in generalized linear models, $M$-estimation, and other curved models, the asymptotic covariance varies with $\b\tht_k$. A fixed-variance NPEB procedure then faces two problems. First, it is entirely unclear how to provide or obtain the $\Sigma_k$ values, and second, the NPEB procedure then solves a misspecified denoising problem. It misweights the likelihood contributions of clients whose plug-in variance differs from $\Sigma_k(\b\tht_k)$, and the estimated support of the prior is distorted accordingly. On the other hand, as the formulation above shows, in many common models the covariance is a known function of the unknown parameter (e.g., the inverse Fisher information), so each client can report the function $\Sigma_k(\cdot)$ rather than a single matrix. What is needed is a method that handles covariances that are functions of the parameter.

This formulation encompasses a broad class of federated learning problems: least squares regression, generalized linear models, and robust M-estimators with smooth loss (e.g., Huber). In each case, the asymptotic covariance matrix $\Sigma_k(\b\tht)$ depends on $\b\tht$ through the model structure and data distribution, inducing heteroskedasticity that VANEB exploits. For instance, in Poisson regression with 
\[\rho(y,\mb z;\b\tht) = -y\langle\mb z,\b\tht\rangle + \exp(\langle\mb z,\b\tht\rangle),\]
the Fisher information depends on $\b\tht$ exponentially through the moment generating function of the covariates, so the variance varies by orders of magnitude across the parameter space (for a correctly specified model): 
\[\Sigma_k(\b\tht) = {\c I}_k(\b\tht)^{-1} = \E\lrp{e^{\langle\mb z_{k,1},\b\tht\rangle}\,\mb z_{k,1}{\mb z_{k,1}}^\top}^{-1},
\]
where the expectation is over the marginal covariate distribution. Similarly, in logistic regression with 
\[\rho(y,\mb z;\b\tht)=-y\langle\mb z,\b\tht\rangle + \log(1+\exp(\langle\mb z,\b\tht\rangle)),
\]
the Fisher information yields
\[\Sigma_k(\b\tht) = \E\lrp{\pi(\langle\mb z_{k,1},\b\tht\rangle)(1-\pi(\langle\mb z_{k,1},\b\tht\rangle))\,\mb z_{k,1}{\mb z_{k,1}}^\top}^{-1},\]
where $\pi(u)=e^u/(1+e^u)$. Here, the variance inflates when predicted probabilities approach $0$ or $1$, i.e., when $|\langle\mb z,\b\tht\rangle|$ is large. We do note that, in some cases, such as GLM, the covariance function may depend on unknown covariate distributions. In such cases, the clients will report estimated covariance functions based on their own observed data. We ignore the added uncertainty due to this here.

\subsection{Model}
We now state the statistical model studied in this paper. Let $\b\tht_1,...,\b\tht_K\in\R^d$ denote latent client parameters drawn i.i.d. from an unknown measure $G_0$ on $\R^d$, and suppose client $k$ reports a local estimator $\bx_k$. Motivated by the asymptotic normality of M-estimators, we adopt the following heteroskedastic Gaussian working model for our EB framework:
\begin{align}\label{eq:like}
    \sqrt{n_k}(\bx_k-\b\tht_k)\;|\;\b\tht_k~\sim~\c N(\b0_d,\Sigma_k(\b\tht_k)),\quad k=1,...,K
\end{align}
for a known covariance function $\Sigma_k:\R^{d}\to\R^{d\times d}$. Here, $n_k$ denotes the local sample size, and $n_k^{-1/2}$ represents the local consistency rate. The model thus captures differences across clients in both sample size and covariance structure.

To simplify notation, we introduce the rescaled local estimates $\tx_k:=\sqrt{n_k}\bx_k$. The likelihood of $\tx_k$ given an arbitrary parameter $\b\tht\in\R^d$ is then
\[\vrp^{(k)}(\tx_k;\b\tht):=|2\pi\Sigma_k(\b\tht)|^{-1/2}\exp\brcs{-\frac{1}{2}\nrm{\Sigma_k(\b\tht)^{-1/2}(\tx_k-\sqrt{n_k}\b\tht)}{}^2}.\]
Let $\un := \min_{k\le K}n_k$ and $\bn := \max_{k\le K}n_k$ denote the minimum and maximum local sample sizes over $K$ clients, respectively.
\begin{assumption}[Regularity]\label{ass:reg}
    For each client $k$, the covariance function $\Sigma_k:\R^d\to\R^{d\times d}$ is known and diagonal:
    \[
        \Sigma_k(\b\tht)=\diag\bigl(\sigma_{k,1}(\b\tht)^2,\ldots,\sigma_{k,d}(\b\tht)^2\bigr).
    \]
    There exist constants $0<\us<\bs<\infty$ and $L_\Sigma<\infty$ such that for all $\b\tht,\b\tht'\in\R^d$ and $k\le K$,
    \begin{align*}
        \us\mb I_d \preceq \Sigma_k(\b\tht) \preceq \bs\mb I_d,
        \qquad
        \nrm{\Sigma_k(\b\tht)-\Sigma_k(\b\tht')}{} \le L_\Sigma\nrm{\b\tht-\b\tht'}{},
    \end{align*}
    and each variance coordinate $\sigma_{k,i}(\cdot)^2$ is continuously differentiable with
    \[
        \sup_{\b\tht\in\R^d}\max_{i\le d,\,k\le K}\nrm{\nabla\log \sigma_{k,i}(\b\tht)^2}{} \le L_\Sigma.
    \]
    In addition, the client sample sizes are of the same order of magnitude, with $\bn/\un\to c_n$ for some constant $c_n\ge1$ as $K\to\infty$.
\end{assumption}
Here, $\mb A\preceq\mb B$ means $\mb B-\mb A$ is positive semi-definite. Assumption \ref{ass:reg} collects all covariance conditions used later in the paper: uniform ellipticity condition controls degeneracy and Gaussian tails, diagonality identifies the theory regime, and the Lipschitz/$C^1$ bounds provide the smoothness needed in the pseudo-EM and van Trees arguments. The bounded ratio $\bn/\un$ prevents any single client from dominating the aggregation.

Our main analysis is carried out under a diagonal heteroskedastic covariance model $\Sigma_k(\b\tht)$. This approximation also reduces the per-client communication from $d(d+1)/2$ covariance entries to $d$ variance coordinates \citep{jhunjhunwala2024fedfisher}. This regime is natural whenever the loss $\rho$ is approximately coordinate-separable or when a diagonal covariance approximation is already used for communication and computation. Then, Assumption \ref{ass:reg} implies
\[
    \us\le\inf_{\b\tht\in\R^d}\min_{i\le d,\,k\le K}\sigma_{k,i}(\b\tht)^2
    \le\sup_{\b\tht\in\R^d}\max_{i\le d,\,k\le K}\sigma_{k,i}(\b\tht)^2
    \le\bs.
\]

\subsection{Generalized Tweedie Formula and VANEB Estimator}
Our aim is to obtain personalized estimators $\bx_1,...,\bx_K$ that minimize the Bayes risk under the unknown prior,
\[\frac{1}{K}\sum_{k=1}^K\E\nrm{\bx_k-\b\tht_k}{}^2\]
and the Bayes rule for this risk is the posterior mean $\E(\b\tht_k\,|\,\tx_k)$ for each client. Under our Gaussian likelihood setup, the posterior mean is given by Tweedie's formula, provided the covariance matrix $\Sigma_k$ is known and not a function of the parameter \citep{efron2014two}. This is because $\Sigma_k$ can be pulled outside the integral when taking the gradient of the density function $f_{k,G}$, and one recovers the formula of \cite{soloff2025multivariate},
$$\E(\b\tht_k\,|\,\tx_k) = \frac{\tx_k}{\sqrt{n_k}} + \frac{\Sigma_k}{\sqrt{n_k}}\nabla\log f_{k,G_0}(\tx_k).$$

However, when $\Sigma_k = \Sigma_k(\b\tht)$ depends on the parameter, $\Sigma_k(\b\tht)$ cannot be pulled outside the integral, and this identity fails. We therefore derive a version of Tweedie's formula for parameter-dependent covariances. The key is to express the marginal density $f_{k,G}(\tx_k)$ and the variance-weighted score $\mb q_{k,G}(\tx_k)$ as functionals of the likelihood integrated against the prior distribution $G_0$ as follows,
\[
f_\GoSk(\tx_k)=\int\vrp^{(k)}(\tx_k;\b\tht)\,G_0({\r d}\b\tht), \quad \mb q_{k,G_0}(\tx_k) := \int\Sigma_k(\b\tht)\nabla\vrp^{(k)}(\tx_k;\b\tht)\,G_0({\r d}\b\tht).
\]
By the diagonal structure, we can write
\[\mb q_{k,G}(\tx_k)=\int\begin{bmatrix}\sigma_{k,1}(\b\tht)^2\frac{\partial}{\partial \underline{\h\tht}_1}\vrp^{(k)}(\tx_k;\b\tht) \\ \vdots \\ \sigma_{k,d}(\b\tht)^2\frac{\partial}{\partial \underline{\h\tht}_d}\vrp^{(k)}(\tx_k;\b\tht)\end{bmatrix}G({\r d}\b\tht)=:\begin{bmatrix}
    \frac{\partial}{\partial \underline{\h\tht}_1}\t f_{k,G,1}(\tx_k)\\\vdots\\\frac{\partial}{\partial \underline{\h\tht}_d}\t f_{k,G,d}(\tx_k)
\end{bmatrix}\]
where $\t f_{k,G,i}(\tx_k):=\int\sigma_{k,i}(\b\tht)^2\vrp^{(k)}(\tx_k;\b\tht)G({\r d}\b\tht)$ represents the \emph{variance-weighted marginal density} for coordinate $i$. The factor $\sigma_{k,i}(\cdot)^2$ accounts for the parameter-dependent variance in coordinate $i$. In homoskedastic settings, this factor reduces to a constant, and the formula collapses to the classical Tweedie formula.  This decomposition allows the scaled score to be expressed as a gradient of component-specific functions.

By direct computation, we expand the score vector as:
\begin{align*}
    \mb q_{k,G_0}(\tx_k)=\int \vrp^{(k)}(\tx_k;\b\tht)\cdot(\sqrt{n_k}\b\tht-\tx_k)\,G_0({\r d}\b\tht).
\end{align*}
From this decomposition, we obtain a generalization of Tweedie's formula that accommodates heteroskedastic, parameter-dependent noise:
\[
\sqrt{n_k}\E(\b\tht_k|\tx_k)=\tx_k+\frac{\mb q_{k,G_0}(\tx_k)}{f_\GoSk(\tx_k)},
\]
which in the original scale becomes
\[
\oraet_k:=\E(\b\tht_k|\tx_k)=\frac{\tx_k}{\sqrt{n_k}}+\frac{\mb q_{k,G_0}(\tx_k)}{\sqrt{n_k}f_\GoSk(\tx_k)}.
\]
The quantity $\oraet_k$ is the \textit{oracle Bayes estimator} for client $k$ under the true prior $G_0$. Compared with the classical Tweedie formula, the correction term now involves variance-weighted densities $\t f_{k,G,i}$, whose coordinate-wise weights $\sigma_{k,i}(\b\tht)^2$ make the shrinkage adaptive to the local noise level of each client.

The oracle Bayes estimator above, $\oraet_k$, depends on the unknown prior $G_0$, so the server must estimate $G_0$ from the estimates $\{\tx_k\}_{k=1}^K$ and covariance functions $\{\Sigma_k(\cdot)\}_{k=1}^K$. We do this with an NPMLE. Given an estimated prior $\h G$, the resulting \textit{VANEB personalized estimator} becomes
\begin{equation}
    {\bx_k}^v:=\frac{\tx_k}{\sqrt{n_k}}+\frac{\mb q_{k,\h G}(\tx_k)}{\sqrt{n_k}f_{k,\h G}(\tx_k)},
    \label{vanebest}
\end{equation}
where the NPMLE $\h G$ serves as a data-driven estimator for the unknown prior distribution $G_0$. We develop the NPMLE for our heteroskedastic setting and establish its convergence properties.

The correction term in \eqref{vanebest} weights coordinate $i$ by $\sigma_{k,i}(\cdot)^2$ through $\t f_{k,\h G,i}$. In a coordinate where the local variance is large, the likelihood is flat and the posterior mean moves substantially toward the regions where $\h G$ places its mass; where the local variance is small, the likelihood dominates, and the correction is small. The estimator therefore borrows more from the other clients exactly where the local estimate is least reliable, in the spirit of empirical Bayes shrinkage \citep{james1961estimation,stein1956inadmissibility,efron2012large}.

Compared with common FL optimizers, this update has a distinct role. FedAvg \citep{mcmahan2017communication} uses isotropic averaging of local steps, and FedProx/SCAFFOLD \citep{li2020federated,karimireddy2020scaffold} stabilize optimization drift through proximal regularization or control variates; all are primarily designed around iterative gradient optimization. In contrast, VANEB applies a one-shot empirical Bayes correction to the local estimates $\{\bx_k\}$ (which may be MLEs, M-estimators, or any asymptotically normal estimator; \cite{van2000asymptotic}), with the direction and magnitude of the correction determined by the client-specific covariance function $\Sigma_k(\cdot)$. Relative to personalized stochastic gradient descent methods (e.g., pFedMe/Per-FedAvg \citep{dinh2020personalized,fallah2020personalized}), VANEB shifts the onus from learning-rate tuning and repeated local iterations to estimating the shared prior $G_0$. It requires a single round of communication for the local estimates and a single server-side update for personalization, which is attractive when communication is expensive or local computation is constrained.

\section{NPMLE and VANEB Algorithm}\label{sec:npmle}
To estimate the unknown prior distribution $G_0$, we employ a nonparametric maximum likelihood approach \citep{koenker2014convex,saha2020nonparametric,soloff2025multivariate}. The NPMLE is defined as a solution $\h G$ to the following optimization problem:
\begin{align}
    \h G\in \argmax_{G\in{\c P}(\R^d)} \ell_K(G),~~\ell_K(G) := \frac{1}{K}\sum_{k=1}^K\log f_\GSk(\tx_k)
    \label{eq-NPMLE}
\end{align}
where ${\c P}(\R^d)$ is a set of all probability measures on $\R^d$. 
Here, $\ell_K(G)$ is the average log-marginal likelihood across all $K$ clients under the candidate prior distribution $G$.

\begin{lemma}\label{lem-characterization}
    Define the optimality gap as $D(\h G, \b\tht) := \frac{1}{K}\sum_{k=1}^K\frac{\vrp^{(k)}(\tx_k;\b\tht)}{f_{k,\h G}(\tx_k)} -1$. Then, any NPMLE, $\h G$ in \eqref{eq-NPMLE} satisfies $D(\h G, \b\tht) \le 0\text{ for all }\b\tht\in \R^d$ and is supported on $\c{Z} := \{\b\tht\,;\,D(\h G, \b\tht) = 0\}$, which are the global maximizers of the reweighted mixture density
    \begin{align*}
        \h\psi_K(\b\tht) := \sum_{k=1}^K \left(\frac{\h{L}_k^{-1}}{\sum_{i=1}^K\h{L}_i^{-1}}\right)\vrp^{(k)}(\tx_k;\b\tht),
    \end{align*}
    where $\h{L}_k=f_{k,\h G}(\tx_k)$ are the estimated marginal likelihoods.
\end{lemma}
In the fixed-covariance case, the reweighted mixture density $\h\psi_K$ admits a ridgeline manifold that reduces support estimation to a tractable low-dimensional problem \citep[Lemma 3]{soloff2025multivariate}. When $\Sigma_k$ depends on $\b\tht$, no analogous characterization is available: the stationarity equation acquires additional terms involving the gradients of the variance functions (see Section \ref{sec:vaneb}), and its solutions admit no closed form. Nonetheless, this suggests an alternating update scheme. We first freeze the covariances at given $\b\tht$ and use the resulting ridgeline problem as the computational primitive inside VANEB. The covariance matrices are then updated at the new atoms using the known covariance functions $\Sigma_k(\cdot)$. We describe the detailed algorithm next.

\subsection{Computation of the VANEB Estimator}\label{sec:vaneb}
We begin with the intuition in the classical setting where $\Sigma_k$ is independent of $\b\tht$, and the ridgeline manifold $\c M$ (see Lemma \ref{lem-ridgeline}) characterizes the NPMLE atoms exactly. In this case, the score equation is linear in the atom, and the M-step reduces to a closed-form weighted average. While we allow for covariance matrices to depend on the parameter $\b\tht$, the following lemma will be useful to update the atoms of the mixture density at a particular list of fixed covariance matrices.

\begin{lemma}[Fixed-covariance ridgeline, \citealp{ray2005topography}]\label{lem-ridgeline} 
    Let us fix the plug-in covariances $\mb S_k:=\Sigma_k(\bx_k)$ for all $\bx_k$, $k=1,\dots,K$; and define the fixed-covariance mixture density
    \[\bar\psi(\b\tht;\mb w):=\sum_{k=1}^Kw_k|2\pi\mb S_k|^{-1/2}\exp\brcs{\frac{-1}{2}\nrm{(\mb S_k)^{-1/2}(\tx_k-\sqrt{n_k}\b\tht)}{}^2}.\]
    Every critical point of $\bar\psi(\cdot;\mb w)$ lies in the compact set
    \begin{equation}
        \label{eq-ridgeline}
        \begin{aligned}
            \c M &:= \left\{\lrp{\sum_{k=1}^Kw_kn_k\mb S_k^{-1}}^{-1}\sum_{k=1}^Kw_kn_k\mb S_k^{-1}\bx_k\,;\,w_k\ge0,\;\sum_{k=1}^Kw_k=1\right\}
        \end{aligned}
    \end{equation}
    whose element can be represented using at most $d+1$ nonzero weights.
\end{lemma}

However, when $\Sigma_k(\b\tht)$ depends on $\b\tht$, the score of the mixture density under our setup $\psi(\b\tht)=\sum_{k=1}^Kw_k\vrp^{(k)}(\tx_k;\b\tht)$ becomes
\[\nabla\psi(\b\tht)=\sum_{k=1}^Kw_k\vrp^{(k)}(\tx_k;\b\tht)\,\mb s_k(\b\tht),\]
where
\[\mb s_k(\b\tht):=\nabla_{\b\tht}\log\vrp^{(k)}(\tx_k;\b\tht)=\sqrt{n_k}\Sigma_k(\b\tht)^{-1}(\tx_k-\sqrt{n_k}\b\tht)+\sum_{i=1}^d\brcs{\frac{(\underline{\h\tht}_{k,i}-\sqrt{n_k}\tht_i)^2}{\sigma_{k,i}(\b\tht)^2}-1}\frac{\nabla\sigma_{k,i}(\b\tht)}{\sigma_{k,i}(\b\tht)}\]
is the log-likelihood score of client $k$. So unlike \cite{ray2005topography,saha2020nonparametric,soloff2025multivariate}, the stationarity equation $\nabla\psi(\b\tht)=\b0_d$ involves the gradients of the variance functions $\sigma_{k,i}$, and no closed-form solution exists for the critical point(s).

Algorithm \ref{alg:vaneb} addresses this intractability by alternating between (i) solving a fixed-covariance NPMLE subproblem, where substituting $\mb S_k=\Sigma_k(\mb a_j)$ recovers the tractable ridgeline property, and (ii) updating the covariances evaluated at the newly estimated atom locations. This iterative scheme produces, at convergence, a fixed point of the coupled atom-covariance system and can be viewed as an approximation by a sequence of computable ridgeline manifolds.

\subsubsection{Algorithm}\label{sec:algorithm}

Our implementation uses a fixed-covariance conic solver only as an initialization, followed by a custom pseudo-EM routine in which
\begin{enumerate}[leftmargin=*]
    \item The E-step computes responsibilities using the current atom-dependent likelihoods,
    \item The M-step updates each atom by precision-weighted averaging, and
    \item The covariance is re-evaluated at the updated atoms before the next loop begins.
\end{enumerate}
Algorithm \ref{alg:vaneb} summarizes this procedure. Let us denote a point mass at $\b a\in\R^d$ by $\dlt_{\b a}$.
\begin{algorithm}[t]
    \DontPrintSemicolon
    \KwIn{Observations $\{\bx_k\}_{k=1}^K$; covariance function $\Sigma_k$; Pseudo-EM iterations $T$}
    \KwOut{Personalized estimates $\{{\bx_k}^v\}_{k=1}^K$}
    \BlankLine
    Initialize atoms $\{\mb a_j\}_{j=1}^m \gets \{\bx_k\}_{k=1}^K$ and covariance $\Sigma_{k,j} \gets \Sigma_k(\mb a_j)$\;
    Fit NPMLE weights: $\h{\mb w} \gets \argmax_{\mb w} \sum_{k=1}^K \log \sum_{j=1}^m w_j\,\vrp^{(k)}(\sqrt{n_k}\bx_k; \mb a_j)$
    \For{$j = 1, \ldots, m$}{
        \For{$t = 1, \ldots, T$}{
        E-step: $r_{kj} \propto \h w_j\,\vrp^{(k)}(\sqrt{n_k}\bx_k; \mb a_j)$\;
        M-step: $\mb a_j \gets \dfrac{\sum_{k=1}^K r_{kj}\,[\Sigma_{k,j}/n_k]^{-1}\bx_k}{\sum_{k=1}^K r_{kj}\,[\Sigma_{k,j}/n_k]^{-1}}$\;
        Recompute covariances: $\Sigma_{k,j} \gets \Sigma_k(\mb a_j)$ \quad \tcp{Atom-covariance coupling}
        Update weights: $\h w_j \gets \dfrac{1}{K}\sum_{k=1}^K r_{kj}$\;
        }
    }
    \Return ${\bx_k}^v \gets \sum_{j=1}^m r_{kj}\,\mb a_j$ for all $k$
    \caption{VANEB Estimator}
    \label{alg:vaneb}
\end{algorithm}
The final estimate produced by Algorithm \ref{alg:vaneb},
\[
{\bx_k}^v = \sum_{j=1}^m r_{kj}\,\mb a_j
= \frac{\sum_{j=1}^m \h w_j\,\mb a_j\,\vrp^{(k)}(\tx_k; \mb a_j)}{\sum_{j=1}^m \h w_j\,\vrp^{(k)}(\tx_k; \mb a_j)},
\]
is the posterior mean under the fitted discrete mixture prior $\h{G} = \sum_{j=1}^m \h w_j\,\dlt_{\mb a_j}$ and the heteroskedastic likelihood $\vrp^{(k)}(\tx_k; \b\tht)$. This is equivalent to applying the modified heteroskedastic Tweedie formula with the NPMLE $\h{G}$ in place of $G_0$:
\[
    {\bx_k}^v
    = \frac{\int \b\tht\,\vrp^{(k)}(\tx_k; \b\tht)\,{\r d}\h{G}(\b\tht)}{\int \vrp^{(k)}(\tx_k; \b\tht)\,{\r d}\h{G}(\b\tht)}
    = \frac{\tx_k}{\sqrt{n_k}}+\frac{\mb q_{k,\h G}(\tx_k)}{\sqrt{n_k}f_{k,\h G}(\tx_k)}.
\]
The pseudo-EM iterations refine the atom locations and weights, while the Tweedie representation provides a closed-form expression for the resulting posterior mean; this is the representation analyzed in Section \ref{sec:theory}.
\subsubsection{Computation}
Following \cite{koenker2014convex}, we restrict the NPMLE to a support-constrained version. For any non-empty, closed set $\c A \subseteq \R^d$, define the support-constrained NPMLE as
\begin{align}
    \h G^{\c A} &\in \argmax_{G \in {\c P}(\c A)} \ell_K(G), \label{eq-discretized-npmle-defn}
\end{align}
where $\c P(\c A)$ denotes the collection of probability measures supported on $\c A$. By construction, $\h G = \h G^{\R^d}$.

For computation, the key tuning parameter is the grid resolution $\dlt$. As in \cite{soloff2025multivariate}, we discretize the ridgeline region $\c M$ by covering it with axis-aligned hypercubes:
\[{\c H} = \left\{\mb H_j + [-\dlt/2, \dlt/2]^d\,;\,j\in\{1,\dots, J\}\right\},\]
where the centers $\{\mb H_1, \dots, \mb H_J\}$ are spaced $\dlt$ apart. The candidate support $\c A$ is then the set of all $2^d$ corner points of each hypercube, i.e., $\mb H_j + \frac{\dlt}{2}\mb v$ for $\mb v \in \{-1, 1\}^d$. Since $\c M$ is compact, $\c A = \{\mb a_1, \ldots, \mb a_m\}$ is finite, and the support-constrained NPMLE reduces to a finite-dimensional optimization over mixing weights: $\h G^{\c A} = \sum_{j=1}^m \h w_j\dlt_{\mb a_j}$ where
\begin{align}
    \h{\mb w} \in \argmax_{\mb w\in \Dlt_{m-1}} \frac{1}{K}\sum_{k=1}^K\log \left(\sum_{j=1}^mw_j\vrp^{(k)}(\tx_k; \mb a_j)\right).
    \label{eq-discretized-npmle}
\end{align}
In \cite{soloff2025multivariate}, the precision matrices stay fixed throughout the optimization. Here they must move with the atoms. The role of the discretized solve is therefore only to provide a stable initialization for the support and weights; the subsequent pseudo-EM iterations in Algorithm \ref{alg:vaneb} perform the heteroskedastic recoupling.

The following proposition quantifies the approximation error introduced by this discretization. As $\dlt\to 0$, the discretized NPMLE's log-likelihood converges to that of the unrestricted NPMLE, with an explicit bound that can guide practical grid selection:

\begin{proposition}[\citealp{soloff2025multivariate}]\label{prop-approximation} 
    Let ${\c M}\subset\R^d$ be any compact set containing all solutions to the true NPMLE problem \eqref{eq-NPMLE}. Assume the diameter of ${\c M}$ is at most $D$ and fix grid spacing $\dlt \in \left(0, \sqrt{\frac{3}{4d}}\frac{\us}{D}\right)$. Construct a grid covering ${\c H}$ of ${\c M}$ with hypercubes of width $\dlt$, and let ${\c A}$ be the set of corner points. Then any discretized NPMLE solution $\h G^{\c A}$ from \eqref{eq-discretized-npmle-defn} satisfies the approximation error bound
    \begin{align}
        \sup_{G\in{\c P}(\R^d)}\ell_K(G)
        - \ell_K(\h G^{{\c A}}) \le \frac{d}{\us}\lrp{\frac{1}{2}+\frac{2D^2}{\us}}\dlt^2.
        \label{eq-log-likelihood-approximation}
    \end{align}
\end{proposition}

\subsubsection{Convergence of VANEB Iterates}
We next state a fixed-point guarantee for the iterative scheme in Algorithm \ref{alg:vaneb}. The key is that contraction is controlled by Brouwer's fixed-point theorem.
\begin{lemma}[Convergence of VANEB iterates]\label{lem-vaneb-convergence}
    Under Assumption \ref{ass:reg}, at iterate $(t)$, fix E-step weights $w_{kj}^{(t)}\propto w_k\vrp^{(k)}(\tx_k;\mb a_j^{(t)})$ with $\sum_{k=1}^Kw_{kj}^{(t)}=1$, and define
    \[\bar{\b\psi}_j(\mb a):=\lrp{\sum_{k} w_{kj}^{(t)}n_k\Sigma_k(\mb a)^{-1}}^{-1}\sum_{k}w_{kj}^{(t)}n_k\Sigma_k(\mb a)^{-1}\bx_k.\]
    Write $n_j^w:=\sum_{k=1}^Kw_{kj}^{(t)}n_k$ and the local weighted radius
    \[
        R_j^w:=\inf_{\mb u\in\R^d}\brcs{\sum_{k=1}^K\frac{w_{kj}^{(t)}n_k}{n_j^w}\nrm{\bx_k-\mb u}{}},
    \]
    and set $L_{\bar\psi,j}:=\tau(1+\tau)R_j^w L_\Sigma/\us$ and $R:=\max_{1\le k,\ell\le K}\nrm{\bx_k-\bx_\ell}{}$. Assume that $\mb a_j^{(t)}\in\bigcup_k{\bb B}_{r_K}(\bx_k)$ for $r_K:=\sqrt{\frac{2\bs}{\un}\log\lrp{e\tau^{d/2}K}}$. Then the following holds.
    \begin{enumerate}
        \item A fixed point $\mb a_j^*$ exists and satisfies
        \[\sum_{k=1}^Kw_{kj}^{(t)}n_k\Sigma_k(\mb a_j^*)^{-1}(\bx_k-\mb a_j^*)=\b0_d.\]
        If $L_{\bar\psi,j}<1$, then $\mb a_j^*$ is unique and the iterates of $\bar{\b\psi}_j$ converge to it at rate $L_{\bar\psi,j}$.
        Furthermore, under \eqref{eq:like} and $K=o(e^\un)$, 
        \[
            R_j^w={\c O}_P\bxs{\sqrt{\frac{\bs\tau}{\un}\log\bigl(e\tau^{d/2}K^2\un(1+R)^2\bigr)}}.
        \]
        Also write $R_\tht:=\max_{k,\ell}\nrm{\b\tht_k-\b\tht_\ell}{}$. Then
        \begin{align}\label{eq-R-decomp}
            \bigl|R-R_\tht\bigr|=\cO_P\lrp{\sqrt{\frac{d+\log K}{\un}}},
        \end{align}
        and $L_{\bar\psi,j}=o_P(1)$ under the following:
        \begin{enumerate}[label=(\alph*)]
            \item If $R_\tht\le D_0$ almost surely for some $D_0>0$.
            \item If $G_0$ has a sub-Gaussian tail with finite mean.
        \end{enumerate}
        \item Under the conditions in 1 and $K=o(e^{\sqrt{\un}})$, for each atom $j$ there exists a true critical point $\mb a_{j0}^*$ of the score, $\sum_{k=1}^Kw_{kj}^{(t)}\mb s_k(\mb a_{j0}^*)=\b0_d$, with
        \begin{align}\label{eq-full-score-distance}
            \nrm{\mb a_{j0}^*-\mb a_j^*}{}=o_P\lrp{1}.
        \end{align}
    \end{enumerate}
\end{lemma}
The first claim states that under the condition $K=o(e^\un)$, the fixed point uniquely exists and the VANEB iterates converge to it at the linear rate $L_{\bar\psi,j}$, which is vanishing under a bounded client population or sub-Gaussian tail $G_0$.

The second claim quantifies the cost of freezing: a fixed point $\mb a_j^*$ is not an exact stationary point of the target score, but its distance to the true critical point is $o_P(1)$.
In summary, each VANEB iteration solves a standard ridgeline problem, and the resulting fixed points approach the target critical points as the local sample sizes grow.

\section{Theoretical Properties of VANEB}\label{sec:theory}
\subsection{Hellinger Accuracy of the NPMLE}
Under the asymptotic normality of M-estimators, we now analyze the statistical performance of the NPMLE and the resulting personalized estimator. Since different clients have different likelihoods, the natural loss is the average squared Hellinger distance.
For any two priors $G, H \in {\c P}(\R^d)$, define
\begin{align*}
    \bar{h}^2(f_{\blt,G}, f_{\blt,H})& := \frac{1}{K}\sum_{k=1}^K h^2\lrp{f_\GSk, f_\HSk},\\
    \bar{h}^2(\t f_{\blt,G,i}, \t f_{\blt,H,i}) & := \frac{1}{K}\sum_{k=1}^K h^2\lrp{\t f_{k,G,i}, \t f_{k,H,i}}, \quad \quad 
    i=1,...,d,
\end{align*}
where $h^2(f,g)= \frac{1}{2}\int_{\R^d}\brcs{\sqrt{f(\mb x)}-\sqrt{g(\mb x)}}^2{\rm d}\mb x$ is the standard squared Hellinger distance. Our first main result establishes convergence rates for the NPMLE estimator $\h G$ measured by the average squared Hellinger distance to the true prior $G_0$. 

To state our main convergence result, we introduce a rate function that captures the estimation difficulty. Fix a positive scalar $M$ and a compact set $S\subset \R^d$. The convergence rate is controlled by the function
\begin{align}\label{eq-eps}
    \eps_M^2(S, G_0)
    := \frac{M^d}{K}\Vol(S^1)\un^{d/2}\log^{d/2+1}K + \log K\inf_{q\ge\frac{d+1}{2\log K}}\frac{1}{K}\sum_{k=1}^K\left(\frac{2\mu_{k,q}(S,G_0)}{M}\right)^q,
\end{align}
where $\mu_{k,q}(S,G_0):=\E_{\b\tht\sim G_0}[\mathfrak{d}_S(\sqrt{n_k}\b\tht)^q]^{1/q}$ measures how close a random parameter is to the set $S$, with $\mathfrak{d}_S(t):= \inf_{s\in S}\|t - s\|$. $S^1:=\brcs{\mb x\in\R^d;\mathfrak{d}_S(\mb x)\le1}$ denotes the 1-enlargement of $S$.

Let $\tau:=\bs/\us$ denote the condition number of the covariance bounds. The following theorem provides a non-asymptotic guarantee linking the optimization quality of an approximate NPMLE to its statistical accuracy in estimating the true data distribution.
We use $\le_{d,\us,\bs}$ henceforth to denote the inequality up to a constant multiplication where the constant depends on $d$, $\us,$ and $\bs$ only.
\begin{theorem}\label{thm-hllngr}
    Any NPMLE $\h G\in{\c P}(\R^d)$ such that
    \begin{align}\label{eq-hllngr-npmle}
        \ell_K(G_0) - \ell_K(\h G) \le_{d,\us,\bs} \eps_M^2(S, G_0)
    \end{align}
    satisfies
    \begin{align}\label{eq-hllngr-whp}
        \P\bxs{\bar{h}^2(f_{\blt,\h G}, f_{\blt,G_0}) \gtrsim_dt^2\eps_M^2(S, G_0)}&\le 2\exp\lrp{-t^2\un^{d/2}\log^dK} \\
        \P\bxs{\max_{i=1,...,d}\bar{h}^2(\t f_{\blt,\h G,i}, \t f_{\blt,G_0,i}) \gtrsim_dt^2\eps_M^2(S, G_0)}&\le 2\exp\brcs{-t^2(\un^{d/2}\log^dK-\log\tau)}
    \end{align}
    for all $t>0$.
\end{theorem}
The proof of Theorem \ref{thm-hllngr} is in section \ref{sec:prf-hell}. The first inequality controls the usual mixture fit $f_{\blt,\h G}$, while the second controls the variance-weighted component densities $\brcs{\t f_{\blt,\h G,i}}_{i\le d}$ that appear in the generalized Tweedie formula. The second bound is nontrivial only when $\un^{d/2}\log^dK>\log\tau$, which holds in the regime of interest here.
The concentration exponent scales with $\un^{d/2}\log^d K$, so larger local sample sizes and more clients improve concentration.
At the same time, the approximation error is governed by the metric entropy of the induced marginal class
\[\bb F:=\{f_{\blt,G}:G\in\c{P}(\R^d)\}.\]
As $\un$ increases, each kernel in the integral $f_{k,G}=\int\vrp^{(k)}(\cdot\;;\b\tht)\,{\r d}G(\b\tht)$ becomes more concentrated, so small perturbations of $G$ produce sharper, more localized changes in $f_{\blt,G}$, and resolving the class $\bb F$ to a fixed accuracy requires a finer cover.
Specifically, the covering number $N(a,S^M)$ of an $M$-enlargement of $S$ scales as $\un^{d/2}$, because the cover must use balls of radius $a\propto\sqrt{\un^{-1}\log K}$ instead of $a\propto\sqrt{\log K}$ as in the one-sample case of \cite{soloff2025multivariate,saha2020nonparametric}.

To prevent an unscaled set $S$ from overstating the covering complexity as local sample sizes increase, we introduce a client-specific local scaling that normalizes the concentration. Specifically, we define $S_k:=\frac{1}{\sqrt{n_k}}S$ and work with the refined rate function parameterized by $S_\blt:=\brcs{S_k}_{k\le K}$:
\begin{align}\label{eq-eps2}
    \eps_M^2(S_\blt, G_0)
    := \frac{M^d}{K}\max_{k\le K}\Vol(S_k^1)\un^{d/2}\log^{d/2+1}K + \log K\inf_{q\ge\frac{d+1}{2\log K}}\frac{1}{K}\sum_{k=1}^K\left(\frac{2\mu_{k,q}(S_k,G_0)}{M}\right)^q \nonumber \\
    =_d \frac{M^d}{K}\Vol(S^1)\log^{d/2+1}K + \log K\inf_{q\ge\frac{d+1}{2\log K}}\max_{k\le K}\left(\frac{2\mu_{k,q}(S_k,G_0)}{M}\right)^q
\end{align}
Using $S_k$ keeps the approximation term aligned with each client's effective resolution and balances the concentration gain against entropy growth.
\begin{corollary}\label{cor-hllngr}
    In Theorem \ref{thm-hllngr}, suppose we use the Hellinger error rate $\eps_M(S_\blt,G_0)$.
    Any approximate NPMLE $\h G$ such that
    \begin{align}\label{eq-hllngr-npmle2}
        \ell_K(G_0) - \ell_K(\h G) \le_{d,\us,\bs} \eps_M^2(S_\blt, G_0)
    \end{align}
    satisfies
    \begin{align*}
        \P\bxs{\bar{h}^2(f_{\blt,\h G}, f_{\blt,G_0}) \gtrsim_d\eps_M^2(S_\blt, G_0)}&\le 2\exp\lrp{-\log^dK} \\
        \P\bxs{\max_{i=1,...,d}\bar{h}^2(\t f_{\blt,\h G,i}, \t f_{\blt,G_0,i}) \gtrsim_d\eps_M^2(S_\blt, G_0)}&\le 2\exp\brcs{-(\log^dK-\log\tau)}.
    \end{align*}
    Furthermore, we have
    \[\E\bxs{\bar{h}^2(f_{\blt,\h G}, f_{\blt,G_0})}\le_{d,\us,\bs}\eps_M^2(S_\blt,G_0).\]
\end{corollary}
We henceforth use the rate $\eps_M^2(S_\blt, G_0)$ as the benchmark for the NPMLE's statistical accuracy. Theorem \ref{thm-hllngr} and Corollary \ref{cor-hllngr} show that any approximate NPMLE with a log-likelihood gap of order $\eps_M^2(S_\blt, G_0)$ attains the same average squared Hellinger accuracy as an exact maximizer. This gives a concrete target for optimization: any $\h G$ satisfying \eqref{eq-hllngr-npmle2} inherits the rate $\eps_M^2(S_\blt, G_0)$.

\begin{remark}[Discretization budget]\label{rem-disc}
    Like in \cite{soloff2025multivariate}, we derive the desired size of the hypercubes in $\c H$.
    The width of each hypercube $\dlt > 0$ can be chosen to satisfy the error bound required in \eqref{eq-hllngr-npmle2}, and it suffices to choose $\dlt$ such that
    \[
    \lrp{\frac{1}{2}+\frac{2D^2}{\us}}\dlt^2\le_{d,\us,\bs}\eps_M^2(S_\blt, G_0).
    \]
    The LHS is the discretization error from replacing a continuous support by a grid, while the RHS is the statistical error budget allowed by the Hellinger theory.
    Assuming $M\ge_{\bs}\sqrt{\log K}$ and $D^2 > \us$, we obtain
    \[\eps_M^2(S_\blt, G_0)\ge_{d,\us,\bs}\frac{\log^{d+1}K}{K},\]
    and since $\dlt$ must also satisfy $D^2\dlt^2\le\frac{3\us^2}{4d}$ (Proposition \ref{prop-approximation}), the condition
    \begin{align}\label{eq-disc-rate}
        D^2\dlt^2\le_{d,\us,\bs}\frac{\log^{d+1}K}{K}
    \end{align} 
    suffices.
    
    In words, once the grid is fine enough to satisfy \eqref{eq-disc-rate}, further refinement does not improve statistical accuracy at the target scale.
    According to \citet[Proposition 4(b)]{soloff2025multivariate}, the support of the NPMLE $\h G$ expands at rate $D \asymp \sqrt{\log K}$ with high probability. Hence, we require $\dlt \asymp \sqrt{K^{-1}\log^d K}$ and $|\c A|=m \asymp (K/\log^d K)^{d/2}$.
\end{remark}

\begin{corollary}\label{cor-hllngr-rsk} 
    Suppose $\h G\in {\c P}(\R^d)$ is any approximate NPMLE such that
    \begin{align}\label{eq-apprx-npmle}
        \ell_K(G_0) - \ell_K(\h G) \le_{d,\us,\bs}\frac{\log^{d+1}K}{K}.
    \end{align}
    If $G_0$ has compact support $S$, then 
        \[
        \E\,\bar{h}^2\lrp{f_{\blt,\h G}, f_{\blt,G_0}} \le_d 
        \Vol(S^1)\frac{\log^{d+1}K}{K}.
        \]
\end{corollary}
The above corollary's claim follows from the calculations in \citet[Corollary 2.2]{saha2020nonparametric}, noting that the client-scaled sets $S_k$ ensure the entropy term remains near-parametric rate of $\cO(\log^{d+1}K/K)$ for the case when $G_0$ has compact support. Under the weaker moment condition, $K$ in the denominator will converge in a slower rate of $K^c$ for some $c\in(0,1)$ \citep[equation 2.10]{saha2020nonparametric}.

\subsection{Oracle denoising inequality}
Building upon the Hellinger convergence guarantees, we now establish the statistical risk of the resulting personalized estimates. Specifically, we derive an oracle inequality that controls the excess risk of the VANEB estimator relative to the oracle posterior mean $\oraet_k$.
\begin{theorem}\label{thm-denoising} 
    Let $\h G$ denote any approximate NPMLE satisfying \eqref{eq-apprx-npmle} and
    \begin{align}\label{eq-apprx-npmle-2}
        \ell_K\left(K^{-1} \dlt_{\tx_k} + (1-K^{-1})\h G\right)\le\ell_K(\h G)
    \end{align}
    for all $k\le K$. Fix some $M \ge \sqrt{8\bs\log (eK^2)}$ and a nonempty, compact set $S\subset\R^d$. Define $\eps_M^2(S_\blt, G_0)$ as in \eqref{eq-eps2}.
    If $\un\gtrsim \frac{K}{M^{d-2}\log^{d/2+1}K}$, then
    \begin{align}
        \frac{1}{K}\sum_{k=1}^K\E\|{\bx_k}^v - \oraet_k\|^2
        \le_{d,\us,\bs}\frac{(\log K)^{3\vee(d/2-1)}}{\un}\eps_M^2(S_\blt,G_0).
    \end{align}
    In particular, consider the following special cases for $G_0$:
    \begin{enumerate}
    \item (Discrete support) If $G_0 = \sum_{j=1}^{k^*} w^*_j\dlt_{\mb a^*_j}$, then 
    \[
        \frac{1}{K}\sum_{k=1}^K\E\,\|{\bx_k}^v - \oraet_k\|^2
        \le_{d,\us,\bs}\frac{k^*}{K\un}(\log K)^{d+(d/2\vee4)}.
    \]
    \item (Compact support) If $G_0$ has compact support $S$, then 
    \[
    \frac{1}{K}\sum_{k=1}^K\E\,\|{\bx_k}^v - \oraet_k\|^2
    \le_{d,\us,\bs}\frac{\Vol(S^1)}{K\un}(\log K)^{d+(d/2\vee4)}.
    \]
    \item (Gaussian Mixture Model) For $\mb w_j^*\in\brcs{w_1^*,...,w_{k^*}^*;\sum_{j=1}^{k^*}w_j^*=1}$, suppose the prior is of $G_0 = \sum_{j=1}^{k^*} w^*_j\,\c N(\mb a^*_j, \Gamma^*_j)$. Let $r>0$ be a sequence such that $r=\cO(\sqrt{\un})$ and $S_r:=\bb R_r(\b0_d)$. Then,
    \[
        \frac{1}{K}\sum_{k=1}^K\E\,\|{\bx_k}^v - \oraet_k\|^2
        \le_{d,\us,\bs}\frac{\Vol(S_r)}{K\un}(\log K)^{d+(d/2\vee4)}.
    \]
    \end{enumerate}
\end{theorem}
The proof of Theorem \ref{thm-denoising} is in section \ref{sec:prf-denoising}. Theorem \ref{thm-denoising} bounds the excess risk of the VANEB estimator relative to the oracle Bayes estimator, averaged over clients. Condition \eqref{eq-apprx-npmle-2} is a local maximality property: the log-likelihood cannot be improved by mixing in the point mass $\dlt_{\tx_k}$ at any single client's estimate. It holds automatically for any exact global maximizer of $\ell_K$.

Compared to the Hellinger convergence rates for the NPMLE, the denoising bound carries an additional factor of $\un^{-1}(\log K)^{3\vee(d/2-1)}$.
The factor $\un^{-1}$ is the local parametric scale of each client estimate, whereas the remaining $K$-dependence reflects how quickly the shared prior is learned. In all three special cases (discrete support, compact support, and GMM), the bound decreases in both $\un$ and $K$, so larger local samples and more clients both improve denoising.

When $G_0$ is supported on discrete atoms, the risk bound scales with the number of atoms $k^*$, reflecting the complexity of the underlying parameter distribution. For compactly supported priors, the risk scales with the volume of the support, indicating that more spread-out distributions are harder to learn.

When $G_0$ has sub-Gaussian tails as in the GMM case of Theorem \ref{thm-denoising}, a radius of $R \asymp \sqrt{\log K}$ reduces the tail mass below $1/K$ (provided $\un\gg\log K$). The compact support assumption can then be dropped at the cost of a polylogarithmic factor, and the excess risk remains $\t\cO\lrp{(\un K)^{-1}}$.

\subsection{Risk decomposition and benefits for personalized federated learning}
Consider collections of local estimators $\h{\b\tht}_k\in\R^d$ and parameters ${\b\tht}_k\in\R^d$ as in \eqref{eq:like}, i.e., 
$\h\Tht=\{\h{\b\tht}_k\}_{k=1}^K,\quad\Tht=\{\b\tht_k\}_{k=1}^K$, 
and define the (averaged) Bayes risk of the estimators by
\[
{\c R}(\h\Tht):=\frac{1}{K}\sum_{k=1}^K\E\|\h{\b\tht}_k-\b\tht_k\|^2.
\]
Let $\Tht^o:=\{\oraet_k\}_{k=1}^K$ be the collection of oracle Bayes rules.
Since $\tx_k=\sqrt{n_k}\bx_k$ is a deterministic one-to-one rescaling and the clients are independent under the true prior $G_0$, we have
\[
\E_{G_0}(\b\tht_k\,|\,\tx_k)=\E_{G_0}(\b\tht_k\,|\,\bx_k)=\oraet_k.
\]
Therefore $\c R$ admits the orthogonal decomposition,
\begin{align*}
    {\c R}(\h\Tht)
    &=\frac{1}{K}\sum_{k=1}^K\E\|\h{\b\tht}_k-\oraet_k+\oraet_k-\b\tht_k\|^2 \\
    &=\frac{1}{K}\sum_{k=1}^K\E\|\h{\b\tht}_k-\oraet_k\|^2
    +\frac{1}{K}\sum_{k=1}^K\E\|\oraet_k-\b\tht_k\|^2 \\
    &=:{\c R}_E(\h\Tht)+{\c R}(\Tht^o).
\end{align*}
Once $G_0$ is fixed, ${\c R}(\Tht^o)$ is determined entirely by within-client posterior uncertainty. The next theorem specializes the multivariate van Trees inequality, as in \citet{gassiat2024vantrees}, to our diagonal heteroskedastic Gaussian working model.
\begin{theorem}[van Trees lower bound for the oracle baseline]\label{thm:oracle-vantrees}
Let $\lambda$ denote the Lebesgue measure on $\R^d$. Suppose that $G_0$ admits a $C^1$ density $g_0:=\frac{{\r d} G_0}{{\r d} \lambda}$ such that $g_0(\b\tht)\to0$ as $\|\b\tht\|\to\infty$ and
\[\c J_0:=\E\bxs{\nabla\log g_0(\b\tht)\brcs{\nabla\log g_0(\b\tht)}^\top}\]
is finite. Then, under Assumption \ref{ass:reg}, $\Tht^o$ satisfies
\[
\frac{d}{\nrm{\c J_0}{}+\bn/\us+dL_\Sigma^2/2}
\le {\c R}(\Tht^o)
\le \frac{d\bs}{\un}.
\]
\end{theorem}
This theorem shows that the oracle term remains on the local $1/n$ scale and does not by itself produce any federation gain, so it remains to compare the excess risks of $\h\Tht$ and $\h\Tht^v$.
For this, we consider the special case when $G_0$ is a Gaussian mixture, as in case 3 of Theorem \ref{thm-denoising}. Using $\us\mb I_d\preceq\Sigma_k(\b\tht_k)\preceq\bs\mb I_d$, we also have
\[
\frac{d\us}{\bn}\le {\c R}(\h\Tht)\le \frac{d\bs}{\un},
\]
so
\[
{\c R}_E(\h\Tht)\le \frac{d\bs}{\un}-\frac{d}{\nrm{\c J_0}{}+\bn/\us+d L_\Sigma^2/2}.
\]
Thus, under Assumption \ref{ass:reg} with $\bn\asymp \un$, this upper bound on the excess risk of the local M-estimators that are asymptotically normal is of order $1/\un$. Since the local estimators use no cross-client information, their excess risk does not decrease as $K$ grows.

By contrast, Theorem \ref{thm-denoising} and \eqref{eq-eps2} give the following excess risk bound for VANEB:
\[
{\c R}_E(\h\Tht^v)
\le_{d,\us,\bs}\Vol(S_r)\frac{M^d}{K\un}\log^{d/2+(d/2\vee4)}K
\]
which decays with $K$ through the prior-learning term. Both risk decompositions share the same oracle baseline ${\c R}(\Tht^o)$, so the two methods differ only in the excess term: for the local M-estimators it can remain of order $1/\un$ regardless of $K$, whereas for VANEB it is $\t\cO((K\un)^{-1})$, decreasing in both the local sample sizes and the number of clients. The gain from federation enters entirely through learning the shared prior.

\section{Numerical Simulations}\label{sec:simulations}
We evaluate VANEB in three heterogeneous federated learning scenarios in $\R^3$. The first uses a parameter-dependent diagonal covariance matrix that closely aligns with our assumption, and the other two are generalized linear model settings with a full Fisher information-based covariance matrix that naturally depends on the parameter. In the two GLM settings (logistic and Poisson), we additionally introduce client-specific covariate covariance matrices $\Sigma_{x,k}$ so both local likelihood curvature and covariate design vary across clients. The goal is to understand when variance-aware shrinkage yields practical gains over fixed-variance EB and standard FedAvg.

All three scenarios share a common prior $G_0$ in $\R^3$ designed to challenge parametric methods with complex, non-Gaussian geometry. The support variable is sampled on a truncated domain $t\in[-\pi/2,\pi/2)$ and mapped to a full phase $u=2(t+\pi/2)\in[0,2\pi)$, so each component keeps its full geometric shape. The five components are a \emph{trefoil knot} $\gamma_1$ centered at $(-2,0,0)$, a \emph{helix} $\gamma_2$ centered at $(2,0,0)$, a \emph{tilted ellipse} $\gamma_3$ centered at $(0,2.5,0)$, a \emph{figure-eight} $\gamma_4$ (lemniscate of Gerono) centered at $(0,-2.5,1)$, and a \emph{Viviani curve} $\gamma_5$ centered at $(0,0,-2.5)$ (Figure \ref{fig:prior3d}). Client parameters are sampled from a weighted mixture over curves with weights
\[(w_1,w_2,w_3,w_4,w_5)=(0.35,0.35,0.1,0.1,0.1),\]
and uniform sampling in $t$ on each selected curve. 
This five-curve prior occupies a three-dimensional support with multiple connected components, self-intersections on the trefoil, and varying curvature, features that a Gaussian mixture prior approximates poorly.
\begin{figure}[h]
    \centering
    \includegraphics[width=0.65\textwidth]{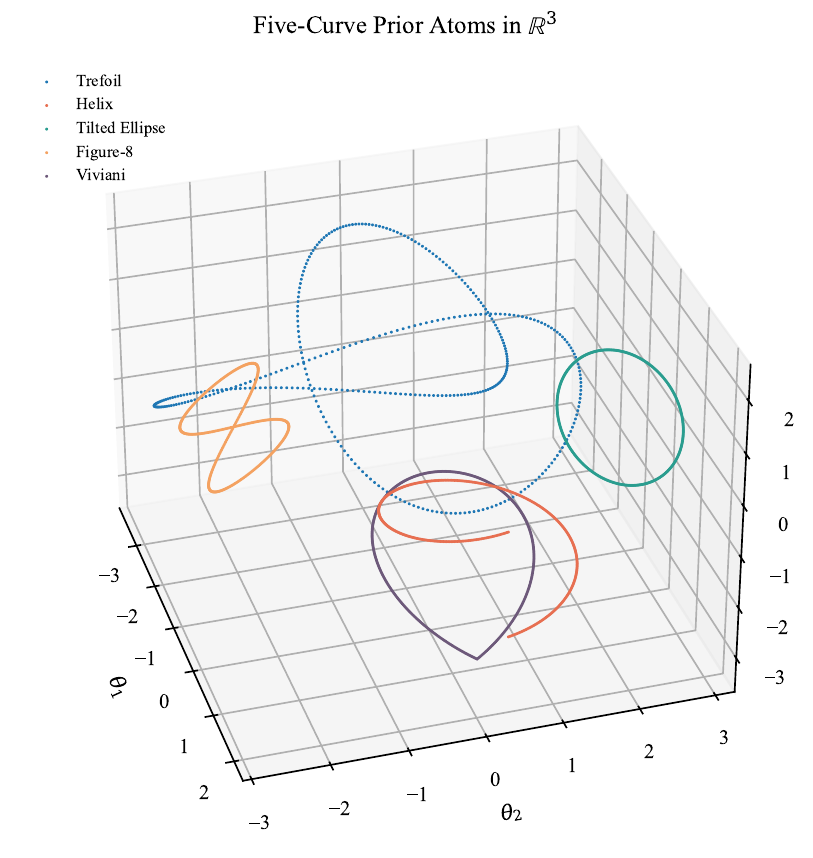}
    \caption{Support of $G_0$}
    \label{fig:prior3d}
\end{figure}

\subsection{Baselines and metrics.}
We evaluate the denoising risk $\operatorname{RMSE}(\h{\Theta}^v) = \left({K^{-1}\sum_{k=1}^K \|{\bx_k}^v - \b\tht_k\|^2}\right)^{1/2}$ in each scenario with (i) increasing $\un \in \{5,10,20,40\}$ at fixed $K=3200$, and (ii) increasing $K\in\{50,200,800,3200\}$ at fixed $\un=40$. Client sample sizes are drawn as $n_k\sim\operatorname{Unif}(\un, 2\un)$. All results are averaged over 200 Monte Carlo replicates.

We compare six estimators. The \textbf{Oracle Bayes} benchmark uses fixed atoms on the true five-curve manifold to form posterior means under the correct heteroskedastic likelihood:
\[\oraet_k=\int\frac{\b\tht\vrp^{(k)}(\tx_k;\b\tht)}{f_\GoSk(\tx_k)}{\r d}G_0(\b\tht)=\sum_{j}w_j^*\frac{\mb a_j^*\vrp^{(k)}(\tx_k;\b a_j^*)}{f_\GoSk(\tx_k)}\]
where $\{\mb a_j^*\}$ are atoms along the true prior curves with uniform weights $w_j^*$.
The \textbf{AdaMix} estimator \citep{ozkara2023statistical} iteratively fits a Gaussian mixture prior on the server and performs client-side MAP shrinkage (Algorithm \ref{alg:adamix} in Appendix).
The modified \textbf{NPEB} estimator uses each client's local sample-variance estimate as a fixed, parameter-independent precision throughout all EM iterations, providing a comparison where covariances do not adapt to atom locations (Note the original NPEB estimator in \citep{soloff2025multivariate} requires known covariances).
The \textbf{FedAvg} estimator computes a global average of client MLEs over a single communication round, representing a classical federated averaging approach that treats MLEs as fixed point estimates without inferring an underlying prior.
The \textbf{Local} baseline represents zero-communication learning: each client uses its own local MLE without any information pooling from other clients, serving as a baseline estimator that federated approaches are expected to outperform.
The \textbf{VANEB} estimator fits the NPMLE $\h{G}$ with parameter-dependent heteroskedastic covariances and returns posterior means under the fitted mixture (Section~\ref{sec:algorithm}).

\subsection{The scenarios}
\paragraph{Scenario (i): Sample Means with Quadratic Variance}
Each client transmits a sample mean for $\bx_k$ with quadratic variance function $\Sigma(\b\tht) = \diag(\tht_1^2,...,\tht_d^2)$, clipped coordinate-wise to $[s_{\min},s_{\max}] = [0.01,100]$ for numerical stability. VANEB uses the known variance function of the parameter, while NPEB uses local sample-variance estimates reported by each client. The quadratic variance creates a strong heteroskedastic signal, and this scenario evaluates VANEB's ability to leverage this heteroskedastic structure for improved estimation.

\paragraph{Scenario (ii): Multiclass Logistic Regression}
We consider the following 3-class logistic regression model. Let $\b\tht\in\bb{R}^3$ control all class boundaries, and $\mb B_c \in \bb{R}^{3 \times 3}$ ($c=1,2,3$) are fixed sparse diagonal projection matrices. Specifically: $$\mb B_1 = \text{diag}(1, -1, 0), \quad \mb B_2 = \text{diag}(0, 1, -1), \quad \mb B_3 = \text{diag}(-1, 0, 1).$$
Each client $k$ generates $n_k$ observations with features
\[
\mb z_{k,1},...,\mb z_{k,n_k} \overset{i.i.d.}{\sim} \c{N}(\mb{0}, \Sigma_{x,k}),
\]
where $\Sigma_{x,k}=\mb Q_k\diag(\lambda_{k,1},\lambda_{k,2},\lambda_{k,3})\mb Q_k^\top$, $\mb Q_k$ is an orthogonal random matrix, and $\lambda_{k,j}\sim\operatorname{Unif}(0.5,2.0)$. Class labels are then generated according to the softmax probabilities:
$$p_c(\mb z, \b\tht) := \Pr(Y_k=c\,|\,\mb z) = \frac{e^{\langle\mb B_c \mb z, \b\tht\rangle}}{\sum_{x=1}^3 e^{\langle\mb B_x \mb z, \b\tht\rangle}}.$$

Let $\b l_{c,k}^{(i)}:=\mb B_c\mb z_{k,i}$.
The expected Fisher information and its plug-in estimator at each client's MLE $\bx$ takes the form
\begin{align*}
    {\c I}_k(\b\tht) &= \bb{E} \left[ \sum_{c=1}^C p_c(\mb z, \b\tht) \b l_{c,k}^{(1)}\b l_{c,k}^{(1)}\,^\top - \left( \sum_{c=1}^C p_c(\mb z, \b\tht) \b l_{c,k}^{(1)} \right) \left( \sum_{c'=1}^C p_{c'}(\mb z, \b\tht) \b l_{c',k}^{(1)}\,^\top \right)\right] \\
    \h{\c I}_k(\bx) &= \frac{1}{n_k}\sum_{i=1}^{n_k} \left[ \sum_{c=1}^C p_c(\mb z_{k,i}, \bx)\b l_{c,k}^{(i)}\b l_{c,k}^{(i)}\,^\top - \sum_{c=1}^C p_c(\mb z_{k,i}, \bx) \b l_{c,k}^{(i)} \left( \sum_{c'=1}^C p_{c'}(\mb z_{k,i}, \bx) {\b l_{c',k}^{(i)}}^\top \right) \right].
\end{align*}
Then the VANEB method updates the precision $\h{\c I}_k(\bx)$ according to the M-step in Algorithm \ref{alg:vaneb}, while our adaptation of \cite{soloff2025multivariate}'s NPEB uses the client-reported precisions as fixed throughout. Note that the variance field in this simulation experiment is more complex than in the quadratic case. In particular, this setting lies outside the diagonal assumption in the theoretical results: the Fisher covariance is not diagonal. We include it to test whether the variance-aware mechanism remains useful in a richer covariance geometry.

\paragraph{Scenario (iii): Poisson Regression}\label{sec:sim_poisson}

In this scenario, each client fits Poisson regression $Y_k\sim\text{Poisson}(e^{\langle\mb z,\b\tht_k\rangle})$ with client-specific features generated as $\mb z\sim\c N(\b 0, (0.1)^2\Sigma_{x,k})$, where $\Sigma_{x,k}$ is generated as in the logistic scenario.
Each client obtains MLE $\bx_k$. 
The population Fisher and total empirical Fisher are
\begin{align*}
    {\c I}_k(\b\tht_k) &= \exp\left(\frac{0.1^2}{2}\b\langle\b\tht_k,\Sigma_{x,k}\b\tht_k\rangle\right)0.1^2\Sigma_{x,k}\left[\mb I_3 + 0.1^2\b\tht_k\b\tht_k^\top{\Sigma_{x,k}}\right], \\
    \h{\c I}_k(\bx_k)&=\frac{1}{n_k}\sum_{i=1}^{n_k}\exp(\langle\mb z_{k,i},\bx_k\rangle)\,\mb z_{k,i}\mb z_{k,i}^\top.
\end{align*}
This Fisher-derived covariance changes \emph{exponentially} with $\|\b\tht_k\|$, creating an extreme heteroskedastic regime: atoms at the periphery of the prior have high precision, and atoms near the origin are much noisier. As in the logistic setting, this is a full-covariance empirical test that lies outside the regime covered by the main proofs.

\begin{figure}[h]
\centering
\begin{subfigure}{0.48\textwidth}
    \includegraphics[width=\textwidth]{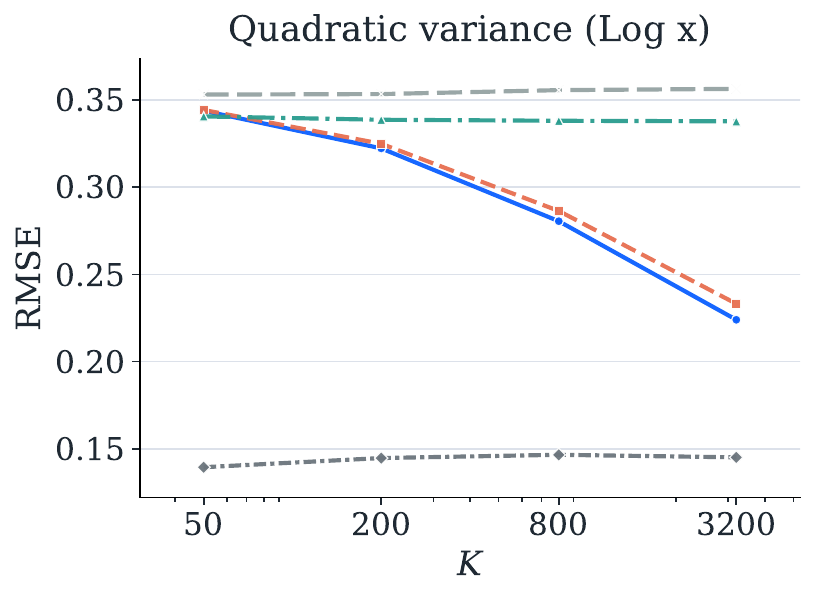}
    \caption{increasing $K$ (fixed $\un=40$).}
    \label{fig:quad_k}
\end{subfigure}
\hfill
\begin{subfigure}{0.48\textwidth}
    \includegraphics[width=\textwidth]{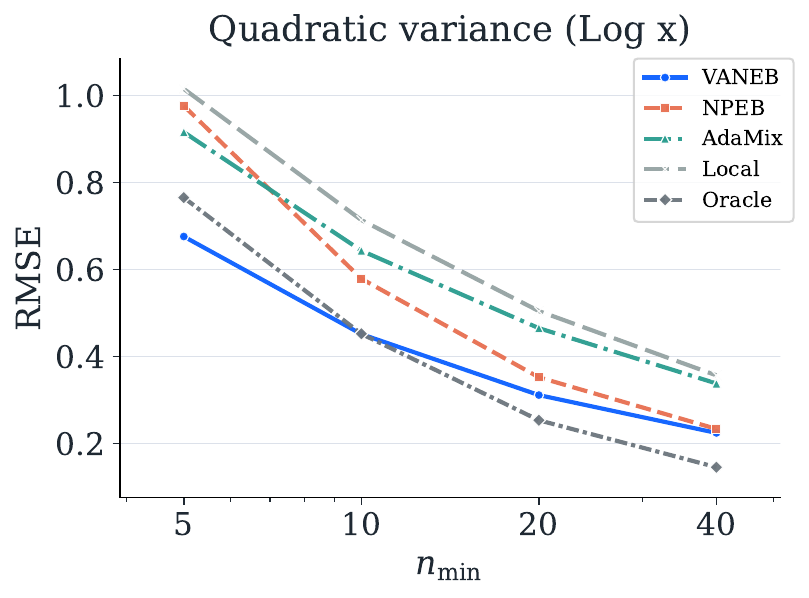}
    \caption{increasing $\un$ (fixed $K=3200$).}
    \label{fig:quad_nmin}
\end{subfigure}
\caption{Quadratic variance scenario: RMSE comparison ($d=3$, five-curve prior with support parameter $-\pi/2\le t<\pi/2$) over 200 replicates. In both figures, FedAvg is omitted for readability due to the large RMSE scale.}
\label{fig:simulation_quadratic}
\end{figure}

\subsection{Results.}

\paragraph{Quadratic variance scenario (Figure \ref{fig:simulation_quadratic}).}
In Figure \ref{fig:simulation_quadratic}(a), we note that VANEB improves steadily with an increasing number of clients and consistently outperforms fixed-variance EB once $K$ is moderate.
Figure \ref{fig:simulation_quadratic}(b) shows that VANEB strongly outperforms NPEB in the small sample regime. In both simulations, VANEB and NPEB generally outperform AdaMix and local-only estimation. The FedAvg method, which averages local estimates, has a high RMSE due to client dissimilarity and is omitted from the figures due to its large scale. These results support the claim that dynamic atom-covariance coupling is most valuable when local summaries are noisy.

\begin{figure}[h]
\centering
\begin{subfigure}{0.48\textwidth}
    \includegraphics[width=\textwidth]{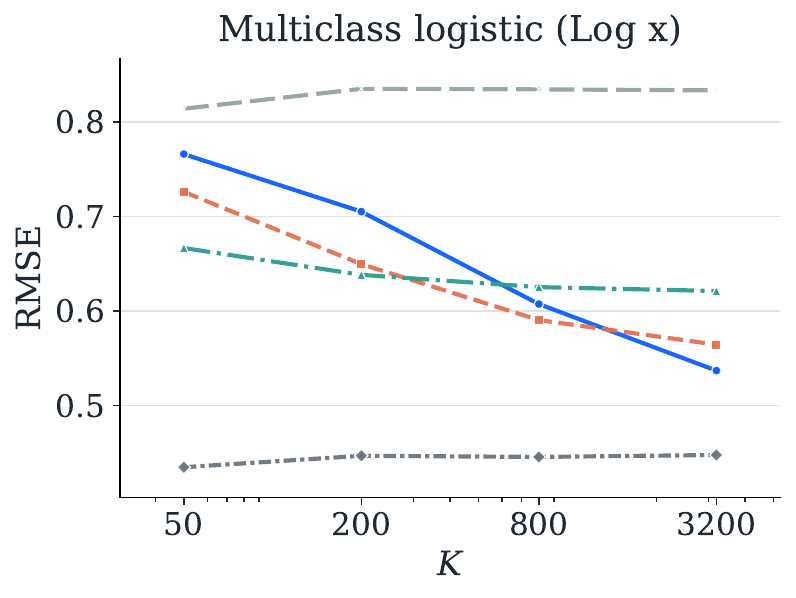}
    \caption{$K$-sweep (fixed $\un=40$).}
    \label{fig:logistic_k}
\end{subfigure}
\hfill
\begin{subfigure}{0.48\textwidth}
    \includegraphics[width=\textwidth]{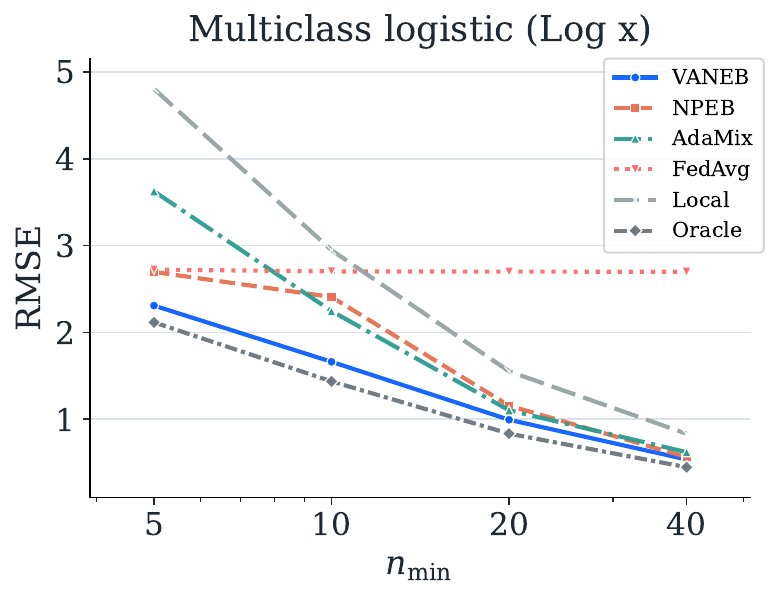}
    \caption{$\un$-sweep (fixed $K=3200$).}
    \label{fig:logistic_nmin}
\end{subfigure}
\caption{Multiclass logistic scenario: RMSE comparison. FedAvg omitted from part (a) for readability due to large RMSE scale.}
\label{fig:simulation_logistic}
\end{figure}
\paragraph{Multiclass logistic scenario (Figure \ref{fig:simulation_logistic}).}
The logistic regression environment presents a more complex, parameter-coupled Fisher covariance and exhibits a clear trade-off between sample size and client pool size. In the $\un$-sweep at $K=3200$, VANEB substantially outperforms NPEB, AdaMix, and FedAvg in the hardest regimes with small client sample sizes. In the $K$-sweep at fixed $\un=40$, however, VANEB is not the best for small and moderate client counts ($K=50,200,800$), where NPEB and AdaMix have slightly lower errors. The crossover appears at a larger number of clients, and at $K=3200$, VANEB has the best performance among the compared methods.

For intuition behind the slow $K$-gain, consider binary logistic regression with a single covariate $z=1$:
${\c I}(\tht)=p(\tht)\bigl(1-p(\tht)\bigr)$ with $p(\tht)=e^{\tht}/(1+e^{\tht})$.
Then $\nabla{\c I}(\tht)={\c I}(\tht)\bigl(1-2p(\tht)\bigr)$, so the Fisher gradient vanishes at the decision boundary $\tht=0$ and, exponentially fast, as $|\tht|\to\infty$.
Thus ${\c I}(\tht)$ (and with it the local variance-shift signal that VANEB exploits when re-evaluating $\Sigma$ at the atoms) is appreciable only in a thin transition region around $\tht=0$ and is quickly negligible under separation.
When this signal is small, the error in estimating $G_0$ can dominate the benefit of learning the variance field, so a larger client pool $K$ is needed before variance-aware updates improve upon fixed-covariance NPEB.

\begin{figure}[h]
\centering
\begin{subfigure}{0.48\textwidth}
    \includegraphics[width=\textwidth]{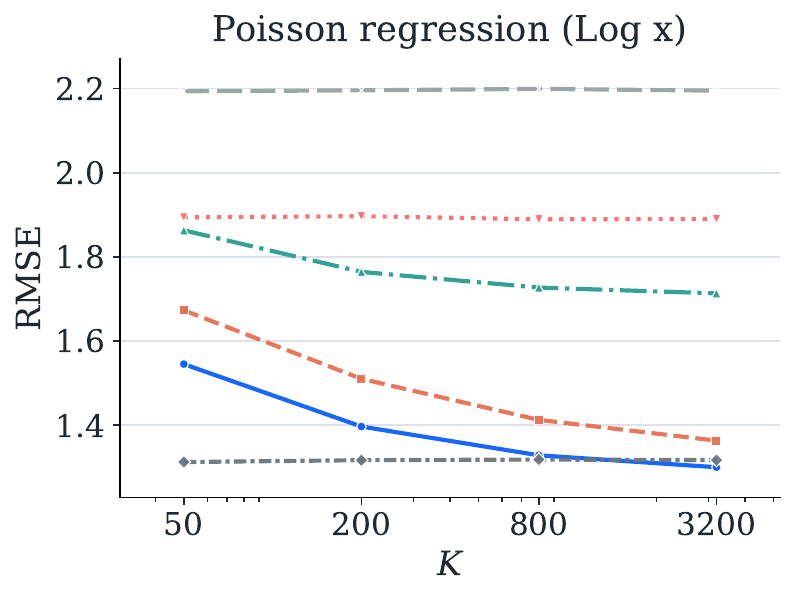}
    \caption{$K$-sweep (fixed $\un=40$).}
    \label{fig:poisson_k}
\end{subfigure}
\hfill
\begin{subfigure}{0.48\textwidth}
    \includegraphics[width=\textwidth]{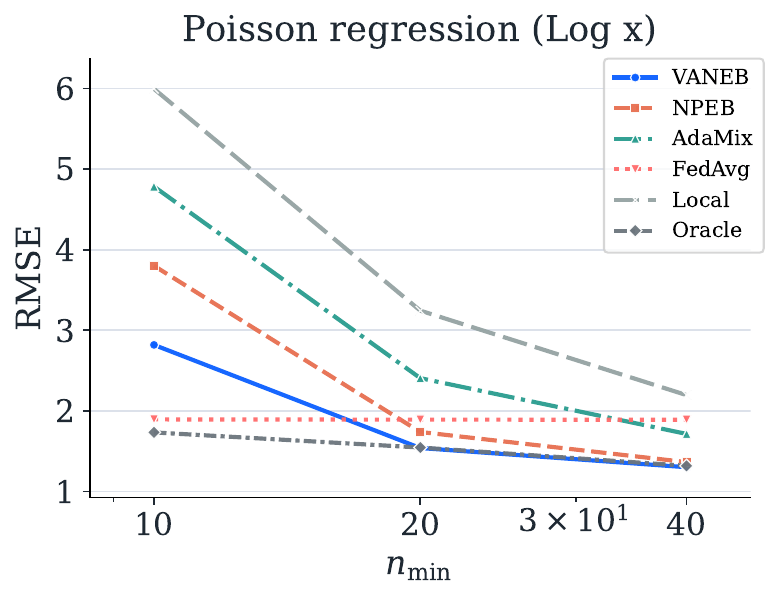}
    \caption{$\un$-sweep (fixed $K=3200$).}
    \label{fig:poisson_nmin}
\end{subfigure}
\caption{Poisson regression scenario: RMSE comparison.}
\label{fig:simulation_poisson}
\end{figure}
\paragraph{Poisson regression scenario (Figure \ref{fig:simulation_poisson}).}
The Poisson regression task is the most strongly heteroskedastic setting. At $\un=5$, the local MLEs are numerically unstable and the RMSE of all local-estimator-based methods is several orders of magnitude larger than at $\un\ge10$, so this point is omitted from the plots. From $\un=10$ onward, VANEB is consistently the strongest non-oracle method (e.g., $2.819$ vs NPEB $3.8$ at $\un=10$, and $1.3$ vs $1.363$ at $\un=40$). In the $K$-sweep at fixed $\un=40$, VANEB also maintains a strict advantage over NPEB for every $K$ ($1.545$ vs $1.673$ at $K=50$, $1.300$ vs $1.363$ at $K=3200$). 
Compared with the logistic case, the Poisson Fisher field is easier for VANEB to exploit at small $K$ because ${\c I}_k(\b\tht)$ changes essentially monotonically with the signal magnitude through the factor $\exp\big(0.1^2\langle\b\tht,\Sigma_{x,k}\b\tht\rangle/2\big)$. The Poisson variance field thus varies more strongly than the logistic one, but its dependence on $\b\tht$ is simpler to infer.

Across scenarios, VANEB is strongest under $\Sigma_k$-$\b\tht$ coupling and either local sample sizes are small or the client pool is sufficiently large. The results also show that variance-adaptive gains are amplified when enough clients are available to stabilize the learned prior. Overall, the simulation studies support VANEB as a robust nonparametric personalization method while highlighting finite-sample regimes in which fixed-covariance baselines can remain competitive.

In contrast, both local estimates and aggregation across clients, such as FedAvg, fail significantly in many cases. FedAvg is omitted from several plots because its high RMSE would affect readability. Because our 5-curve prior is highly multimodal and geometrically spread out, enforcing a single global point estimate incurs massive squared errors across heterogeneous clients.

\begingroup
\SetAlFnt{\small}
\SetAlCapFnt{\small}
\SetAlCapNameFnt{\small}
\SetAlgoNoEnd
\DontPrintSemicolon
\begin{algorithm}[!h]
\caption{VANEB-Head and VANEB-FT for neural-network personalization}
\label{alg:vaneb_nn_modes}

\KwIn{Clients $[K]$; rounds $T$; participation $q$; atoms $m$; pseudo-EM iterations $S$;
EB frequency $b$; mode $M\in\{\mathrm{Head},\mathrm{FT}\}$; covariance
rule $C\in\{\mathrm{Direct},\mathrm{Coordinate}\}$}

\SetKwFunction{EB}{VANEB}
\SetKwProg{Fn}{Procedure}{:}{}

\Fn{\EB{$\mathcal A,\{\h{\b\phi}_k,\h{\Sigma}_k,n_k\}_{k\in\mathcal A}$}}{
Construct $\h{\Sigma}_k$ by \eqref{eq:nn-head-covariance}. For Direct, set
$\Sigma_k(\mb a)=\h{\Sigma}_k$; for Coordinate, fit the coordinatewise
variance functions and evaluate them at $\mb a$\;
$m\leftarrow\min(m,|\mathcal A|)$; initialize
$\{\mb a_j\}_{j=1}^{m}$ from the reported heads and set $\pi_j=1/m$\;

\For{$s=0,\ldots,S-1$}{
  \ForEach{$k\in\mathcal A,\ j\in[m]$}{
    $r_{kj}\leftarrow
    \dfrac{\pi_j\varphi_d\!\left(
      \h{\b\phi}_k;\mb a_j,\Sigma_k(\mb a_j)/n_k\right)}
    {\sum_{\ell=1}^{m}\pi_\ell\varphi_d\!\left(
      \h{\b\phi}_k;\mb a_\ell,\Sigma_k(\mb a_\ell)/n_k\right)}$\;
  }
  \ForEach{$j\in[m]$}{
    $\pi_j\leftarrow|\mathcal A|^{-1}\sum_{k\in\mathcal A}r_{kj}$,
    $\mb a_j\leftarrow
    \left[\sum_{k\in\mathcal A}r_{kj}n_k\Sigma_k(\mb a_j)^{-1}\right]^{-1}
    \sum_{k\in\mathcal A}r_{kj}n_k
    \Sigma_k(\mb a_j)^{-1}\h{\b\phi}_k$\;
    \If{$C=\mathrm{Coordinate}$}{
      Update $\{\Sigma_k(\mb a_j)\}_{k\in\mathcal A}$ at the new atom\;
    }
  }
}
Recompute $r_{kj}$ and return
$\tilde{\b \phi}_k=\sum_{j=1}^{m}r_{kj}\mb a_j$,
$k\in\mathcal A$\;
}

\eIf{$M=\mathrm{Head}$}{
Initialize $\b\psi^{(0)}$ and $\b\phi_k^{(0)}=\b\phi^{(0)}$\;

\For{$t=0,\ldots,T-1$}{
  Sample $\mathcal S_t\subset[K]$ with
  $|\mathcal S_t|=\max(1,\lfloor qK\rfloor)$ and broadcast
  $(\b\psi^{(t)},\b\phi_k^{(t)})$ to $k\in\mathcal S_t$\;

  Each selected client performs $E$ local epochs and returns
  $(\h{\b\psi}_k,\h{\b\phi}_k,\h{\Sigma}_k,n_k)$\;

  $\b\psi^{(t+1)}\leftarrow
  \dfrac{\sum_{k\in\mathcal S_t}n_k\h{\b\psi}_k}
        {\sum_{k\in\mathcal S_t}n_k}$\;

  \eIf{$t+1\equiv0\pmod b$}{
    $\{\b\phi_k^{(t+1)}\}_{k\in\mathcal S_t}
    \leftarrow
    \EB(\mathcal S_t,
    \{\h{\b\phi}_k,\h{\Sigma}_k,n_k\}_{k\in\mathcal S_t})$\;
  }{
    $\b\phi_k^{(t+1)}\leftarrow\h{\b\phi}_k$,
    $k\in\mathcal S_t$\;
  }
  Retain $\b\phi_k^{(t+1)}=\b\phi_k^{(t)}$ for
  $k\notin\mathcal S_t$\;
}
}{
Train all parameters for $T$ rounds by FedAvg, obtaining
$(\b\psi^{(T)},\b\phi^{(T)})$\;

Broadcast the final model to all clients and freeze $\b\psi^{(T)}$\;

Each client fine-tunes $\b\phi^{(T)}$ for $E_{\rm FT}$ epochs and returns
$(\h{\b\phi}_k,\h{\Sigma}_k,n_k)$\;

$\{\b\phi_k^{v}\}_{k=1}^{K}\leftarrow
\EB([K],\{\h{\b\phi}_k,\h{\Sigma}_k,n_k\}_{k=1}^{K})$\;
}

\KwOut{Shared backbone and personalized client heads}
\end{algorithm}
\endgroup

  \section{Neural network adaptations and applications}\label{sec:nn}
In this section, we extend our Variance-Aware Nonparametric Empirical Bayes
(VANEB) framework to federated training of deep neural networks. Training personalized AI models for each
user, while borrowing information from other users, has emerged as a major use case for federated learning.
Although the asymptotic normality assumption central to our theory does not
strictly hold for highly non-convex neural networks, we heuristically treat
local weight estimates as approximately Gaussian and investigate whether
VANEB remains effective in this setting.

Updating the entire neural-network weight vector through NPEB can be noisy,
computationally expensive, and communication-inefficient. We therefore use a
common body--head decomposition: a shared backbone $\b\psi$ learns the
representation, while the final layer $\b\phi_k$ is personalized for client
$k$ \citep{bengio2012deep,tripuraneni2020theory,collins2021exploiting}. The
backbone is updated using FedAvg \citep{mcmahan2017communication}, and VANEB is
applied only to the relatively lower-dimensional client heads.

We consider two versions that use the same variance-aware EB operation but at
different stages. \emph{VANEB-Head} applies the EB update intermittently during
federated training (every five communication rounds in our experiments), while
the backbone is updated in every round. \emph{VANEB-FT} first trains a common
model by FedAvg, then fine-tunes each client's head with the backbone fixed,
and finally applies one EB update across all locally adapted heads. Thus,
VANEB-Head integrates personalization into federated training, whereas
VANEB-FT uses EB to denoise and share information among the final locally
fine-tuned heads. Algorithm~\ref{alg:vaneb_nn_modes} presents both versions in
a common form.

For either version, client $k$ reports a diagonal covariance estimate based on an approximation to the observed Fisher information matrix as, 
\begin{equation}\label{eq:nn-head-covariance}
\h{\Sigma}_k=\diag\bxs{\lrp{\frac{1}{n_k}\sum_{i=1}^{n_k}
\mb g_{k,i}\mb g_{k,i}^{\top}+\lambda\mb I}^{-1}},
\end{equation}
where $\mb g_{k,i}$ is the gradient of the head loss for observation $i$.
Note that, for communication efficiency and to remain closer to our theoretical setup, we retain only the diagonal elements from the full covariance matrix. As in
\cite{jhunjhunwala2024fedfisher}, this provides a Fisher information based approximation
to local uncertainty. We consider two ways of supplying the covariance to
pseudo-EM. The \emph{direct} construction sets
$\Sigma_k(\mb a_j)=\h{\Sigma}_k$ at every candidate atom, whereas the
\emph{coordinatewise} construction estimates, separately for each head
coordinate, a low-order variance function and evaluates it at $\mb a_j$.
Crossing these two constructions with the two training stages gives the four
reported variants: VANEB-Direct FT, VANEB-Coordinate FT, VANEB-Direct Head,
and VANEB-Coordinate Head. We do note that there are multiple consistent estimators of the Fisher information matrix, and in our results are not sensitive to the particular estimator used.

\begin{table}[!h]
  \centering
  \caption{Final personalized test accuracy (\%) on MNIST.
  The total training sample size is fixed at 50,000 and divided among the
  indicated number of clients. Bold denotes the highest result in each column
  and evaluation protocol.}
  \label{tab:mnist-personalized-nine-column}
  \begingroup
  \scriptsize
  \setlength{\tabcolsep}{2.3pt}
  \renewcommand{\arraystretch}{0.82}
  \resizebox{\textwidth}{!}{%
  \begin{tabular}{l*{9}{r}}
    \toprule
    & \multicolumn{3}{c}{$\alpha=0.01$}
    & \multicolumn{3}{c}{$\alpha=0.1$}
    & \multicolumn{3}{c}{$\alpha=0.5$} \\
    \cmidrule(lr){2-4}\cmidrule(lr){5-7}\cmidrule(lr){8-10}
    Method & 100 & 200 & 500 & 100 & 200 & 500 & 100 & 200 & 500 \\
    \midrule
    \multicolumn{10}{l}{\textit{Panel A: Balanced test distribution}} \\
    \addlinespace[1pt]
    FedAvg & 80.51 & 79.86 & 70.74 & 89.22 & 83.27 & 75.18 & 90.46 & 87.16 & 74.36 \\
    FedAvg+FT & 49.49 & 39.59 & 33.61 & 71.05 & 57.68 & 46.41 & 84.57 & 75.30 & 59.99 \\
    FedBABU & 49.69 & 34.12 & 22.25 & 65.23 & 50.98 & 33.87 & 80.55 & 69.73 & 49.46 \\
    FedPer & 26.09 & 30.02 & 28.20 & 59.04 & 54.67 & 47.34 & 81.36 & 73.50 & 61.10 \\
    FedRep & 12.29 & 11.41 & 10.60 & 36.23 & 25.83 & 14.68 & 66.14 & 51.80 & 23.71 \\
    LG-FedAvg & 59.77 & 50.13 & 46.42 & 70.67 & 57.29 & 48.12 & 83.19 & 70.83 & 56.70 \\
    Per-FedAvg & 11.15 & 10.66 & 10.50 & 17.90 & 14.55 & 12.42 & 27.46 & 17.93 & 12.39 \\
    SCAFFOLD & 73.61 & 75.42 & 68.30 & 82.51 & 78.06 & 69.56 & 85.61 & 81.32 & 67.38 \\
    Ditto & 13.40 & 11.42 & 10.31 & 24.20 & 16.25 & 11.40 & 44.16 & 22.98 & 12.81 \\
    Local Only & 11.52 & 11.19 & 10.26 & 20.28 & 15.18 & 11.14 & 31.58 & 18.40 & 11.91 \\
    VANEB-FT (direct) & 69.30 & 73.81 & \textbf{75.84} & 86.94 & 82.65 & \textbf{77.25} & 92.10 & 87.16 & \textbf{77.11} \\
    VANEB-FT (coordinate) & \textbf{83.60} & \textbf{82.10} & 74.77 & \textbf{90.91} & \textbf{84.95} & 76.50 & \textbf{92.34} & \textbf{87.55} & 76.18 \\
    VANEB-Head (direct) & 31.34 & 44.52 & 46.68 & 59.08 & 61.24 & 60.40 & 81.88 & 77.92 & 66.82 \\
    VANEB-Head (coordinate) & 37.35 & 47.45 & 44.26 & 67.13 & 64.94 & 58.26 & 84.12 & 78.57 & 66.08 \\
    \midrule
    \multicolumn{10}{l}{\textit{Panel B: Client-matched test distribution}} \\
    \addlinespace[1pt]
    FedAvg & 80.77 & 80.16 & 71.57 & 89.32 & 84.19 & 74.25 & 90.45 & 87.06 & 75.03 \\
    FedAvg+FT & 96.82 & 92.99 & 82.06 & 95.81 & \textbf{91.85} & \textbf{79.06} & \textbf{93.83} & \textbf{89.47} & 75.67 \\
    FedBABU & 96.26 & 92.15 & 79.05 & 95.04 & 90.37 & 74.93 & 93.10 & 87.29 & 70.18 \\
    FedPer & 95.16 & 91.40 & 80.81 & 94.25 & 90.33 & 78.53 & 92.95 & 87.94 & 75.45 \\
    FedRep & 93.71 & 88.00 & 74.03 & 91.41 & 84.11 & 61.66 & 88.44 & 80.31 & 51.57 \\
    LG-FedAvg & \textbf{97.37} & \textbf{93.97} & \textbf{84.55} & \textbf{96.12} & 91.09 & 79.06 & 93.74 & 86.94 & 73.08 \\
    Per-FedAvg & 93.02 & 86.23 & 73.77 & 81.96 & 72.67 & 58.03 & 67.91 & 51.83 & 36.31 \\
    SCAFFOLD & 73.52 & 75.86 & 67.79 & 82.87 & 78.49 & 68.77 & 85.80 & 81.18 & 68.04 \\
    Ditto & 94.04 & 87.18 & 73.30 & 87.05 & 74.95 & 55.17 & 78.63 & 57.82 & 35.57 \\
    Local Only & 93.58 & 87.88 & 73.38 & 85.80 & 73.86 & 55.35 & 72.22 & 52.62 & 35.15 \\
    VANEB-FT (direct) & 93.99 & 90.25 & 79.08 & 94.07 & 88.77 & 76.94 & 93.31 & 87.96 & \textbf{77.61} \\
    VANEB-FT (coordinate) & 83.63 & 82.61 & 74.53 & 91.00 & 85.30 & 75.59 & 92.86 & 87.75 & 76.49 \\
    VANEB-Head (direct) & 95.11 & 91.04 & 80.81 & 94.30 & 90.24 & 78.79 & 93.24 & 87.93 & 75.64 \\
   VANEB-Head (coordinate) & 91.44 & 87.56 & 79.53 & 92.78 & 88.60 & 77.86 & 92.93 & 87.25 & 75.12 \\
    \bottomrule
  \end{tabular}%
  }
  \endgroup
\end{table}

\begin{table*}[!h]
  \centering
  \caption{Personalized CIFAR-10 test accuracy (\%) with total sample size fixed at 50,000.}
  \label{tab:cifar10-publication9}
  \begingroup
  \scriptsize
  \setlength{\tabcolsep}{2.2pt}
  \renewcommand{\arraystretch}{0.86}
  \resizebox{\textwidth}{!}{%
  \begin{tabular}{lrrrrrrrrr}
    \toprule
    & \multicolumn{3}{c}{$\alpha=0.01$}
    & \multicolumn{3}{c}{$\alpha=0.1$}
    & \multicolumn{3}{c}{$\alpha=0.5$} \\
    \cmidrule(lr){2-4}\cmidrule(lr){5-7}\cmidrule(lr){8-10}
    Method & 100 & 200 & 500 & 100 & 200 & 500 & 100 & 200 & 500 \\
    \midrule
    \multicolumn{10}{l}{\textit{Panel A: Balanced test distribution}} \\
    \addlinespace[1pt]
    FedAvg & 30.14 & 34.64 & 32.86 & 47.33 & 40.95 & 37.11 & 50.69 & 47.49 & 37.52 \\
    FedAvg+FT & 13.27 & 14.25 & 15.08 & 24.34 & 22.29 & 20.39 & 39.00 & 33.08 & 27.95 \\
    FedBABU & 14.94 & 14.96 & 15.02 & 24.34 & 22.00 & 19.49 & 34.96 & 30.06 & 25.09 \\
    FedPer & 13.63 & 13.66 & 14.61 & 23.19 & 19.90 & 18.81 & 35.68 & 30.32 & 25.00 \\
    FedRep & 12.36 & 12.26 & 13.61 & 21.46 & 20.20 & 18.87 & 35.38 & 30.33 & 24.74 \\
    LG-FedAvg & 18.27 & 18.18 & 18.08 & 28.59 & 25.95 & 23.00 & 42.56 & 35.52 & 30.69 \\
    Per-FedAvg & 10.92 & 10.50 & 10.29 & 14.09 & 12.95 & 11.82 & 17.30 & 15.99 & 13.30 \\
    SCAFFOLD & 17.78 & 19.71 & 14.35 & 26.92 & 21.61 & 14.61 & 29.36 & 24.77 & 14.42 \\
    Ditto & 16.34 & 14.55 & 12.55 & 25.16 & 20.79 & 16.00 & 36.77 & 28.95 & 20.08 \\
    Local Only & 11.15 & 10.60 & 10.26 & 14.99 & 13.19 & 11.16 & 18.62 & 15.82 & 12.12 \\
   VANEB-FT (coordinate) & \textbf{39.40} & \textbf{38.16} & \textbf{34.63} & \textbf{48.36} & \textbf{43.86} & 37.93 & \textbf{55.10} & 47.73 & 38.52 \\
    VANEB-FT (direct) & 20.89 & 21.82 & 34.20 & 36.69 & 38.39 & \textbf{38.45} & 48.80 & \textbf{48.29} & \textbf{38.85} \\
    VANEB-Head (coordinate) & 18.67 & 19.29 & 20.24 & 29.02 & 25.85 & 23.98 & 40.17 & 35.06 & 28.76 \\
    VANEB-Head (direct) & 14.16 & 14.94 & 18.32 & 23.07 & 21.90 & 24.18 & 35.73 & 33.01 & 29.18 \\
    \midrule
    \multicolumn{10}{l}{\textit{Panel B: Client-matched test distribution}} \\
    \addlinespace[1pt]
    FedAvg & 29.93 & 34.73 & 33.24 & 47.46 & 40.80 & 36.79 & 50.54 & 47.32 & 37.40 \\
    FedAvg+FT & 93.11 & 88.00 & 75.34 & 81.69 & \textbf{76.54} & \textbf{64.11} & 69.60 & 61.98 & 49.79 \\
    FedBABU & 92.34 & 87.31 & 74.75 & 78.28 & 72.80 & 62.12 & 64.18 & 57.13 & 45.92 \\
    FedPer & 92.73 & 87.49 & 74.22 & 79.74 & 74.66 & 61.85 & 66.74 & 58.49 & 46.88 \\
    FedRep & 92.93 & 87.82 & 74.29 & 81.29 & 75.74 & 61.27 & 68.36 & 60.53 & 44.52 \\
    LG-FedAvg & \textbf{93.40} & \textbf{88.09} & \textbf{75.41} & \textbf{81.82} & 76.51 & 64.07 & \textbf{69.66} & \textbf{62.35} & \textbf{50.09} \\
    Per-FedAvg & 91.77 & 86.04 & 73.35 & 71.34 & 67.51 & 55.73 & 49.49 & 45.81 & 35.99 \\
    SCAFFOLD & 17.56 & 19.61 & 14.42 & 26.43 & 21.69 & 14.56 & 28.99 & 24.64 & 14.51 \\
    Ditto & 92.00 & 86.88 & 73.78 & 78.71 & 72.09 & 58.97 & 64.47 & 56.88 & 41.73 \\
    Local Only & 92.01 & 86.14 & 73.38 & 73.61 & 67.98 & 55.28 & 51.28 & 46.17 & 34.63 \\
    VANEB-FT (coordinate) & 45.12 & 38.01 & 34.36 & 54.32 & 43.91 & 37.44 & 55.00 & 48.05 & 38.97 \\
    VANEB-FT (direct) & 85.68 & 79.51 & 42.96 & 70.56 & 62.57 & 38.07 & 62.60 & 48.52 & 39.14 \\
    VANEB-Head (coordinate) & 80.68 & 73.86 & 64.29 & 74.56 & 66.78 & 55.25 & 63.42 & 55.48 & 44.09 \\
    VANEB-Head (direct) & 92.27 & 85.76 & 71.04 & 80.08 & 73.42 & 56.11 & 67.03 & 57.93 & 44.29 \\
    \bottomrule
  \end{tabular}%
  }
  \endgroup
\end{table*}

\subsection{CNN on MNIST and CIFAR-10 vision datasets}

We compare personalized CNN performance on MNIST and CIFAR-10. The MNIST
network has two convolutional and two fully connected layers, with
\texttt{fc2} as the personalized head. The CIFAR-10 network has three
convolutional blocks followed by three fully connected layers, with
\texttt{fc3} as the personalized head. In both cases the remaining layers form
the shared backbone. We compare two variants of VANEB-Head and VANEB-FT with FedAvg, FedAvg followed by head fine-tuning, FedRep \citep{collins2021exploiting}, FedBABU
\citep{oh2021fedbabu}, FedPer \citep{arivazhagan2019federated}, LG-FedAvg
\citep{liang2020think}, Per-FedAvg \citep{fallah2020personalized}, SCAFFOLD
\citep{karimireddy2020scaffold}, Ditto \citep{li2021ditto}, and local-only
training.

We generate joint label and covariate heterogeneity. Client training-label
proportions are drawn from a Dirichlet distribution, with smaller $\alpha$
giving stronger label skew. Clients are also assigned to transformation
clusters: rotations, translations, blur, and contrast changes are used for
MNIST, while color jitter, blur, and affine perturbations are used for
CIFAR-10. Holdout images are disjoint from training images and use
deterministic transforms. We report two personalized evaluation protocols.
The \emph{balanced} protocol represents every class equally in each client's
test distribution and measures whether personalization preserves broad
predictive coverage. The \emph{client-matched} protocol follows that client's
training-label proportions and measures specialization to its local
distribution.

All experiments use 10\% client participation. For MNIST data, we use two local epochs and 50
communication rounds, while for CIFAR-10 data we use 100 communication rounds and one local epoch. The LG-Fedavg additionally used 100-round LG stage. For each value $\alpha\in\{0.01,0.1,0.5\}$, the total
training sample size is fixed at 50,000 and divided among 100, 200, or 500
clients. The VANEB procedures use 10 atoms. The local-only results are for local training with effective compute match. Bold denotes the best result within
each column and test protocol.

\begin{table*}[!h]
  \centering
  \caption{Final personalized test accuracy (\%) on MNIST under joint
  Dirichlet label skew ($\alpha=0.1$) and clustered covariate shift.}
  \label{tab:mnist-transform-ablation-main}
  \begingroup
  \tiny
  \setlength{\tabcolsep}{2.0pt}
  \renewcommand{\arraystretch}{0.82}
  \resizebox{\textwidth}{!}{%
  \begin{tabular}{l*{12}{r}}
    \toprule
    & \multicolumn{4}{c}{No covariate shift ($s=0$)}
    & \multicolumn{4}{c}{Default shift ($s=1$)}
    & \multicolumn{4}{c}{Strong shift ($s=2$)} \\
    \cmidrule(lr){2-5}\cmidrule(lr){6-9}\cmidrule(lr){10-13}
    Method & 50 & 100 & 200 & 500 & 50 & 100 & 200 & 500 & 50 & 100 & 200 & 500 \\
    \midrule
    \multicolumn{13}{l}{\textit{Panel A: Balanced test distribution}} \\
    \addlinespace[1pt]
    FedAvg & 92.24 & 92.82 & 89.62 & 83.18 & 86.97 & 89.22 & 83.27 & 75.18 & 74.67 & 79.03 & 72.02 & 60.89 \\
    FedAvg+FT & 86.80 & 80.42 & 70.45 & 59.82 & 77.53 & 71.05 & 57.68 & 46.41 & 62.08 & 54.32 & 43.02 & 30.23 \\
    FedBABU & 78.00 & 74.58 & 64.40 & 46.02 & 69.49 & 65.31 & 50.58 & 33.51 & 52.85 & 49.59 & 37.77 & 21.51 \\
    FedPer & 74.27 & 71.34 & 66.32 & 59.93 & 64.12 & 58.92 & 54.52 & 47.44 & 50.69 & 44.35 & 40.62 & 29.58 \\
    FedRep & 54.77 & 44.92 & 30.96 & 16.39 & 47.47 & 36.24 & 25.89 & 14.62 & 37.93 & 27.59 & 20.56 & 12.83 \\
    LG-FedAvg & 86.48 & 78.80 & 67.91 & 59.71 & 78.83 & 70.23 & 57.36 & 48.25 & 67.76 & 57.57 & 46.76 & 34.21 \\
    Per-FedAvg & 25.13 & 19.55 & 15.35 & 13.30 & 23.53 & 17.88 & 14.49 & 12.40 & 21.19 & 15.97 & 13.53 & 11.60 \\
    SCAFFOLD & 91.39 & 89.58 & 85.65 & 79.24 & 84.84 & 82.60 & 78.20 & 69.68 & 74.64 & 70.09 & 64.53 & 54.76 \\
    Ditto & 58.18 & 27.21 & 17.34 & 11.73 & 50.49 & 24.37 & 16.33 & 11.42 & 41.93 & 21.11 & 14.88 & 11.20 \\
    Local Only & 27.84 & 21.64 & 15.80 & 11.36 & 26.04 & 20.18 & 15.20 & 11.21 & 23.56 & 18.09 & 14.30 & 11.05 \\
    VANEB-Direct FT & 93.66 & 93.56 & 90.34 & \textbf{85.24} & 87.45 & 86.94 & 82.65 & \textbf{77.25} & 72.62 & 73.40 & 68.33 & \textbf{63.28} \\
    VANEB-Coordinate FT & \textbf{95.73} & \textbf{94.05} & \textbf{90.84} & 84.79 & \textbf{91.95} & \textbf{90.91} & \textbf{84.95} & 76.50 & \textbf{77.79} & \textbf{81.18} & \textbf{73.72} & 62.22 \\
    VANEB-Direct Head & 75.27 & 71.63 & 73.25 & 70.65 & 64.05 & 59.08 & 61.24 & 60.40 & 52.22 & 44.37 & 46.42 & 44.33 \\
    VANEB-Coordinate Head & 81.35 & 77.62 & 74.78 & 69.03 & 70.61 & 67.13 & 64.94 & 58.26 & 56.56 & 53.21 & 49.80 & 41.85 \\
    \midrule
    \multicolumn{13}{l}{\textit{Panel B: Client-matched test distribution}} \\
    \addlinespace[1pt]
    FedAvg & 92.08 & 92.76 & 90.00 & 82.80 & 86.91 & 89.32 & 84.19 & 74.25 & 76.78 & 79.77 & 72.85 & 60.33 \\
    FedAvg+FT & 98.23 & \textbf{97.31} & \textbf{94.41} & \textbf{85.48} & 97.24 & 95.81 & \textbf{91.85} & \textbf{79.06} & 94.68 & 92.80 & 86.79 & 70.07 \\
    FedBABU & 97.60 & 96.89 & 93.11 & 80.97 & 96.51 & 95.04 & 90.37 & 74.93 & 93.22 & 91.53 & 85.60 & 66.42 \\
    FedPer & 97.80 & 96.27 & 93.42 & 84.17 & 96.76 & 94.25 & 90.33 & 78.53 & 93.75 & 90.28 & 84.53 & 69.39 \\
    FedRep & 96.92 & 93.79 & 87.08 & 64.74 & 95.37 & 91.41 & 84.11 & 61.66 & 92.39 & 86.87 & 78.63 & 58.28 \\
    LG-FedAvg & \textbf{98.60} & 97.24 & 93.25 & 84.34 & \textbf{97.78} & \textbf{96.12} & 91.09 & \textbf{79.06} & \textbf{96.80} & \textbf{94.21} & \textbf{87.43} & \textbf{71.42} \\
    Per-FedAvg & 91.56 & 85.64 & 74.13 & 59.87 & 89.42 & 81.96 & 72.67 & 58.03 & 85.05 & 77.80 & 70.10 & 56.38 \\
    SCAFFOLD & 91.10 & 89.75 & 85.83 & 78.86 & 84.61 & 82.87 & 78.49 & 68.77 & 75.83 & 70.51 & 64.74 & 53.81 \\
    Ditto & 96.20 & 89.07 & 77.29 & 55.96 & 94.13 & 87.05 & 74.95 & 55.17 & 91.70 & 83.11 & 72.27 & 54.68 \\
    Local Only & 93.68 & 88.02 & 75.37 & 55.98 & 92.05 & 85.80 & 73.86 & 55.35 & 88.20 & 80.83 & 71.64 & 54.87 \\
    VANEB-Direct FT & 97.38 & 95.46 & 92.36 & 84.87 & 95.99 & 94.07 & 88.77 & 76.94 & 91.38 & 89.36 & 82.76 & 64.73 \\
    VANEB-Coordinate FT & 95.53 & 93.79 & 91.20 & 84.47 & 93.85 & 91.00 & 85.30 & 75.59 & 83.53 & 82.41 & 74.92 & 62.00 \\
    VANEB-Direct Head & 97.94 & 96.22 & 93.49 & 85.11 & 96.78 & 94.30 & 90.24 & 78.79 & 93.80 & 89.98 & 83.88 & 69.01 \\
    VANEB-Coordinate Head & 97.39 & 95.54 & 92.57 & 83.89 & 95.43 & 92.78 & 88.60 & 77.86 & 91.65 & 88.21 & 79.36 & 67.65 \\
    \bottomrule
  \end{tabular}%
  }
  \endgroup
\end{table*}

Tables~\ref{tab:mnist-personalized-nine-column} and
\ref{tab:cifar10-publication9} show a clear distinction between the
two evaluation protocols. Under client-matched evaluation, methods based on
independent head adaptation can perform very well because the test labels
follow the same skewed proportions observed during local training. Balanced
evaluation is more demanding: every personalized model must predict the full
label space, including labels that are rare at that client. This also explains
why local-only training can appear strong under severe matched label skew but
degrades sharply on balanced tests.

The VANEB-FT variants give the strongest and most consistent balanced
performance. Across MNIST and CIFAR-10, either Direct or Coordinate FT is best
in every balanced column. More broadly, VANEB-Coordinate FT outperforms FedAvg
in every reported comparison: all nine settings under both balanced and
client-matched evaluation on both datasets (36 of 36 comparisons). However, VANEB-FT Coordinate is not always the overall best-performing method, especially in client-matched tests, since FedAvg can have degraded performance there. On client-matched tests,
FedAvg with fine-tuning, LG-FedAvg, and related local-adaptation methods are
often stronger because their test labels reproduce the client's local label
mixture. Coordinate FT is especially effective with fewer clients, whereas
Direct FT is more stable as the client population grows. VANEB-Head provides a
complementary in-training procedure without a separate terminal fine-tuning
stage and generally improves substantially over local-only training.

\begin{figure*}[!h]
    \centering
    \begin{subfigure}[t]{0.485\textwidth}
        \centering
        \includegraphics[width=\linewidth]{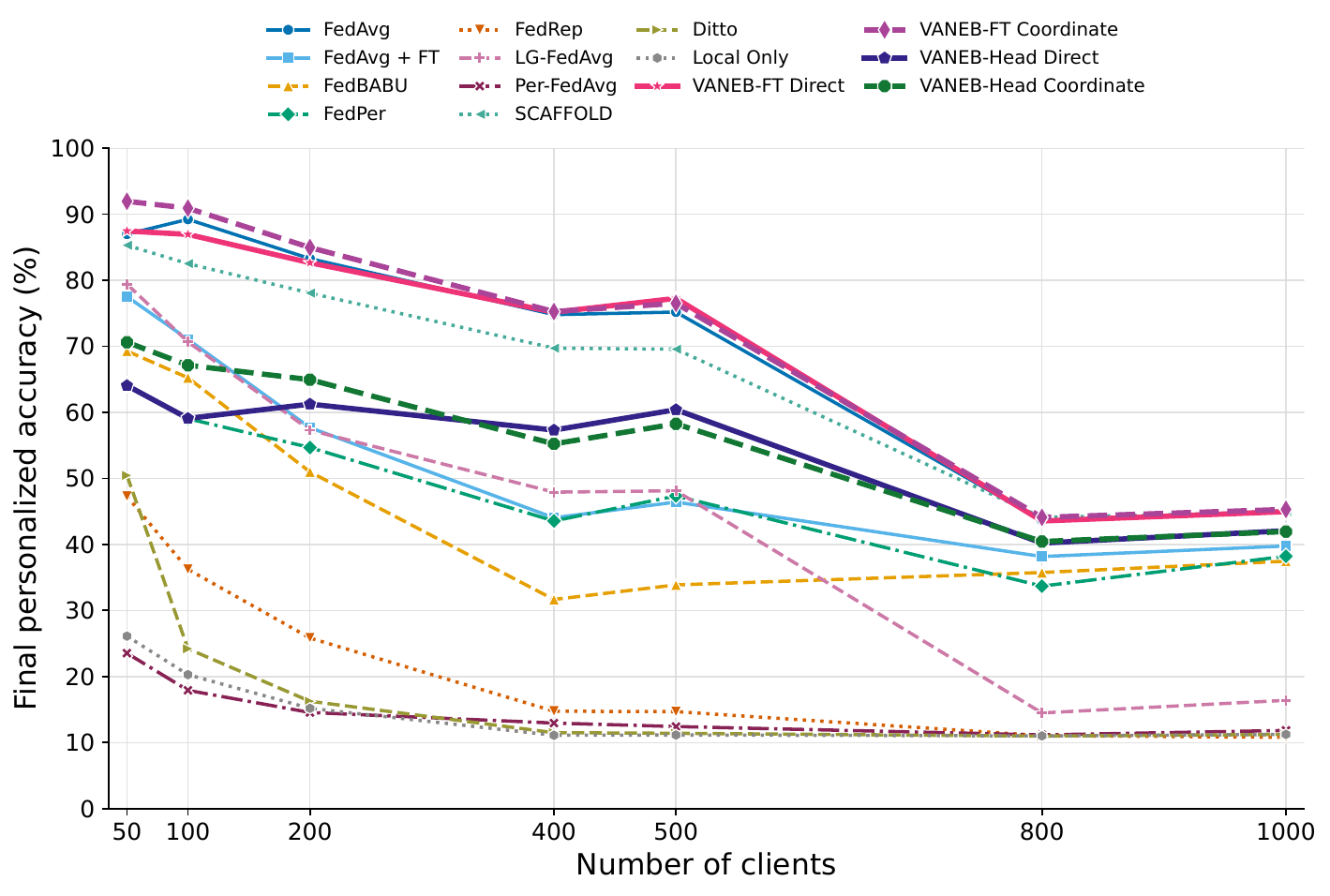}
        \caption{$\alpha=0.1$: balanced test distribution.}
        \label{fig:mnist-client-scaling-a01-balanced}
    \end{subfigure}
    \hfill
    \begin{subfigure}[t]{0.485\textwidth}
        \centering
        \includegraphics[width=\linewidth]{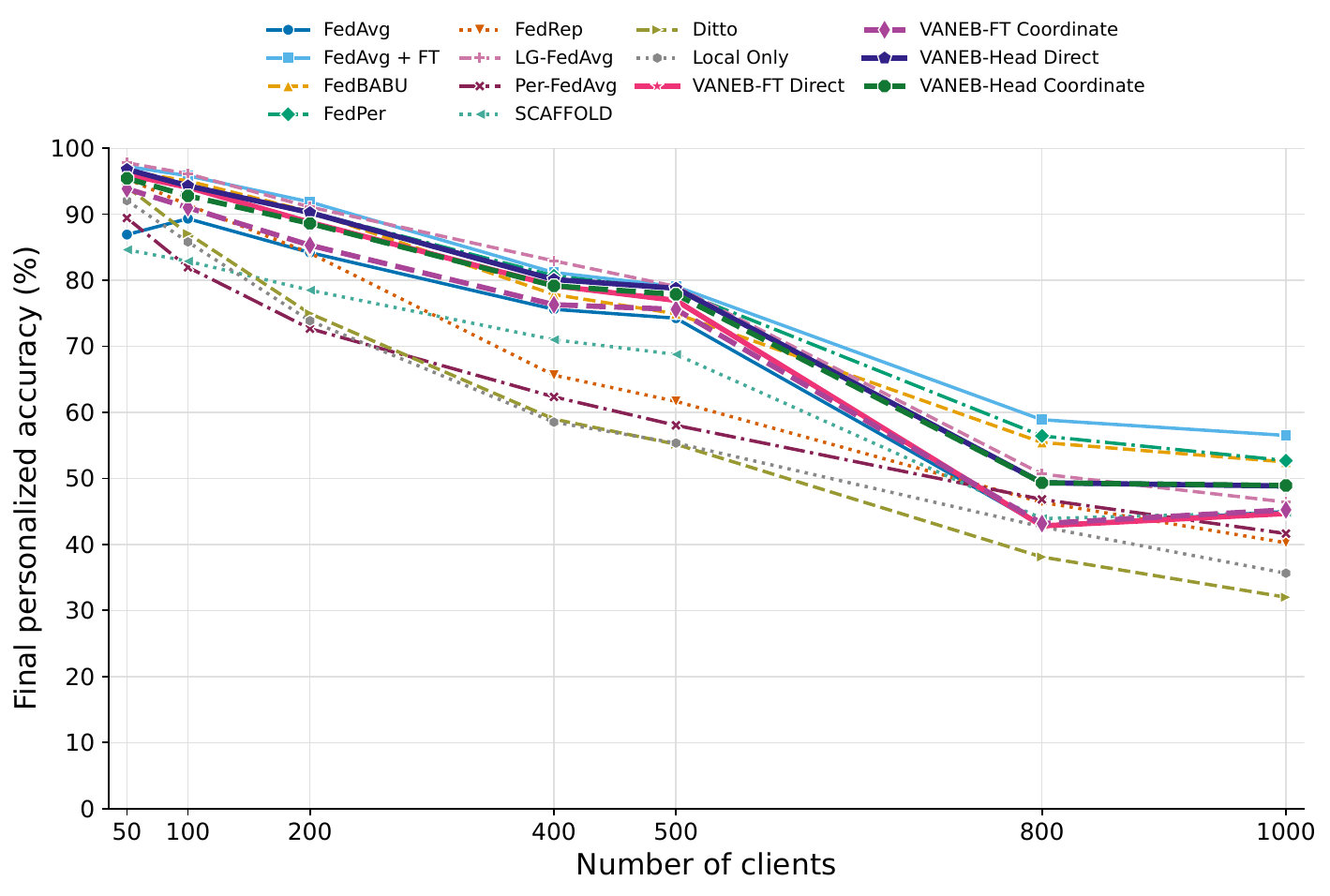}
        \caption{$\alpha=0.1$: client-matched test distribution.}
        \label{fig:mnist-client-scaling-a01-matched}
    \end{subfigure}

    \vspace{0.5em}

    \begin{subfigure}[t]{0.485\textwidth}
        \centering
        \includegraphics[width=\linewidth]{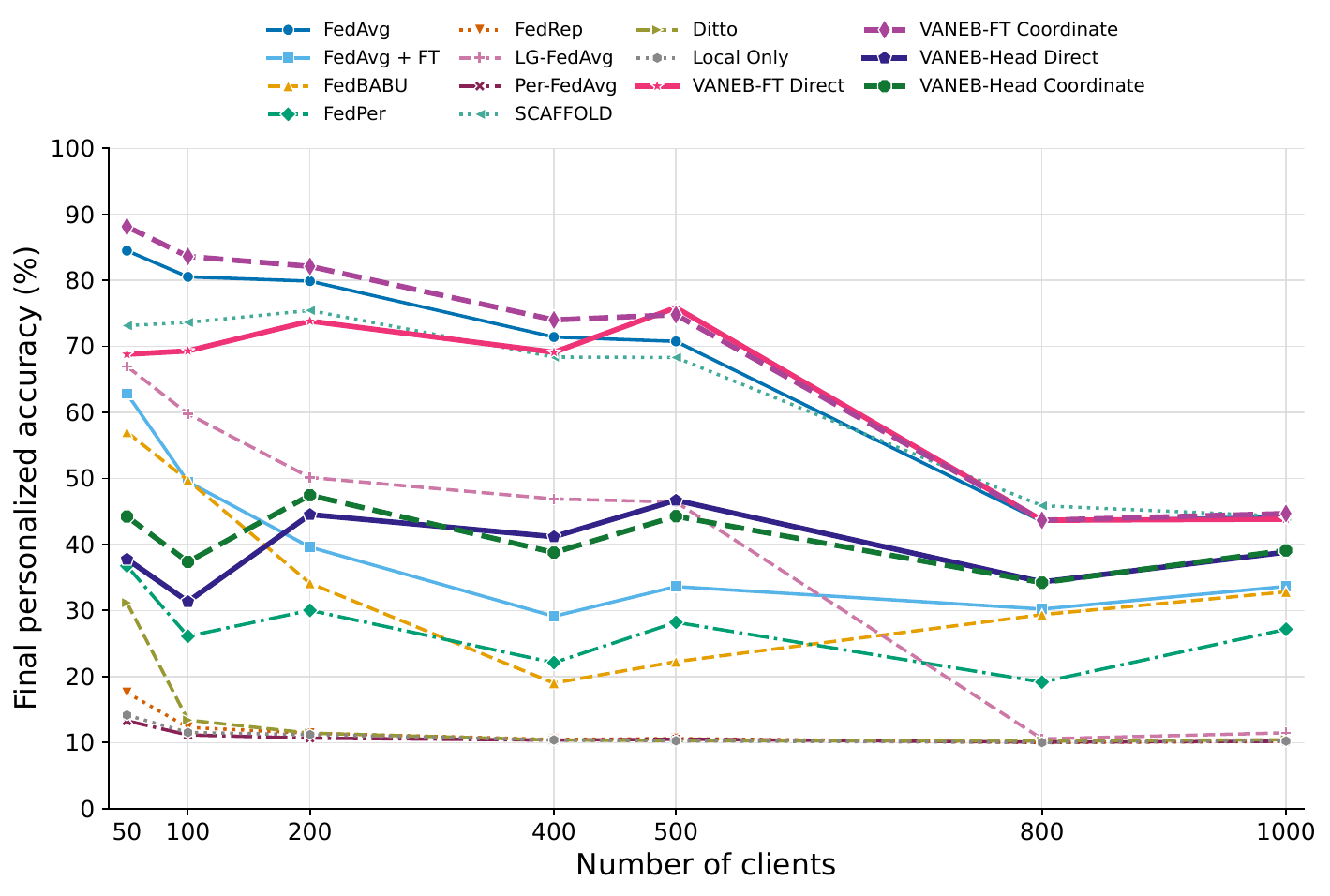}
        \caption{$\alpha=0.01$: balanced test distribution.}
        \label{fig:mnist-client-scaling-a001-balanced}
    \end{subfigure}
    \hfill
    \begin{subfigure}[t]{0.485\textwidth}
        \centering
        \includegraphics[width=\linewidth]{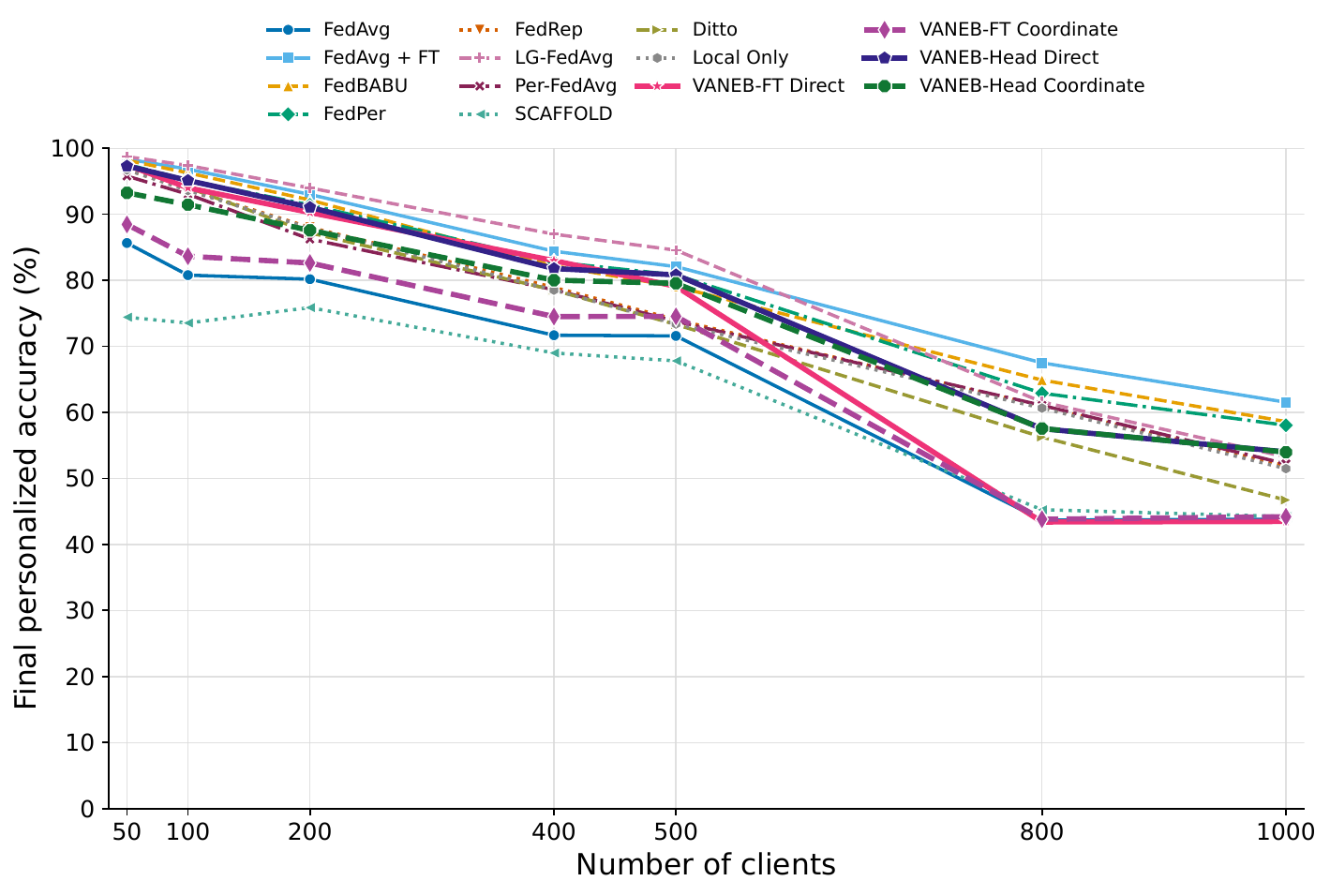}
        \caption{$\alpha=0.01$: client-matched test distribution.}
        \label{fig:mnist-client-scaling-a001-matched}
    \end{subfigure}

    \caption{Final personalized accuracy on MNIST as the number of clients
    increases while the total federated sample size remains fixed.}
    \label{fig:mnist-fixed-total-client-scaling}
\end{figure*}

To separate label heterogeneity from clustered feature heterogeneity,
Table~\ref{tab:mnist-transform-ablation-main} varies the strength $s$ of the
client-cluster transformations at $\alpha=0.1$. Increasing $s$ lowers accuracy
for every method and makes the decline sharper when the fixed sample is split
among more clients. The VANEB-FT variants remain strongest throughout the
balanced protocol: Coordinate FT is best in nine of the 12 settings, while
Direct FT is best for 500 clients at each transformation strength. This
complementarity agrees with the main tables and motivates reporting both
covariance constructions.

Figure~\ref{fig:mnist-fixed-total-client-scaling} examines the effect of
increasing the number of clients while holding the total amount of data fixed.
As the client population grows from 50 to 1,000, the number of observations
available to each client decreases, making both local estimation and
personalization progressively more difficult. This deterioration becomes
particularly pronounced beyond 500 clients.

Under balanced evaluation, the VANEB-FT variants are strongest or nearly
strongest across the client scales for both levels of label skew. Coordinate
FT is generally strongest when each client has more data, while Direct FT is
more stable after the fixed sample is divided among many clients. These
results indicate that pooling locally adapted heads through the estimated
population distribution remains effective as the client data become
increasingly fragmented.

The client-matched panels show a different pattern. Because their test-label
proportions follow local training distributions, methods based on aggressive
local head adaptation, particularly FedAvg with fine-tuning and LG-FedAvg,
often perform best. The contrast between the balanced and matched panels
illustrates the distinction between specialization to a client's observed
label mixture and retaining predictive coverage across the complete label
space.

Figure~\ref{fig:mnist-prior-comparison} visualizes the estimated priors (the first two principal components) for
the two VANEB-FT covariance constructions at 50 and 500 clients. The gray
points are the locally fine-tuned heads, the red stars are the final prior
atoms, and the blue points are the posterior-mean heads; the other panels show
the estimated atom masses and the concentration of client responsibilities.
With 50 clients, Direct retains ten separated atoms and visibly distinct
posterior heads, whereas Coordinate concentrates most mass on two atoms and
produces stronger shrinkage. With 500 clients, Coordinate becomes essentially
single-atom, while Direct retains diffuse prior mass across the ten atoms.
Nevertheless, Direct responsibilities are then sufficiently diffuse that its
posterior means also become very similar. Thus the direct construction
preserves a richer estimated prior, while the coordinatewise construction
induces more aggressive pooling as the federation becomes increasingly
fragmented.

\begin{figure*}[!t]
  \centering
  \begin{subfigure}[t]{0.88\textwidth}
    \centering
    \includegraphics[width=\linewidth]{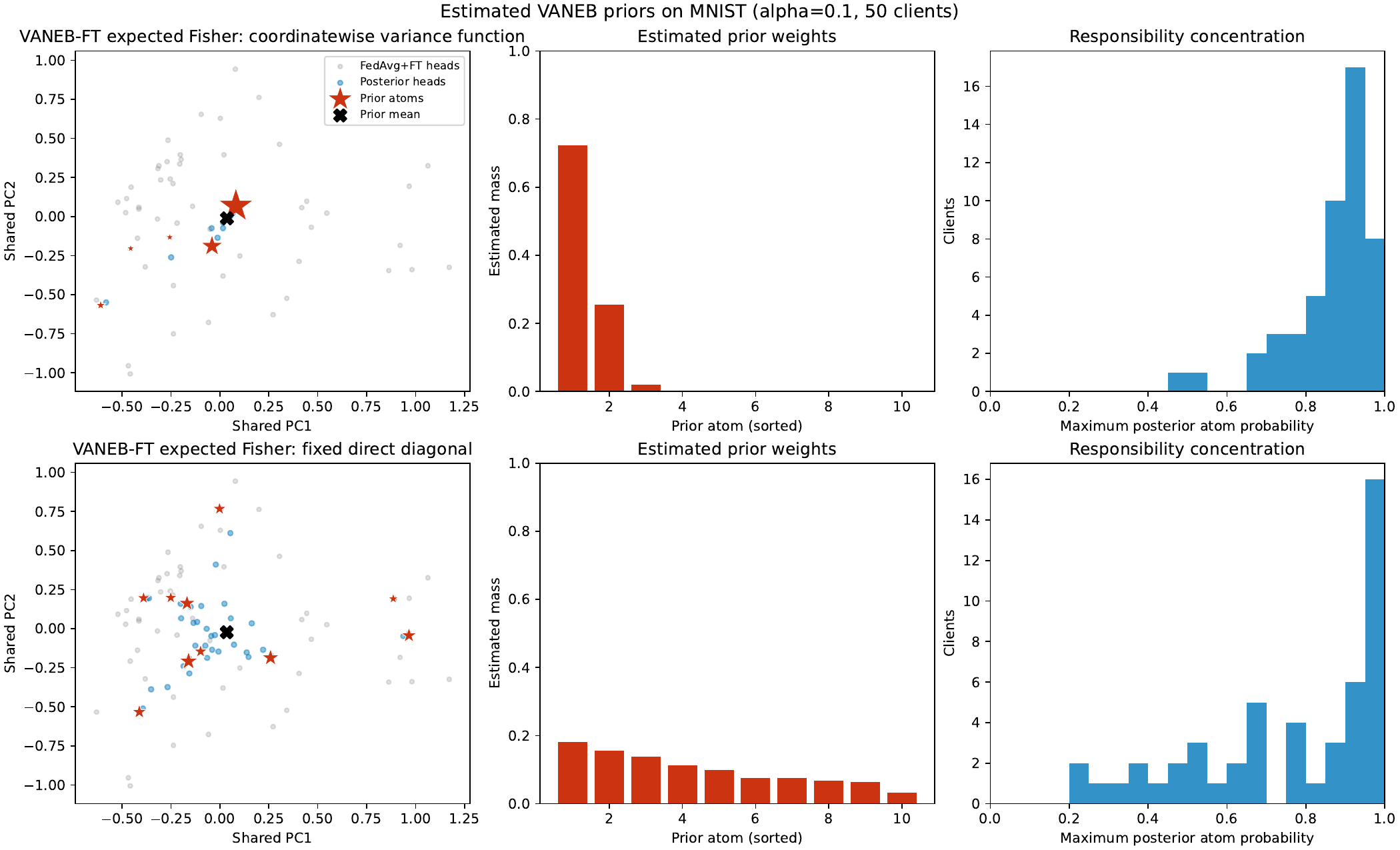}
    \caption{50 clients (1,000 training observations per client).}
    \label{fig:mnist-prior-comparison-50}
  \end{subfigure}

  \vspace{0.5em}

  \begin{subfigure}[t]{0.88\textwidth}
    \centering
    \includegraphics[width=\linewidth]{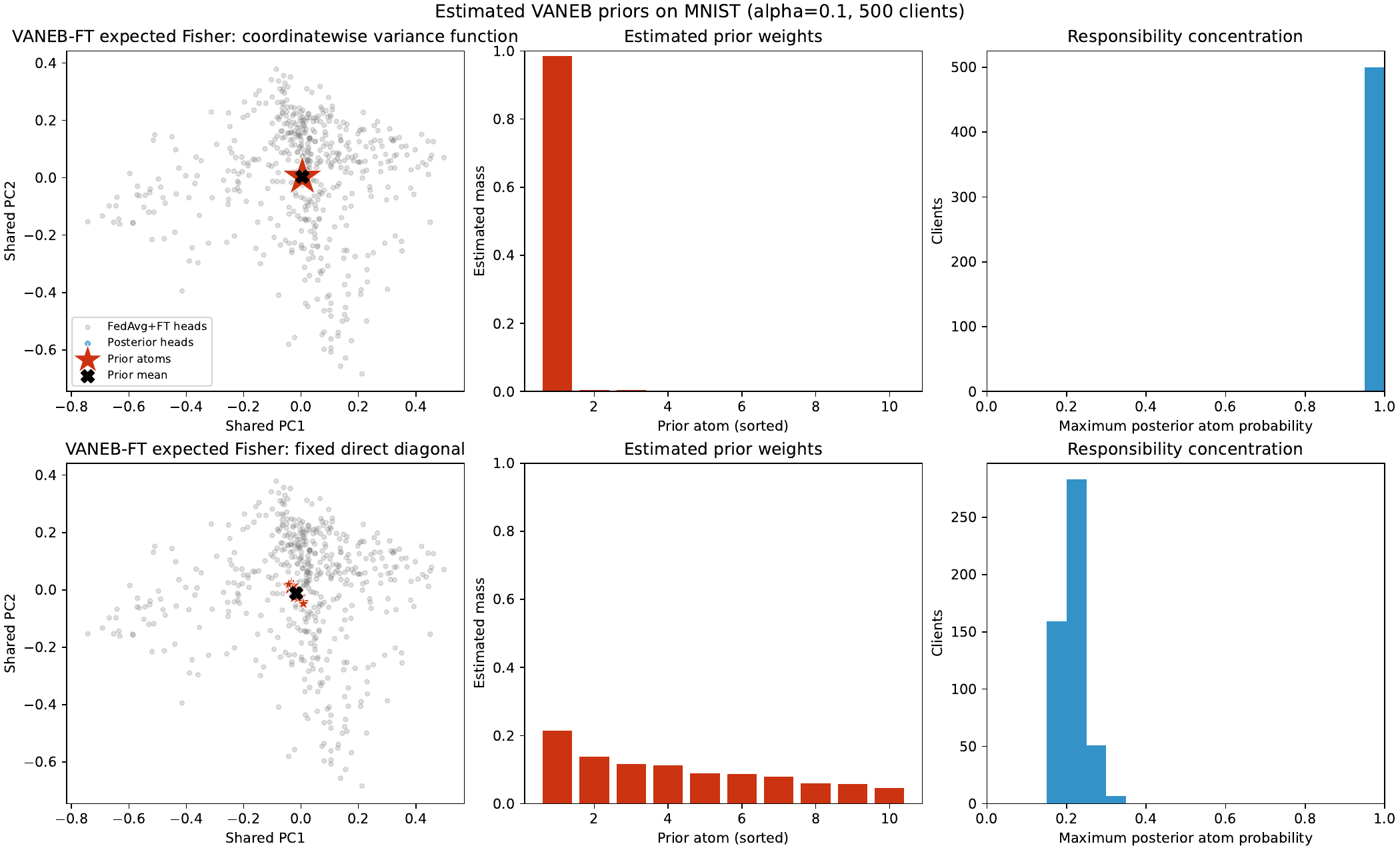}
    \caption{500 clients (100 training observations per client).}
    \label{fig:mnist-prior-comparison-500}
  \end{subfigure}
  \caption{Estimated VANEB-FT priors on MNIST with $\alpha=0.1$ and a fixed
  total training sample size of 50,000. The top and bottom rows within each
  subfigure show the coordinatewise and direct covariance constructions,
  respectively.}
  \label{fig:mnist-prior-comparison}
\end{figure*}

Together, the two versions show how the same EB principle can be used either
during federated training or after local adaptation. VANEB-Head continually
shares information among uncertain client heads, while VANEB-FT applies a
final population-level denoising step to independently fine-tuned heads. The
balanced results indicate that this shrinkage prevents local adaptation from
discarding information learned from the broader federation, allowing VANEB to
personalize client models without sacrificing global predictive coverage.

\section{Discussion}
In this paper, we introduced the Variance-Aware Nonparametric Empirical Bayes (VANEB) framework for personalized federated learning. By deriving a generalized Tweedie formula that accommodates parameter-dependent covariance structures, we remove the known-covariance restriction of existing NPEB methods. Our analysis provides non-asymptotic guarantees for the NPMLE in the average squared Hellinger distance and an oracle denoising inequality, covering M-estimators and generalized linear models (GLMs) under a diagonal covariance regime.

Beyond our rigorous theoretical guarantees, we empirically demonstrated the practical utility of VANEB through heuristic extensions to Deep Neural Networks (DNNs). Although the core assumption of asymptotic normality does not apply to the highly non-convex loss landscapes of these modern architectures, our experimental results reveal that treating local parameters as approximately Gaussian allows VANEB to serve as a powerful, adaptive shrinkage regularizer, outperforming standard baselines like FedAvg.

Several directions remain open. First, our guarantees are restricted to the diagonal covariance approximation; extending them to full, non-diagonal Fisher information matrices is a natural next step. Second, bridging the gap between the theory and deep learning applications requires relaxing the asymptotic normality condition and addressing high dimensionality. A non-asymptotic analysis based on empirical process theory, controlling the concentration of the local empirical risks directly, could provide guarantees for VANEB in non-convex, finite-sample settings.

\bibliographystyle{plainnat}
\bibliography{references,references2}

\newpage
\appendix
\begin{center}
    \Huge{\textbf{Appendices}}
\end{center}

\section{Proof of Main Results}

\subsection{Proof overview and comparison with prior work}
Our proofs build on the general NPMLE analysis framework developed by \citet{saha2020nonparametric} and \citet{soloff2025multivariate} for the fixed-covariance (homoskedastic) setting. However, the extension to parameter-dependent covariances $\Sigma_k(\b\tht)$ introduces several structural difficulties that require new arguments.

The classical analysis works directly with $f_{k,G}$. In our setting, the oracle Bayes rule involves variance-weighted densities $\t f_{k,G,i}$ that have no counterpart in prior work. We introduce a sandwiching argument (Section \ref{sec:prf-hell}) to transfer Hellinger bounds between $f_{k,G}$ and $\t f_{k,G,i}$ via the condition number $\tau=\bs/\us$.

The NPMLE support is no longer a tractable ridgeline manifold when $\Sigma_k$ depends on $\b\tht$. We establish convergence of the pseudo-EM iterates under Lipschitz covariances (Lemma \ref{lem-vaneb-convergence}) and quantify the approximation gap to the true heteroskedastic critical points via the implicit function theorem.

The metric entropy bounds require controlling the variation of $\vrp^{(k)}(\bx;\b\tht)$ as a function of $\b\tht$ when the covariance also moves. Lemma \ref{lem:mvt} provides a new mean-value-type bound involving the Jacobian of the variance function, which replaces the standard Lipschitz bounds available in the fixed-covariance case.

The covering numbers of the marginal class $\bb F$ grow polynomially in $\un^{d/2}$ due to the sharpening of kernels with increasing local sample size---a phenomenon absent when each client contributes a single observation as in \citet{saha2020nonparametric,soloff2025multivariate}.

\subsection{Notation}
We write $\partial_i^jf:= \frac{\partial^j}{\partial x_i^j}f(\mb x)$ to denote a $j$-th partial derivative of a vector function $f:\R^d\to\R$.
Let $\bb{B}_r(\bx) = \{\mb y\in \R^d\,;\,\|\bx-\mb y\|\le r\}$ denote a closed ball in $\R^d$. For an integer $m>0$, we use a generic notation $[m] = \{1, \dots, m\}$. We adopt the following notation from \cite{soloff2025multivariate}:
\begin{enumerate}
    \item Consider a pseudo-metric space $(M, m)$. For some $\eps > 0$, let $N(\eps, M, m)$ denote the $\eps$-covering number:
    \[N:=\argmin_{N}\left\{M\subset\bigcup_{i=1}^N\{y\,;\,m(y, x_i)\le \eps\};x_1,\dots,x_N\in M\right\}.\] 
    We use a shorthand $N(\eps, M)$ instead when $M\subseteq\R^d$.
    \item For notations $f_{\blt,G} = (f_\GSk)_{k=1}^K$ and $\mb q_{\blt,G}=(\mb q_\GSk)_{k=1}^K$, define the spaces of marginals $f_{\blt,G}$ and $\mb q_{\blt,G}$:
    \[\bb{F} = \left\{f_{\blt,G}\,;\,G\in {\c P}(\R^d)\right\},~\bb F_q = \left\{{\mb q}_{\blt,G}\,;\,G\in {\c P}(\R^d)\right\}.\]
    \item For $S\subset\R^d$ and $M > 0$, $S^M$ denotes an $M$-enlargement of $S$:
    \[S^M = \{\mb x\in \R^d\,;\,\mathfrak{d}_S(\mb x) \le M\}\]
    for $\mathfrak{d}_S(\mb x)=\inf_{s\in S}\|\mb x - s\|$.
    \item Define the semi-norms
    \begin{align*}
        \|f_{\blt,G} - f_{\blt,H}\|_{\infty, S^M}
        &:= \max_{1\le k\le K} \sup_{\bx\in S^M}\vrt{f_\GSk(\bx) - f_\HSk(\bx)} \\
        \|{\mb q}_{\blt,G} - {\mb q}_{\blt,H}\|_{\infty, S^M}
        &:= \max_{1\le k\le K}\sup_{\bx\in S^M}\nrm{\mb q_\GSk(\bx) - \mb q_\HSk(\bx)}{}
    \end{align*}
\end{enumerate}
where $\nrm{\cdot}{}$ denotes a vector $\ell_2$ norm.

\subsection{Proof of Theorem \ref{thm-hllngr}}\label{sec:prf-hell}
The overall concentration strategy follows the $\eta$-net approach of \citet[Theorem 7]{soloff2025multivariate}; the key new ingredient is the transfer between $f_{k,G}$ and the variance-weighted densities $\t f_{k,G,i}$ via the condition number $\tau$.

By the conditions on $\Sigma_k$, we have
\[\us f_{k,G}\le\t f_{k,G,i}\le \bs f_{k,G},\quad i\le d,k\le K\]
hence
\[\frac{\us f_{k,G}}{\bs f_{k,\h G}}\le\frac{\t f_{k,G,i}}{\t f_{k,\h G,i}}\le \frac{\bs f_{k,G}}{\us f_{k,\h G}},\quad i\le d,k\le K\]
for any measures $\h G$ and $G$. This leads to
\[\log\frac{\bs}{\us}+\ell_K(\h G)-\ell_K(G_0)\ge\frac{1}{K}\sum_{k=1}^K\log\frac{\t f_{k,\h G,i}}{\t f_{k,G_0,i}}\ge\log\frac{\us}{\bs}+\ell_K(\h G)-\ell_K(G_0),\quad i\le d.\]
Therefore, we only need to prove the bound for $\bar h(f_{\blt,\h G},f_{\blt,G_0})$ as
\[\brcs{\prod_{k=1}^K \frac{\t f_{k,G,i}(\tx_k)}{\t f_{k,G_0,i}(\tx_k)}\ge\exp\lrp{-K(t-\log\tau)}}\subseteq\brcs{\prod_{k=1}^K \frac{f_{k,\h G}(\tx_k)}{f_\GoSk(\tx_k)}\ge\exp(-Kt)}\]
for any $t\in\R$ and $i=1,...,d$.

Given the approximation \eqref{eq-hllngr-npmle} in the likelihood, we acknowledge that for $t>0$,
\begin{align*}
    \P\bigg(&\bar{h}(f_{\blt,\h G}, f_{\blt,G_0}) \ge C_{d,\us,\bs} t\eps_M\bigg) \\
    &= \P\bxs{\bar{h}(f_{\blt,\h G}, f_{\blt,G_0}) \ge C_{d,\us,\bs}t\eps_M, \prod_{k=1}^K \frac{f_{k,\h G}(\tx_k)}{f_\GoSk(\tx_k)}\ge \exp\left(-C_{d,\us,\bs}K\eps_M^2\right)} \\
    &\le\P\bxs{\prod_{k=1}^K \frac{f_{k,\h G}(\tx_k)}{f_\GoSk(\tx_k)}\ge \exp\left(-C_{d,\us,\bs}K\eps_M^2\right)}.
\end{align*}
Suppose for some $\t\gamma_K>0$, we have
\[
\prod_{k=1}^K \frac{f_{k,\h G}(\tx_k)}{f_\GoSk(\tx_k)}\ge \exp\left((\beta-\alpha)K\t\gamma_K^2\right) \mathrm{~for~some~} 0<\beta < \alpha < 1
\]
and fix $t>0$.
Let $\{f_{\blt,H_j}\}_{j=1}^N\subset \bb{F}$ be an $\eta$-net of $\bb{F}$ under the semi-norm $\|\cdot\|_{\infty, S^M}$ for covering number $N:=N(\eta,\bb F,\nrm{\cdot}{\infty,S^M})$.
Let $J\subseteq[N]$ collect the indices for which a distribution $H_{0j}$ on $\R^d$ satisfying the following exists: 
\[
\|f_{\blt,H_{0j}} - f_{\blt,H_j}\|_{\infty, S^M}\le \eta 
\quad\&\quad
\bar{h}(f_{\blt,H_{0j}}, f_{\blt,G_0})\ge t\t\gamma_K.
\]
As proclaimed by an $\eta$-net, there exsits an index $j^*\in [N]$ such that \[\|f_{\blt,H_{j^*}} - f_{\blt,\h G}\|_{\infty, S^M}\le \eta.\] 
On $\{\bar{h}(f_{\blt,\h G}, f_{\blt,G_0}) \ge  t\t\gamma_K\}$, it follows that $j^*\in J$ by the definition of $J$ hence
\[\|f_{\blt,H_{0j^*}} - f_{\blt,\h G}\|_{\infty, S^M}\le 2\eta.\]
This gives
\[
f_{k,\h G}(\tx_k)
\le \begin{cases}
f_{k,H_{0j^*}}(\tx_k) + 2\eta, & \text{if } \tx_k\in S^M \\(2\pi\us)^{-d/2}, & \text{otherwise}.
\end{cases}
\]
Defining $v(\mb x) := \eta 1_{\mb x\in S^M} + \eta\left(\frac{M}{\mathfrak{d}_S( \mb x)}\right)^{d+1}1_{\mb x\not\in S^M}$, we have 
\begin{align*}
    \exp&\brcs{(\beta-\alpha)Kt^2\t\gamma_K^2} \\
    &\le \max_{j\in J}\left[\prod_{k=1}^K \frac{f_{k, H_{0j}}(\tx_k) + 2v(\tx_k)}{f_\GoSk(\tx_k)}\right] \cdot \left[\prod_{k :\tx_k\not\in S^M}\frac{(2\pi\us)^{-d/2}}{2v(\tx_k)}\right]
\end{align*}
on the event $\{\bar{h}(f_{\blt,\h G}, f_{\blt,G_0}) \ge  t\t\gamma_K\}$. Hence
\begin{align*}
    \P\bigg(\bar{h}(f_{\blt,\h G}, f_{\blt,G_0}) \ge  t\t\gamma_K\bigg) 
    \le \P\bxs{\max_{j\in J}\prod_{k=1}^K \frac{f_{k, H_{0j}}(\tx_k) + 2v(\tx_k)}{f_\GoSk(\tx_k)}\ge \exp(-\alpha Kt^2\t\gamma_K^2)}\\
    +\P\bxs{\prod_{k :\tx_k\not\in S^M}\frac{(2\pi\us)^{-d/2}}{2v(\tx_k)}\ge \exp(\beta Kt^2\t\gamma_K^2)}.
\end{align*}
By Markov's inequality, the second term is bounded by
\begin{align*}
    &\exp\left(-\frac{\beta Kt^2\t\gamma_K^2}{2\log K}\right) \E\left\{\prod_{k\,;\,\tx_k\not\in S^M}\frac{(2\pi\us)^{-d/2}}{2v(\tx_k)}\right\}^{\frac{1}{2\log K}} \\
    &= \exp\left(-\frac{\beta Kt^2\t\gamma_K^2}{2\log K}\right) \E\bxs{\prod_{k=1}^K\brcs{\frac{(2\pi\us)^{-\frac{d/2}{d+1}}\mathfrak{d}_S(\tx_k)}{(2\eta)^{1/(d+1)}M}}^{1_{\mathfrak{d}_S(\tx_k)\ge M)}}}^{\frac{d+1}{2\log K}}.
\end{align*}
Using similar arguments as in \citet[Proof of Theorem 7]{soloff2025multivariate}, we obtain the following tail probability:
\begin{align*}
    \P\bxs{\bar{h}(f_{\blt,\h G}, f_{\blt,G_0}) \ge  t\t\gamma_K}
    &\le N(\eta,\bb F,\nrm{\cdot}{\infty,S^M})\exp\brcs{\left(\alpha-1\right)\frac{Kt^2\t\gamma_K^2}{2}+KC_d\sqrt{\eta\Vol(S^M)}} \\
    &\quad+\exp\left(-\frac{\beta Kt^2\t\gamma_K^2}{2\log K}+\kappa^\lambda\us^{-\frac{d\lambda/2}{d+1}} \sum_{k=1}^K\E\left[\mathfrak{d}_S(\tx_k)^{\lambda} 1_{\mathfrak{d}_S(\tx_k)\ge M}\right]\right)
\end{align*}
for $\kappa := \lrp{\us^{\frac{d/2}{d+1}}\eta^{\frac{1}{d+1}} M}^{-1}$ and $\lambda := \frac{d+1}{2\log K}$.
Lemma \ref{lem-lip} gives us
\begin{align*}
    \E\left[\mathfrak{d}_S( \tx_k)^{\lambda} 1_{\mathfrak{d}_S(\tx_k)\ge M}\right]\le C_dM^{d+\lambda-2}{\bs}^{1-d/2}e^{-M^2/(8\bs)} + M^\lambda \left(\frac{2\mu_{k,q}(S,G_0)}{M}\right)^q
\end{align*}
whenever $\frac{d+1}{2(1\wedge q)}\le\log K$. Taking $M\ge\sqrt{8\bs\log K}$, we have $e^{-M^2/(8\bs)}\le 1/K$ so
\begin{align*}
    \kappa^\lambda\sum_{k=1}^K\E&\left[\mathfrak{d}_S( \tx_k)^{\lambda} 1_{\mathfrak{d}_S(\tx_k)\ge M}\right]\le (\kappa M)^\lambda\brcs{C_dM^{d-2}\bs^{1-d/2} + \sum_{k=1}^K\lrp{\frac{2\mu_{k,q}(S,G_0)}{M}}^q}.
\end{align*}
Noting $(\kappa M)^\lambda = \left(\us^{d/2}\eta\right)^{-1/(2\log K)}$, choose $\eta = \frac{1}{K^2\us^{d/2}}$ so
\begin{align*}
    \left(\us^{d/2}\eta\right)^{-1/(2\log K)}=\lrp{\frac{4}{K^2}}^{-\frac{1}{2\log K}}\le e
\end{align*}
Also, Lemma \ref{lem-entropy} applies as
\begin{align*}
    \log N(\eta,\bb F,\nrm{\cdot}{\infty,S^M})\le_dN(a,S^{M+a})\log^{d+1}c_{d,\us,\bs,\bs'}K
\end{align*}
for $a=\sqrt{\frac{2\bs}{\un}\log\frac{1}{\eta}}=\sqrt{\frac{4\bs}{\un}\log\us^{d/2}K}$.
Since $a\in\lrp{\sqrt{\frac{2\bs}{\un}\log\us^{d/2}K},\sqrt{\frac{6\bs}{\un}\log\us^{d/2}K}}=:(\underline{a},\bar a)$, we have
\begin{align}
    N(a,(S^M)^a)\le N(\underline{a},S^{M+\bar a})\le_d\Vol((S^{M+\bar a})^{\underline{a}/2})/\underline{a}^d\nonumber \\
    \le_{d,\us,\bs}\Vol(S^{2M})\lrp{\frac{\un}{\log K}}^{d/2}
    \label{eq:arg1}
\end{align}
as $\bar a\ll M$ in large enough $\un$.
So, aggregating the rates,
\begin{align*}
    \log N\le_d\Vol(S^{2M})\un^{d/2}\log^{d/2+1}c_{d,\us}K.
\end{align*}
Therefore, we have the tail probability bounded as
\begin{align*}
    \P&\bigg(\bar{h}(f_{\blt,\h G}, f_{\blt,G_0}) \ge  t\t\gamma_K\bigg) \\
    &\le \exp\brcs{-\left(1-\alpha\right)\frac{Kt^2\t\gamma_K^2}{2} + C_d\Vol(S^{2M})\un^{d/2}\log^{d/2+1}c_{d,\us}K +C_{d,\us}\sqrt{\Vol(S^M)}} \\
    &\qquad+\exp\bxs{-\frac{\beta Kt^2\t\gamma_K^2}{2\log K} + e\brcs{C_{d,\bs}'M^{d-2} + \inf_{q\ge \frac{d+1}{2\log K}}\sum_{k=1}^K\lrp{\frac{2\mu_{k,q}(S,G_0)}{M}}^q}}
\end{align*}
for any $t > 1$.
Now, the choice of $\eps_M^2 = \eps_M^2(S, G_0)$ in \eqref{eq-eps} satisfies
\begin{align*}
    \max&\left\{\Vol(S^{2M})\un^{d/2}\log^{d/2+1}c_{d,\us}K, \sqrt{\Vol(S^M)},M^{d-2},K\log K\inf_{q\ge \frac{d+1}{2\log K}}\sum_{k=1}^K\left(\frac{2\mu_{k,q}(S,G_0)}{M}\right)^q\right\}  \\
    &\le_{d,\us} K\eps_M^2(S, G_0).
\end{align*}
If we then take $\t\gamma_K  = \sqrt{\frac{C_{d,\us}}{(1-\alpha)\wedge\beta}}\eps_M(M, S, G_0)$, then by the choice of $\eps_M$ above, we have $\eps_M^2\ge\Vol(S^1)\frac{M^d\un^{d/2}\log^{d/2+1}K}{K}\ge_d\frac{\un^{d/2}\log^{d+1}K}{K}$ hence
\[\P\bigg(\bar{h}^2(f_{\blt,\h G}, f_{\blt,G_0})\ge_{d,\us,\bs}t^2\t\gamma_K^2\bigg)\le2\exp\left(-\frac{t^2K\eps_M^2}{\log K} \right)\le2\exp\lrp{-t^2\un^{d/2}\log^dK}.\]
Integrating the upper bound gives the bound for
$\E\bxs{\frac{\bar{h}^2(f_{\blt,\h G}, f_{\blt,G_0})}{\t\gamma_K^2}}$.

\subsection{Oracle denoising inequality: proof of Theorem \ref{thm-denoising}}\label{sec:prf-denoising}
For simplicity, we use the notation $\Tht=[\b\tht_1,...,\b\tht_K]^\top\in\R^{K\times d}$.
The regret $\frac{1}{K}\sum_{k=1}^K\E\|\bx_k^v - \oraet_k\|^2$ hence equals $\frac{1}{K}\E\|\h\Tht^v - \Oraet\|_F^2$.

\subsubsection{Regularizing the Bayes rule}
We first verify that the likelihood is lower bounded per observation.
By \eqref{eq-apprx-npmle-2}, for any fixed $j\in [K]$,
\begin{align*}
    \prod_{k=1}^Kf_{k,\h G}(\tx_k) 
    &\ge \prod_{k=1}^K f_{k,K^{-1} \dlt_{\tx_k} + (1-K^{-1})\h G}(\tx_k) \\
    &\ge K^{-1}|2\pi\Sigma_j|^{-1/2}(1-K^{-1})^{K-1}\prod_{k: k\ne j} f_{k,\h G}(\tx_k).
\end{align*}
Cancelling terms for $k\in [K]\setminus \{j\}$, we conclude
\begin{align}\label{eq-npmle-obs-bound}
    f_{j,\h G}(\tx_j)
    \ge \frac{\inf_{\mb u}|2\pi\Sigma_j(\mb u)|^{-1/2}}{eK}=\frac{(2\pi\bs)^{-d/2}}{eK}=:\rho_0.
\end{align}
Given this, \cite{saha2020nonparametric,soloff2025multivariate} adopt the regularized empirical Bayes and oracle Bayes rules 
\begin{align}
    \bx^v_{k,\rho} 
    &= \frac{\tx_k}{\sqrt{n_k}}+\frac{\mb q_{k,\h G}(\tx_k)}{\sqrt{n_k}f_{k,\h G}(\tx_k) \vee\rho} \label{eq:mtweedie2} \\
    \oraet_{k,\rho}
    &= \frac{\tx_k}{\sqrt{n_k}}+\frac{\mb q_\GoSk(\tx_k)}{\sqrt{n_k}f_{k,G_0}(\tx_k) \vee\rho}.\nonumber
\end{align}
The inequality \eqref{eq-npmle-obs-bound} means that $\h\Tht^v_\rho = \h\Tht^v$ whenever $\rho \le \rho_0$ hence
\begin{align}\label{eq:reg-decomp}
    \|\h\Tht^v - \Oraet\|_F
    = \|\h\Tht^v_\rho - \Oraet\|_F 
    \le \|\h\Tht^v_\rho - \Oraet_\rho\|_F  + \|\Oraet_\rho - \Oraet\|_F.
\end{align}
The last two terms are analyzed separately in \cite{soloff2025multivariate}, and we take similar approaches.

\subsubsection{Control of $\h\Tht_\rho^v-\Tht^v$}
The argument of \citet[Proof of Theorem 9]{soloff2025multivariate} does not apply here because our marginal densities are scaled mixtures over kernels that depend on both $\b\tht$ and $\Sigma_k(\b\tht)$. We therefore develop a modified analysis that routes through the variance-weighted functional $\t\Dlt_k$ (Lemmas \ref{lem:cvx} and \ref{lem:dco}), extending \citet[Lemmas 4.3 and F.6]{saha2020nonparametric} to our setting. Define
\[\t\Dlt_k(G, \rho) := \int\left(1 - \frac{f_\GSk}{f_\GSk\vee\rho}\right)^2\frac{\lrnrm{\mb q_\GoSk}{}^2}{{f_\GSk}}.\]
Then, for any compact $S\subseteq\R^d$, letting $S_{\un}:=\frac{1}{\sqrt{\un}}S$:
\begin{align*}
    \E\nrm{\Oraet_\rho-\Oraet}{F}^2=\sum_{k=1}^K\frac{1}{n_k}\E\lrnrm{\frac{\mb q_\GoSk(\tx_k)}{f_{k,G_0}(\tx_k) \vee\rho}-\frac{\mb q_\GoSk(\tx_k)}{f_{k,G_0}(\tx_k)}}{}^2
    \le\sum_{k=1}^K\frac{1}{n_k}\t\Dlt_k\lrp{G_0,\rho}\\
    \le_{d,\us,\bs} \frac{K}{\un}\brcs{N
    \left(\frac{4}{\sqrt{\un}L\lrp{\rho}}, S_{\un} \right)\rho L(\rho)^d+G_0(S_{\un}^c)}
\end{align*}
by Lemmas \ref{lem:cvx} and \ref{lem:dco} for $L(\rho) := \sqrt{\us\log\frac{1}{(2\pi\bs)^d\rho^2}}$ whenever $\rho\le (2\pi\bs)^{-d/2}e^{-8/\us}$.
Choosing
\begin{align}\label{eq:rho}
    \rho=\frac{(2\pi\bs)^{-d/2}}{e^{1+8/\us}K}
\end{align}
and $S=S^M$, we have $L(\rho)=\sqrt{2\us\log(Ke^{1+8/\us})}$ hence
\begin{align*}
    \frac{\un}{K}\E\nrm{\Oraet_\rho-\Oraet}{F}^2\le_{d,\us,\bs}N\left(\frac{1}{\sqrt{2\us\,\un\log K}}, S_{\un}^M \right)\frac{\log^{d/2}K}{K}+G_0((S_{\un}^M)^c) \\
    \le_d\Vol(S_{\un}^{2M})\lrp{2\us\,\un\log K}^{d/2}\frac{\log^{d/2}K}{K}+G_0((S_{\un}^M)^c) \\
    \le_{d,\us}\Vol\lrp{S^1}M^d\frac{\log^dK}{K} + \inf_{q\ge\frac{d+1}{2\log K}}\lrp{\frac{2\mu_q}{M}}^q \\
    \text{for }\mu_q:=\E_{G_0}[\frak{d}_{S_{\un}}(\b\tht)^q]^{1/q}
\end{align*}
by applying \citet[Lemma F.6]{saha2020nonparametric}. The second term is relevant with the second term in \eqref{eq-eps2}.
Therefore,
\[\frac{1}{K}\E\nrm{\Oraet_\rho-\Oraet}{F}^2\le_{d,\us,\bs}\Vol\lrp{S^1}M^d\frac{\log^dK}{\un K} + \frac{1}{\un}\inf_{q\ge\frac{d+1}{2\log K}}\lrp{\frac{2\mu_q}{M}}^q\lesssim\frac{\eps_M^2(S_\blt,G_0)}{\un}.\]
\subsubsection{Control of $\h\Tht^v_\rho - \Oraet_\rho$}
Now, we recall that
\[\lrnrm{\h\Tht^v_\rho - \Oraet_\rho}{F}^2=\sum_{k\le K}\frac{1}{n_k}\lrnrm{\frac{\mb q_{k,\h G}(\tx_k)}{f_{k,\h G}(\tx_k)\vee\rho}-\frac{\mb q_\GSk(\tx_k)}{f_\GSk(\tx_k)\vee\rho}}{}^2.\]
For $\dlt > 0$ define an event
$A_\dlt = \brcs{\max_{i\le d}\bar h^2(\t f_{\blt,\h G,i},\t f_{\blt,G_0,i})\vee\bar{h}^2\lrp{f_{\blt,\h G}, f_{\blt,G_0}}\le\dlt}$. Given a compact set $S\subset \R^d$, define a metric
\[
m^S(G,H):=\frac{1}{\sqrt{\un}}\max_{k\in [K]}\sup_{\mb z\in\R^d:\mathfrak{d}_{S_k}(\mb z)\le M}\left\|\frac{\mb q_\GSk(\mb z)}{f_\GSk(\mb z) \vee\rho} - \frac{\mb q_{k,H}(\mb z)}{f_\HSk(\mb z) \vee\rho}\right\|
\]
for $S_k:=\frac{1}{\sqrt{n_k}}S$. Let $G_1,\dots,G^{(N)}$ denote a minimal $\eta^*$-covering of
\begin{align}\label{eq:hellball}
    {\c H}_\dlt:=\brcs{G:\max_{i\le d}\bar{h}^2(\t f_{\blt,G,i}, \t f_{\blt,G_0,i})\vee\bar{h}^2(f_{\blt,G}, f_{\blt,G_0})\le\dlt}
\end{align}
under $m^S$ for $N:=N(\eta^*,{\c H}_\dlt,m^S)$.
For $j\in [N]$ define a $K\times d$ matrix $\h\Tht^{v,{(j)}}_\rho$ analogously as $\Tht$ whose $i^{th}$ row is $\frac{\tx_k}{\sqrt{n_k}}+\frac{\mb q_{k,G^{(j)}}(\tx_k)}{\sqrt{n_k}f_{k,G^{(j)}}(\tx_k) \vee\rho}$. We bound the regret as $\|\h\Tht^v_\rho - \Oraet_\rho\|_F \le \sum_{t=1}^4Q_t$, where
\begin{equation}
    \begin{aligned}
    Q_1 
    &:= \|\h\Tht^v_\rho - \Oraet_\rho\|_F 1_{A_\dlt^c} \\
    Q_2
    &:= \left(\|\h\Tht^v_\rho - \Oraet_\rho\|_F - \max_{j\in [N]}\|\h\Tht^{v,(j)}_\rho - \Oraet_\rho\|_F\right)_+ 1_{A_\dlt} \\
    Q_3
    &:= \max_{j\in [N]}\left(\|\h\Tht^{v,(j)}_\rho - \Oraet_\rho\|_F - \E\|\h\Tht^{v,(j)}_\rho - \Oraet_\rho\|_F\right)_+ \\
    Q_4
    &:= \max_{j\in [N]}\E\|\h\Tht^{v,(j)}_\rho - \Oraet_\rho\|_F.
    \end{aligned}
\end{equation}
\paragraph{Bounding \texorpdfstring{$\E(Q_1^2)$}{}.}
By the shrinkage Tweedie's formula in \eqref{eq:mtweedie2},
\begin{align*}
    \nrm{\bx^v_{k,\rho} - \oraet_{k,\rho}}{}^2
    &=\frac{1}{n_k}\left\|\frac{\mb q_{k,\h G}(\tx_k)}{f_{k,\h G}(\tx_k) \vee\rho}-\frac{\mb q_\GoSk(\tx_k)}{f_{k,G_0}(\tx_k) \vee\rho}\right\|^2 \\
    &\le\frac{1}{n_k}\lrnrm{\frac{\mb q_{k,\h G}(\tx_k)}{f_{k,\h G}(\tx_k) \vee \rho}}{}^2+\frac{1}{n_k}\lrnrm{\frac{\mb q_\GoSk(\tx_k)}{f_{k,G_0}(\tx_k) \vee\rho}}{}^2.
\end{align*}
Lemma \ref{p1} provides
\begin{align*}
    \E(Q_1^2) \le_{\bs} \frac{K}{\un}\log\frac{1}{(2\pi\us)^d\rho^2}\times\Pr(A_\dlt^c)
    \le\frac{K\log(K^2e^{1+8/\us})}{\un}\Pr(A_\dlt^c).
\end{align*}
By Corollary \ref{cor-hllngr-rsk}, there is a constant $C_d > 0$ such that $\dlt = C_d\eps_M^2(S_\blt,G_0)$ satisfies $\Pr(A_\dlt^c)\le2\exp\lrp{-\log^dK}$. Hence
\[\frac{1}{K}\E(Q_1^2)\le_{\bs}\frac{2\log(Ke^{1/2+4/\us})}{\un\exp(\log^dK)}.\]

\paragraph{Bounding \texorpdfstring{$\E\lrp{Q_2^2}$}{}.}
Observe that
\begin{align*}
    Q_2^2
    &\le 1_{A_\dlt}\min_{j\in [N]}\|\h\Tht^v_\rho - \h\Tht_\rho^{v, (j)}\|_F^2 \\
    &\qquad= 1_{A_\dlt} \min_{j\in [N]}\sum_{k=1}^K\frac{1}{n_k}\left\|\frac{\mb q_{k,\h G}(\tx_k)}{f_{k,\h G}(\tx_k) \vee \rho} - \frac{\mb q_{k,G^{(j)}}(\tx_k)}{f_{k,G^{(j)}}(\tx_k) \vee\rho}\right\|^2.
\end{align*}
On $A_\dlt$, we take $j$ such that $m^S(\h G, G^{(j)})\le \eta^*$. For each $k$, we have either $\tx_k\in S_k^M$ or $\tx_k\not\in S_k^M$. When $\tx_k\in S_k^M$, we can bound the above by the supremum over all $\tx_k\in S_k^M$. Otherwise, bound the regularized rules as before.
Then,
\begin{align*}
    Q_2^2
    &\le_d 1_{A_\dlt}\bxs{K(\eta^*)^2 + \#\{k: \tx_k\not\in S_k^M\}\frac{\log(K^2e^{1+8/\us})}{\un}}
\end{align*}
hence
\begin{align*}
    \frac{\E(Q_2^2)}{K}\le_d (\eta^*)^2 + \frac{\log(K^2e^{1+8/\us})}{\un K}\sum_{k=1}^K\Pr\left(\mathfrak{d}_{S_k}(\tx_k)\ge M\right).
\end{align*}
By Lemma \ref{lem-lip}, taking $\lambda\downarrow 0$,
\begin{align*}
    \frac{\E(Q_2^2)}{K}
    &\le_{d,\us} (\eta^*)^2+\frac{\log(K^2e^{1+8/\us})}{\un K}\brcs{M^{d-2}e^{-M^2/(8\bs)}+\sum_{k=1}^K\inf_{q\ge\frac{d+1}{2\log K}}\lrp{\frac{2\mu_{k,q}(S_k,G_0)}{M}}^q} \\
    &\le (\eta^*)^2+\frac{\log(K^2e^{1+8/\us})}{\un K}\brcs{\frac{M^{d-2}}{K}+\sum_{k=1}^K\inf_{q\ge\frac{d+1}{2\log K}}\lrp{\frac{2\mu_{k,q}(S_k,G_0)}{M}}^q}
\end{align*}
given that $M\ge\sqrt{8\bs\log K}$.
\paragraph{Bounding \texorpdfstring{$\E Q_3^2$}{}.}\label{sec:bounding-zeta3}
Fix $j\in [N]$. Then,
\begin{align*}
    \|\bx_{k,\rho}^{v,(j)} - \oraet_{k,\rho}\|^2
    &\le\frac{1}{n_k}\left\|\frac{\mb q_{k,G^{(j)}}(\tx_k)}{f_{k,G^{(j)}}(\tx_k) \vee\rho} - \frac{\mb q_{k,G_0}(\tx_k)}{f_{k,G_0}(\tx_k) \vee\rho}\right\|^2 \\
    &\le\frac{2}{n_k}\left(\left\|\frac{\mb q_{k,G^{(j)}}(\tx_k)}{f_{k,G^{(j)}}(\tx_k) \vee\rho}\right\|^2 + \left\|\frac{\mb q_{k,G_0}(\tx_k)}{f_{k,G_0}(\tx_k) \vee\rho}\right\|^2\right).
\end{align*}
By Lemma \ref{p1},
\[
\|\bx_{k,\rho}^{v,(j)} - \oraet_{k,\rho}\|^2
\le -\frac{4\bs}{n_k}\log(2\pi\us)^{d}\rho^2\le_{d,\us,\bs}\frac{\log(Ke^{1+8/\us})}{\un}=:r^2.
\]
Let $V_{k,j} = \frac{\lrnrm{\bx_{k,\rho}^{v,(j)} - \oraet_{k,\rho}}{}}{r}$, so $\|\h\Tht_\rho^{v, (j)}- \Oraet_\rho\|_F^2 = r^2\|\b V_j\|^2$ for $\b V_j:=(V_{1,j}, \ldots, V_{K,j})$. Since $V_{k,j}$ are independent random variables in $[0,1]$, it has an exponentially decreasing tail probability with
\[\Pr\left(\|\b V_j\| \ge \E \|\b V_j\| + t\right)\le e^{-t^2/2}\]
for some $c>0$ independent of $t\in\R$. Thus
\[
\Pr\left(Q_3 \ge s\right)\le N\exp\left(-C_{d,\us,\bs}\frac{s^2}{r^2}\right),\quad s>0.
\]
Therefore,
\begin{align}\label{eq-zeta3}
\frac{1}{K}\E(Q_3^2) \le_d\frac{r^2}{K}\log(eN)=\log(eN)\frac{\log(Ke^{1+8/\us})}{K\un}.
\end{align}
\paragraph{Bounding \texorpdfstring{$\E Q_4^2$}{}.}
Again
\begin{align*}
    \E\|\h\Tht_\rho^{v,(j)} - \Oraet_\rho\|_F
    &\le\sqrt{\E\|\h\Tht_\rho^{v,(j)} - \Oraet_{\rho}\|_F^2} \\
    &\qquad\le\sqrt{\frac{1}{\un}\sum_{k=1}^K \E_{\tx_k\sim f_{k,G_0}}\left\|\frac{\mb q_{k,G^{(j)}}(\tx_k)}{f_{k,G^{(j)}}(\tx_k) \vee\rho} - \frac{\mb q_{k,G_0}(\tx_k)}{f_{k,G_0}(\tx_k) \vee\rho}\right\|^2}.
\end{align*}
Then, Theorem \ref{fcc} gives us
\begin{align*}
    &\frac{1}{K}\left(\E\|\h\Tht_\rho^{v,(j)} - \Oraet_\rho\|_F\right)^2\le_{d,\us,\bs}\frac{1}{\un}\bxs{d\max\brcs{\log^3(Ke^{1+8/\us}),\log\frac{1}{\dlt}}}\dlt.
\end{align*}
Recall our choice
\[\dlt = \eps_M^2(S_\blt,G_0)\ge_{d,\us,\bs}\frac{\Vol(S^1)M^d\log^{d/2+1}K}{K}\ge\frac{1}{K}\]
hence we obtain
\[\frac{\E(Q_4^2)}{K}\le_{d,\us,\bs}\frac{\log^3K}{\un}\eps_M^2(S_\blt,G_0).\]
\paragraph{Metric Entropy for $m^S$}
For any $G,\t G\in{\c P}(\R^d)$, note that
\begin{align}
    \sqrt{\un}m^S(G, \t G)&=\max_{k\in [K]}\sup_{\mb x\,;\,\mathfrak{d}_{S_k}(\mb x)\le M}\lrnrm{\frac{\mb q_\GSk(\mb x)}{f_\GSk(\mb x) \vee\rho} - \frac{\mb q_{k,\t G}(\mb x)}{f_{k,\t G}(\mb x) \vee\rho}}{} \nonumber\\
    &\le \max_{k\in [K]}\sup_{\mb x\,;\,\mathfrak{d}_{S_k}(\mb x)\le M}\frac{\lrnrm{\mb q_\GSk(\mb x)}{}}{f_\GSk(\mb x) \vee\rho}\cdot\frac{\vrt{f_{k,\t G}(\mb x)\vee\rho-f_\GSk(\mb x)\vee\rho }}{f_{k,\t G}(\mb x) \vee\rho} \nonumber\\
    &\qquad+\max_{k\in [K]}\sup_{\mb x\,;\,\mathfrak{d}_{S_k}(\mb x)\le M}\frac{\lrnrm{\mb q_\GSk(\mb x)-\mb q_{k,\t G}(\mb x)}{}}{f_{k,\t G}(\mb x) \vee\rho} \nonumber\\
    &\le\sqrt{-\bs\log(2\pi\us)^d\rho^2}\max_{k\in [K]}\sup_{\mb x\,;\,\mathfrak{d}_{S_k}(\mb x)\le M}\frac{\vrt{f_{k,\t G}(\mb x) - f_\GSk(\mb x)}}{\rho} \nonumber\\
    &\qquad+\max_{k\in [K]}\sup_{\mb x\,;\,\mathfrak{d}_{S_k}(\mb x)\le M}\frac{\lrnrm{\mb q_\GSk(\mb x)-\mb q_{k,\t G}(\mb x)}{}}{\rho}\label{eq-cvn1}.
\end{align}
Therefore,
\begin{align*}
    m^S(G,\t G)\le A\max_{k\in [K]}\nrm{f_{k,G}-f_{k,\t G}}{\infty,S_k^M}+B\max_{k\in [K]}\nrm{{\mb q}_{k,G}-\mb q_{k,\t G}}{\infty,S_k^M}
\end{align*}
for $A:=\frac{1}{\rho}\sqrt{-\frac{\bs}{\un}\log(2\pi\us)^d\rho^2}$ and $B:=\frac{1}{\rho\sqrt{\un}}$.
Letting $\eta^*:=\eta(A+B)$, we have
\begin{align*}
    \log N(\eta^*, {\c P}(\R^d), m^S) \le \max_{k\in [K]}\log N(\eta/2, {\bb F}, \|\cdot\|_{\infty, S_k^M}) + \max_{k\in [K]}\log N(\eta/2, \bb F_q, \|\cdot\|_{\infty, S_k^M}) \\
    \le_d\max_{k\in[K]}\log N(a,S_k^{M+a})\log^{d+1}\frac{c_{d,\us,\bs,\bs'}}{\eta}
\end{align*}
for $a=\sqrt{\frac{2\bs}{\un}\log\frac{1}{\eta}}$ and small $\eta>0$ by Lemma \ref{lem-entropy}. Let $\eta=\rho$ so that
\begin{align}\label{eq:finalN}
    \log N(\eta^*, {\c P}(\R^d), m^S)\le_{d,\us,\bs}\Vol(S^1)M^d\log^{d/2+1}K
\end{align}
for $$\eta^*\le_{d,\us,\bs}\frac{M}{\sqrt{\un}}.$$

\paragraph{Final oracle denoising bound}
To obtain the oracle denoising bound, we aggregate the rates from the error decomposition. The overall risk decomposes as
\[\frac{1}{K}\E\|\h\Tht^v - \Oraet\|_F^2 \le \frac{1}{K}\E\|\Oraet_\rho - \Oraet\|_F^2 + \frac{1}{K}\sum_{t=1}^4\E(Q_t^2).\]
Combining the individual bounds:
\begin{itemize}
\item $\frac{1}{K}\E\|\Oraet_\rho - \Oraet\|_F^2 \lesssim \frac{\eps_M^2(S_\blt,G_0)}{\un}\log^{d/2-1}K$ from the worst-case density floor.
\item $\frac{1}{K}\E(Q_1^2) \le_{\bs} \log(Ke^{1+8/\us})\un^{-1}\exp(-\log^dK)$ is exponentially suppressed.
\item $\frac{1}{K}\E(Q_2^2) \le_d \frac{M^2}{\un} + \frac{\log(K^2e^{1+8/\us})}{K\un}\brcs{\frac{M^{d-2}}{K}+\sum_{k=1}^K\inf_{q\ge\frac{d+1}{2\log K}}\lrp{\frac{2\mu_{k,q}(S_k,G_0)}{M}}^q}\lesssim\eps_M^2(S_\blt,G_0)$ provided that $\un\gtrsim KM^{2-d}\log^{-d/2-1}K$ for \eqref{eq-eps2}.
\item $\frac{1}{K}\E(Q_3^2) \le_d \frac{\Vol(S^1)M^d\log^{d/2+2}K}{K\un}$ from the metric entropy of the Hellinger ball.
\item $\frac{1}{K}\E(Q_4^2) \le_{d,\us,\bs} \frac{\log^3(Ke^{1+8/\us})}{\un}\eps_M^2(S_\blt,G_0)$ captures the cost of restricting to the Hellinger-ball neighborhood.
\end{itemize}
Summing these contributions:
\begin{align*}
    \frac{1}{K}\E\|&\h\Tht^v - \Oraet\|_F^2
    \le_{d,\us,\bs} \frac{(\log K)^{3\vee(d/2-1)}}{\un}\eps_M^2(S_\blt,G_0).
\end{align*}

\paragraph{When $G_0$ is a Gaussian Mixture.} For the case when $G_0$ is a Gaussian Mixture, we only need to bound the second term in \eqref{eq-eps2}. Let $R_k:=R/\sqrt{n_k}$ and $S:=\bb B_R(\b0_d)$. By the linearity of expectation and the definition of $G_0$, we have:
$$\E_{G_0}[{\frak d}_{S_k}(\b\tht)^q] = \sum_{j=1}^{k^*} w_j \E_{\b\tht \sim N(\b a_j^*, \Gamma_j)}(\|\b\tht\| - R_k)_+^q.$$
For each component $j$, let $\b\tht = \b a_j^* + \b\epsilon_j$, where $\b\epsilon_j \sim {\c N}(\b0_d, \Gamma_j^*)$. Define $\bar\gamma^2 := \max_{j\le k^*}\nrm{\Gamma^*_j}{}$ and $A^* := \max_{j\le k^*}\nrm{\mb a^*_j}{}$. By the triangle inequality, we have $\|\b\tht\|\le A^* + \|\b\epsilon_j\|.$
Thus, the distance to $S_k$ is bounded by:$${\frak d}_{S_k}(\b\tht) \le (A^* + \|\b\epsilon_j\| - R_k)_+.$$
For a conservative upper bound, we observe that for any $R_k \ge 0$:$${\frak d}_{S_k}(\b\tht)^q \le (A^* + \|\b\epsilon_j\|)^q.$$
$\|\b\epsilon_j\|-\E\|\b\epsilon_j\|$ is sub-Gaussian with parameter $\bar{\gamma}$ and $\E\|\b\epsilon_j\| \le \sqrt{\tr(\Gamma_j)} \le \sqrt{d\bar{\gamma}}$.
Using the property of sub-Gaussian moments, for $q \ge 1$:
$$\E(\|\b\epsilon_j\|^q)^{1/q} \le C\sqrt{\bar\gamma(d+q)}$$
for some universal constant $C > 0$. Therefore, by Minkowski's inequality:
$$\E[{\frak d}_{S_k}(\b\tht)^q]^{1/q} \le A^* + C\sqrt{\bar\gamma(d+q)}.$$
Substituting this back into the mixture sum:
$$\E_{G_0}[{\frak d}_{S_k}(\b\tht)^q] \le \sum_{j=1}^{k^*} w_j \left( A^* + C\sqrt{\bar\gamma(d+q)} \right)^q = \left( A^* + C\sqrt{\bar\gamma(d+q)} \right)^q$$
Now, consider the ratio with $M^q$. Let $M$ be chosen such that $M \gtrsim \sqrt{\log K}$ as before. Therefore, the term of interest becomes negligible as $K$ is assumed to diverge.
This completes the proof of Theorem \ref{thm-denoising}.
\begin{proof}[Proof of Theorem \ref{thm:oracle-vantrees}]
    Since the posterior mean is the Bayes rule under squared loss,
    \begin{align*}
        {\c R}(\Tht^o,\Tht)
        &=\frac{1}{K}\sum_{k=1}^K\E\,\tr\Bigl\{\Var_{G_0}(\b\tht_k|\tx_k)\Bigr\}
         =\frac{1}{K}\sum_{k=1}^K\E\,\tr\Bigl\{\Var_{G_0}(\b\tht_k|\bx_k)\Bigr\} \\
        &\le \frac{1}{K}\sum_{k=1}^K\E\|\bx_k-\b\tht_k\|^2
         \le \frac{d\bs}{\un}.
    \end{align*}
    Under Assumption \ref{ass:reg}, the prior density $ g_0$ is well behaved and the Gaussian family $\{\vrp^{(k)}(\cdot;\b\tht):\b\tht\in\R^d\}$ is dominated with $C^1$ dependence on $\b\tht$. Therefore the multivariate van Trees inequality of \citet{gassiat2024vantrees} applies with prior $G_0$, $\psi(\b\tht)=\b\tht$, and $S=\tx_k$. Since $\int \nabla\psi\,{\r d}G_0=\mb I_d$, its Schur-complement form yields
    \[
    \E\bxs{(\tx_k-\sqrt{n_k}\b\tht_k)(\tx_k-\sqrt{n_k}\b\tht_k)^\top}
    \succeq
    \left(\c J_0+\E_{G_0}\bxs{\c I_k^{\mathrm{obs}}(\b\tht_k)}\right)^{-1}.
    \]
    Taking traces gives the scalar lower bound.
    
    It remains to bound $\c I_k^{\mathrm{obs}}$. Under the diagonal covariance model,
    \[
    \Sigma_k(\b\tht)=\diag\bigl(\sigma_{k,1}(\b\tht)^2,\ldots,\sigma_{k,d}(\b\tht)^2\bigr),
    \]
    so for $\mb x\in\R^d$,
    \[
    \log \vrp^{(k)}(\mb x;\b\tht)\propto-\frac12\sum_{\ell=1}^d\left\{\log \sigma_{k,\ell}(\b\tht)^2+\frac{(x_\ell-\sqrt{n_k}\tht_\ell)^2}{\sigma_{k,\ell}(\b\tht)^2}\right\},
    \]
    where $C_k(\b\tht)$ does not depend on $\mb x$. Hence, for each coordinate $i\le d$,
    \[
    \partial_i\log \vrp^{(k)}(\mb x;\b\tht)
    =\frac{(x_i-\sqrt{n_k}\tht_i)}{\sigma_{k,i}(\b\tht)^2}
    +\frac12\sum_{\ell=1}^d\partial_i\log \sigma_{k,\ell}(\b\tht)^2\left\{\frac{(x_\ell-\sqrt{n_k}\tht_\ell)^2}{\sigma_{k,\ell}(\b\tht)^2}-1\right\}.
    \]
    If $Z_{k,\ell}:=(x_\ell-\sqrt{n_k}\tht_\ell)/\sigma_{k,\ell}(\b\tht)$, then the random variables $Z_{k,1},\ldots,Z_{k,d}$ are independent standard Gaussians. Using $\E Z_{k,\ell}=0$, $\E(Z_{k,\ell}^2-1)=0$, and $\E(Z_{k,\ell}^2-1)^2=2$, the odd cross-terms vanish and we obtain
    \[
    \c I_k^{\mathrm{obs}}(\b\tht)
    =n_k\Sigma_k(\b\tht)^{-1}+\c B_k(\b\tht),
    \qquad
    \c B_k(\b\tht)=\frac12\sum_{\ell=1}^d \mb g_{k,\ell}(\b\tht)\mb g_{k,\ell}(\b\tht)^\top,
    \]
    where $\mb g_{k,\ell}(\b\tht):=\nabla\log \sigma_{k,\ell}(\b\tht)^2$. Therefore, for any $\mb u\in\R^d$,
    \[
    \mb u^\top\c B_k(\b\tht)\mb u
    =\frac12\sum_{\ell=1}^d(\mb u^\top\mb g_{k,\ell}(\b\tht))^2
    \le \frac12\|\mb u\|^2\sum_{\ell=1}^d\|\mb g_{k,\ell}(\b\tht)\|^2
    \le \frac{dL_\Sigma^2}{2}\|\mb u\|^2,
    \]
    and Assumption~\ref{ass:reg} gives $\Sigma_k(\b\tht)^{-1}\preceq \us^{-1}\mb I_d$. Hence
    \[
    \c I_k^{\mathrm{obs}}(\b\tht)
    \preceq
    \left(\frac{n_k}{\us}+\frac{dL_\Sigma^2}{2}\right)\mb I_d.
    \]
    Also, by symmetry
    \[
    \c J_0+\E_{G_0}\bxs{\c I_k^{\mathrm{obs}}(\b\tht_k)}
    \preceq
    \left(\nrm{\c J_0}{}+\frac{n_k}{\us}+\frac{dL_\Sigma^2}{2}\right)\mb I_d,
    \]
    so for the oracle posterior mean $\oraet_k=\E(\b\tht_k\mid\tx_k)=\E(\b\tht_k\mid\bx_k)$,
    \[
    \E\|\oraet_k-\b\tht_k\|^2
    \ge
    \frac{d}{\nrm{\c J_0}{}+\bn/\us+dL_\Sigma^2/2}.
    \]
    Averaging over $k$ yields the lower bound on ${\c R}(\Tht^o,\Tht)$, while the upper bound was proved in the total-risk decomposition above.
\end{proof}

\section{Technical Results}
\begin{proof}[Proof of Lemma \ref{lem-characterization}]
    First, we have the following \citep[Lemma 1]{soloff2025multivariate}:
    \[\frac{1}{K}\sum_{k=1}^K\frac{f_{k,G}(\tx_k)}{f_{k,\h G}(\tx_k)}\le1~~\forall\,G\in{\c P}(\R^d)\]
    and it suffices to check $D(\h G,\b t)\le0$. 
    For the first claim, define $${\c C}:=\brcs{f_{1,G}(\tx_1),...,f_{k,G}(\tx_K);G\in{\c P}(\R^d)}\cup\brcs{0}$$
    and observe that
    \[{\c C}=\operatorname{conv}({\c L})~~\text{where}~~{\c L}:=\brcs{\vrp^{(1)}(\tx_1;\b\tht),...,\vrp^{(K)}(\tx_K;\b\tht);\b\tht\in\R^d}.\]
    $\c L$ is compact and closed by the continuity and decaying tails of $\vrp^{(k)}$.
    Hence $\c C$ is compact and convex, and $\sum_{k=1}^K\log f_{k,\h G}(\tx_k)$ is strictly concave and coordinate-wise monotone over $\c C$. Thus, NPMLE is a unique boundary point $\partial\c C$. By Carathéodory's Theorem, such a point admits a convex combination $\sum_{j=1}^{\h k}\h w_j\vrp^{(k)}(\tx_k;\mb a_j)$ for some $\h k\le K$. The second claim follows directly from \cite{soloff2025multivariate}.
\end{proof}
\begin{proof}[Proof of Lemma \ref{lem-ridgeline}]
    With frozen covariances $\mb S_k$, the conclusion directly follows from \citet[Theorem~1]{ray2005topography}.
\end{proof}

\subsection{Convergence of the pseudo-EM under Lipschitz covariances}
\begin{lemma}
    Let
    \[
        r_K:=\sqrt{\frac{2\bs}{\un}\log\lrp{e\tau^{d/2}K}}.
    \]
    Under Assumption \ref{ass:reg}, every NPMLE $\h G$ solving \eqref{eq-NPMLE} satisfies
    \[
        \operatorname{supp}(\h G)\subseteq\t{\c M}:=\bigcup_{k=1}^K\bar{\bb B}_{r_K}(\bx_k).
    \]
    In particular, $\t{\c M}$ is compact with $\operatorname{diam}(\t{\c M})\le R+2r_K$, where $R:=\max_{k,\ell}\nrm{\bx_k-\bx_\ell}{}$.
\end{lemma}
\begin{proof}
    By Lemma \ref{lem-characterization}, $\operatorname{supp}(\h G)\subseteq\{\b\tht:D(\h G,\b\tht)=0\}$, so it suffices to show that $D(\h G,\b\tht)<0$ whenever $\min_{k\le K}\nrm{\b\tht-\bx_k}{}>r_K$.
    An exact NPMLE satisfies \eqref{eq-apprx-npmle-2}, since $K^{-1}\dlt_{\bx_k}+(1-K^{-1})\h G\in{\c P}(\R^d)$ does not improve on a global maximizer. So, \eqref{eq-npmle-obs-bound} gives
    \begin{align}\label{eq-rho0-loc}
        f_{k,\h G}(\tx_k)\ge\rho_0=\frac{(2\pi\bs)^{-d/2}}{eK},\quad k\le K.
    \end{align}
    Also, Assumption \ref{ass:reg} gives
    \[
        \vrp^{(k)}(\tx_k;\b\tht)\le(2\pi\us)^{-d/2}\exp\lrp{-\frac{\nrm{\sqrt{n_k}\b\tht-\tx_k}{}^2}{2\bs}}.
    \]
    Combining the two displays, bounding the average by the maximum, and using $n_k\ge\un$,
    \[
        D(\h G,\b\tht)+1
        \le\frac{(2\pi\us)^{-d/2}}{\rho_0}\max_{k\le K}\exp\lrp{-\frac{\nrm{\sqrt{n_k}\b\tht-\tx_k}{}^2}{2\bs}}
        \le eK\tau^{d/2}e^{-\frac{1}{2\bs}\min_{k\le K}\nrm{\sqrt{n_k}\b\tht-\tx_k}{}^2}.
    \]
    The right-hand side is strictly less than $1$ precisely when
    \[
        \min_{k\le K}\nrm{\sqrt{n_k}\b\tht-\tx_k}{}^2>2\bs\log\lrp{e\tau^{d/2}K},
    \]
    which proves the inclusion. Compactness of $\t{\c M}$ is clear, and the diameter bound follows from the triangle inequality.
\end{proof}

\begin{proof}[Proof of Lemma \ref{lem-vaneb-convergence}]\label{prf-lem-vaneb-convergence}
    Write $\mb P_k(\mb a):=n_k\Sigma_k(\mb a)^{-1}$, $\mb A(\mb a):=\sum_{k=1}^Kw_{kj}^{(t)}\mb P_k(\mb a)$, and $\mb b(\mb a):=\sum_{k=1}^Kw_{kj}^{(t)}\mb P_k(\mb a)\bx_k$, so $\bar{\b\psi}_j(\mb a)=\mb A(\mb a)^{-1}\mb b(\mb a)$.
    \paragraph{Second claim.}
    The resolvent identity and $\us\mb I_d\preceq\Sigma_k\preceq\bs\mb I_d$ give
    \begin{align}\label{eq-inv-pert}
        \nrm{\Sigma_k(\mb a_1)^{-1}-\Sigma_k(\mb a_2)^{-1}}{}\le\frac{L_\Sigma}{\us^2}\nrm{\mb a_1-\mb a_2}{}.
    \end{align}
    From $\mb A\bar{\b\psi}_j=\mb b$, differencing at two points gives
    \[\bar{\b\psi}_j(\mb a_1)-\bar{\b\psi}_j(\mb a_2)=\mb A(\mb a_1)^{-1}\sum_{k=1}^Kw_{kj}^{(t)}[\mb P_k(\mb a_1)-\mb P_k(\mb a_2)](\bx_k-\bar{\b\psi}_j(\mb a_2)).\]
    As we have $\mb A(\mb a_1)\succeq\sum_kw_{kj}^{(t)}n_k\bs^{-1}\mb I_d=(n_j^w/\bs)\mb I_d$, it follows that $\nrm{\mb A(\mb a_1)^{-1}}{}=\lambda_{\min}\bxs{\mb A(\mb a_1)}^{-1}\le\bs/n_j^w$. Also, $\nrm{\mb P_k(\mb a_1)-\mb P_k(\mb a_2)}{}\le n_kL_\Sigma\nrm{\mb a_1-\mb a_2}{}/\us^2$, so we have
    \begin{align*}
        \nrm{\bar{\b\psi}_j(\mb a_1)-\bar{\b\psi}_j(\mb a_2)}{}
        &\le\frac{\bs L_\Sigma}{\us^2}\left(\sum_{k=1}^K\frac{w_{kj}^{(t)}n_k}{n_j^w}\nrm{\bx_k-\bar{\b\psi}_j(\mb a_2)}{}\right)\nrm{\mb a_1-\mb a_2}{}.
    \end{align*}
    Write $\alpha_k:=w_{kj}^{(t)}n_k/n_j^w$, so that $\sum_k\alpha_k=1$. Because $\mb A(\mb a_2)^{-1}\sum_kw_{kj}^{(t)}\mb P_k(\mb a_2)=\mb I_d$, for every $\mb u\in\R^d$,
    \[
        \bar{\b\psi}_j(\mb a_2)-\mb u=\sum_{k=1}^K\mb A(\mb a_2)^{-1}w_{kj}^{(t)}\mb P_k(\mb a_2)(\bx_k-\mb u),
    \]
    and using $\nrm{\mb A(\mb a_2)^{-1}}{}\le\bs/n_j^w$ together with $\nrm{\mb P_k(\mb a_2)}{}\le n_k/\us$,
    \begin{align}\label{eq-matrix-avg}
        \nrm{\bar{\b\psi}_j(\mb a_2)-\mb u}{}\le\tau\sum_{k=1}^K\alpha_k\nrm{\bx_k-\mb u}{}
    \end{align}
    for the condition number $\tau=\bs/\us$.
    Combining \eqref{eq-matrix-avg} with the triangle inequality $\nrm{\bx_k-\bar{\b\psi}_j(\mb a_2)}{}\le\nrm{\bx_k-\mb u}{}+\nrm{\mb u-\bar{\b\psi}_j(\mb a_2)}{}$ and averaging over $k$ with weights $\alpha_k$,
    \[
        \sum_{k=1}^K\alpha_k\nrm{\bx_k-\bar{\b\psi}_j(\mb a_2)}{}
        \le\lrp{1+\tau}\sum_{k=1}^K\alpha_k\nrm{\bx_k-\mb u}{}.
    \]
    Taking the infimum over $\mb u\in\R^d$ gives
    \[
        \sum_{k=1}^K\frac{w_{kj}^{(t)}n_k}{n_j^w}\nrm{\bx_k-\bar{\b\psi}_j(\mb a_2)}{}\le\lrp{1+\tau}R_j^w,
    \]
    and therefore $\nrm{\bar{\b\psi}_j(\mb a_1)-\bar{\b\psi}_j(\mb a_2)}{}\le L_{\bar\psi,j}\nrm{\mb a_1-\mb a_2}{}$ for $L_{\bar\psi,j}=\tau(1+\tau)R_j^w L_\Sigma/\us$. If $L_{\bar\psi,j}<1$, Banach's fixed-point theorem yields uniqueness and linear convergence at rate $L_{\bar\psi,j}$.
    \paragraph{First claim.}
    For the first claim, note that $$\mb A(\mb a)^{-1}\sum_kw_{kj}^{(t)}\mb P_k(\mb a)=\mb I_d,$$ so $\bar{\b\psi}_j(\mb a)$ is a matrix-weighted average of $\{\bx_k\}$ with weight matrices summing to the identity. By \eqref{eq-matrix-avg}, applied at $\mb u=\bx_1$, every value of the map satisfies $\nrm{\bar{\b\psi}_j(\mb a)-\bx_1}{}\le\tau R$, uniformly in $\mb a$. Hence $\bar{\b\psi}_j$ maps the compact convex ball $\bb B_{\tau R}(\bx_1)$ into itself, and it is continuous. Brouwer fixed-point theorem yields a fixed point $\mb a_j^*$ such that $\bar{\b\psi}_j(\mb a_j^*)=\mb a_j^*$, which is the stationarity equation itself.

    Now we bound $R_j^w$ under \eqref{eq:like}. It is convenient to work with the second-order radius
    \[
        R_j^{w,2}:=\Bigl(\inf_{\mb u\in\R^d}\sum_{k=1}^K\alpha_k\nrm{\bx_k-\mb u}{}^2\Bigr)^{1/2}\ \ge\ R_j^w,
    \]
    from Cauchy--Schwarz. Write $T_{kj}:=\sum_\ell(\underline{\h\tht}_{k,\ell}-\sqrt{n_k}a_{j,\ell}^{(t)})^2/\sigma_{k,\ell}(\mb a_j^{(t)})^2$. 
    First we acknowledge that the responsibility in Algorithm \ref{alg:vaneb} is $r_{kj}=\frac{\h w_j\vrp^{(k)}(\tx_k;\mb a_j^{(t})}{\h L_k}$ and $\sum_jr_{kj}=1$. In the M-step, we are solving an equation $\sum_jr_{kj}\mb s_k(\mb a)=\mb0_d$ to update the atoms hence it is natural to choose $w_{kj}^{(t)}=r_{kj}/\sum_{k'}r_{k'j}$, so for $w_k=\h L_k^{-1}/\sum_i\h L_i^{-1}$,
    \[
    w_{kj}^{(t)}\propto w_k\vrp^{(k)}(\tx_k;\mb a_j^{(t)})
    =w_k(2\pi)^{-d/2}\Bigl(\prod_{\ell=1}^d\sigma_{k,\ell}(\mb a_j^{(t)})\Bigr)^{-1}\exp(-T_{kj}/2).
    \]
    Note that $\us^{d/2}\le\prod_\ell\sigma_{k,\ell}\le\bs^{d/2}$, while $w_k/w_{k'}=\h L_{k'}/\h L_k\le eK\tau^{d/2}$ because $\h L_{k'}\le(2\pi\us)^{-d/2}$ and $\h L_k\ge\rho_0=(2\pi\bs)^{-d/2}/(eK)$ by assumption. Hence, for any pair $k,k'$,
    \begin{align}\label{eq-weight-ratio}
        \frac{w_{kj}^{(t)}}{w_{k'j}^{(t)}}\le eK\tau^{d/2}\exp\brcs{-\tfrac12(T_{kj}-T_{k'j})}.
    \end{align}
    Let $k_*=k_*(j)\in\argmin_kT_{kj}$ and take $\mb u=\bx_{k_*}$ in the definition of $R_j^{w,2}$.
    With the threshold $\Lambda:=2\log\bigl(e\tau^{d/2}K^2\un(1+R)^2\bigr)$, split the clients into $\c N_j:=\{k:T_{kj}\le T_{k_*j}+\Lambda\}$ and its complement.
    
    For $k\notin\c N_j$, \eqref{eq-weight-ratio} applied with $k'=k_*$ gives $w_{kj}^{(t)}\le w_{k_*j}^{(t)}/\bigl(K\un(1+R)^2\bigr)$, so, using $\alpha_k\le(\bn/\un)w_{kj}^{(t)}$ and $\nrm{\bx_k-\bx_{k_*}}{}\le R$, the far clients contribute at most $(\bn/\un)R^2/\bigl(\un(1+R)^2\bigr)\le\bn/\un^2$ to $(R_j^{w,2})^2$.
    
    For $k\in{\c N}_j$, we have $\un\nrm{\bx_k-\mb a_j^{(t)}}{}^2/\bs\le T_{kj}\le T_{k_*j}+\Lambda$, so
    \[
        \nrm{\bx_k-\bx_{k_*}}{}\le\nrm{\bx_k-\mb a_j^{(t)}}{}+\nrm{\bx_{k_*}-\mb a_j^{(t)}}{}
        \le2\sqrt{\frac{\bs(T_{k_*j}+\Lambda)}{\un}} .
    \]
    Moreover, $T_{k_*j}\le\bn\,D_j^2/\us$ with $D_j:=\min_k\nrm{\bx_k-\mb a_j^{(t)}}{}$. Combining the two regimes,
    \begin{align}\label{eq-Rw2-bound}
        R_j^{w,2}\lesssim\sqrt{\frac{\bs}{\un}\brcs{\frac{\bn D_j^2}{\us}+2\log\bigl(e\tau^{d/2}K^2\un(1+R)^2\bigr)}}+\frac{\sqrt{\bn}}{\un}.
    \end{align}
    We assume that $\mb a_j^{(t)}\in\t{\c M}=\bigcup_k\bar{\bb B}_{r_K}(\bx_k)$, which implies $D_j\le r_K$. Since $\bn/\un=\cO(1)$, \eqref{eq-Rw2-bound} leads to
    \[  
        R_j^w={\c O}_P\bxs{\sqrt{\frac{\bs\tau}{\un}\log\bigl(e\tau^{d/2}K^2\un(1+R)^2\bigr)}}.
    \]
    Hence $L_{\bar\psi,j}=o_P(1)$ as $\un\to\infty$ whenever $\log K(1+R)=o_P(\un)$.
    \paragraph{Order of the random diameter $R$}
    Write $\b\eps_k:=\sqrt{n_k}(\bx_k-\b\tht_k)$ and $R_\tht:=\max_{k,\ell}\nrm{\b\tht_k-\b\tht_\ell}{}$. The triangle inequality gives
    \[
        R_\tht-2\max_kn_k^{-1/2}\nrm{\b\eps_k}{}\le R\le R_\tht+2\max_kn_k^{-1/2}\nrm{\b\eps_k}{}.
    \]
    Under our working model \eqref{eq:like}, the Gaussian concentration bound gives $\P(\nrm{\b\eps_k}{}>\sqrt{\bs}(\sqrt{d}+t))\le e^{-t^2/2}$ for $t>0$ and a union bound over $k\le K$ yields
    \begin{align}\label{eq-noise-max}
        \max_{k\le K}n_k^{-1/2}\nrm{\b\eps_k}{}
        =\cO_P\brcs{\sqrt{\frac{\bs(d+\log K)}{\un}}}.
    \end{align}
    It remains to bound $R_\tht$ under conditions on $G_0$. Note that we proceed under $K=o(e^\un)$.
    \begin{enumerate}[label=(\alph*)]
        \item If $R_\tht\le D_0$ a.s., then $R\le D_0+\cO_P(\sqrt{\bs(d+\log K)/\un})=\cO_P(1)$ and $\log(1+R)=\cO_P(1)$.
        \item If $G_0$ is $\sigma_0^2$-sub-Gaussian with mean $\b\mu$, then $\P(\nrm{\b\tht-\b\mu}{}>\sigma_0(\sqrt{d}+t))\le 2e^{-t^2/2}$. A union bound gives $\max_k\nrm{\b\tht_k-\b\mu}{}=\cO_P(\sigma_0\sqrt{d+\log K})$, so $R_\tht=\cO_P(\sigma_0\sqrt{d+\log K})$. Combined with \eqref{eq-noise-max},
        \[
            R=\cO_P\brcs{\log(d+\log K)}.
        \]
    \end{enumerate}
    So, $\log K(1+R)=o_P(\un)$ under considered cases and the condition $K=o(e^\un)$ hence $L_{\bar\psi,j}=o_P(1)$.
    \paragraph{Last claim.}
    Write $\b f_j(\mb a):=\sum_{k=1}^Kw_{kj}^{(t)}\mb s_k(\mb a)$ for the full frozen-weight score and write
    \[
        \mb r_k(\mb a):=\sum_{\ell=1}^d\brcs{\frac{(\underline{\h\tht}_{k,\ell}-\sqrt{n_k}a_\ell)^2}{\sigma_{k,\ell}(\mb a)^2}-1}\frac{\nabla\sigma_{k,\ell}(\mb a)}{\sigma_{k,\ell}(\mb a)},
        \qquad \b\rho_j(\mb a):=\sum_{k=1}^Kw_{kj}^{(t)}\mb r_k(\mb a).
    \]
    So,
    \begin{align}\label{eq-score-split}
        \b f_j(\mb a)=\mb A(\mb a)\bigl(\bar{\b\psi}_j(\mb a)-\mb a\bigr)+\b\rho_j(\mb a).
    \end{align}
    Because $\mb A(\mb a)$ is invertible, $\b f_j(\mb a)=\b0_d$ if and only if $\mb a$ is a fixed point of the perturbed map
    \[
        \b\Phi_j(\mb a):=\bar{\b\psi}_j(\mb a)+\mb A(\mb a)^{-1}\b\rho_j(\mb a).
    \]
    We therefore only need to show that $\b\Phi_j$, a perturbation of $\bar{\b\psi}_j$, has a fixed point near $\mb a_j^*$.

    Put $T_{kj}^2(\mb a):=\sum_\ell(\underline{\h\tht}_{k,\ell}-\sqrt{n_k}a_\ell)^2/\sigma_{k,\ell}(\mb a)^2$ and $T_j^2(\mb a):=\sum_kw_{kj}^{(t)}T_{kj}^2(\mb a)$, and abbreviate $T_j^2:=T_j^2(\mb a_j^*)$.
    Assumption \ref{ass:reg} bounds $\nrm{\nabla\log\sigma_{k,\ell}^2}{}\le L_\Sigma$, i.e.\ $\nrm{\nabla\sigma_{k,\ell}}{}/\sigma_{k,\ell}\le L_\Sigma/2$, so by the triangle inequality and $|u-1|\le u+1$,
    \begin{align}\label{eq-full-score-residual}
        \nrm{\mb r_k(\mb a)}{}\le\frac{L_\Sigma}{2}\brcs{T_{kj}^2(\mb a)+d},
        \quad
        \nrm{\mb A(\mb a)^{-1}\b\rho_j(\mb a)}{}\le\frac{\bs L_\Sigma}{2n_j^w}\brcs{T_j^2(\mb a)+d}.
    \end{align}
    Note that 
    \[
        T_j^2=\sum_kw_{kj}^{(t)}T_{kj}^2\le\frac{n_j^w}{\us}\sum_k\alpha_k\nrm{\bx_k-\mb a_j^*}{}^2.
    \]
    Let $\mb u^*$ attain the infimum defining $R_j^{w,2}=\inf_{\mb u}\sum_k\alpha_k\nrm{\bx_k-\mb u}{}^2$. Using $\mb a_j^*=\bar{\b\psi}_j(\mb a_j^*)$ and \eqref{eq-matrix-avg} at $\mb u=\mb u^*$ gives $\nrm{\mb a_j^*-\mb u^*}{}\le\tau\sum_k\alpha_k\nrm{\bx_k-\mb u^*}{}\le\tau R_j^{w,2}$ by applying Cauchy-Schwarz on $\brcs{\sqrt{\alpha_k}},\brcs{\sqrt{\alpha_k}(\bx_k-\mb u^*)}$. Therefore,
    \[
        \sum_k\alpha_k\nrm{\bx_k-\mb a_j^*}{}^2\le2(R_j^{w,2})^2+2\nrm{\mb u^*-\mb a_j^*}{}^2\le2(1+\tau^2)(R_j^{w,2})^2
    \]
    hence
    \begin{align}\label{eq-Tj2-bound}
        T_j^2=\cO_P\brcs{\frac{n_j^w\tau^2(1+\tau^2)}{\un}\log\bigl(e\tau^{d/2}K^2\un(1+R)^2\bigr)}.
    \end{align}
    Set
    \[
        \bar\eps_j:=\frac{\bs L_\Sigma}{2n_j^w}\bigl(2\tau T_j^2+d+2\bigr),\qquad
        \varrho_j:=\frac{\bar\eps_j}{1-L_{\bar\psi,j}},
    \]
    and write $L_{\log}:=\log\bigl(e\tau^{d/2}K^2\un(1+R)^2\bigr)$. From \eqref{eq-Tj2-bound} and $L_{\bar\psi,j}=o_P(1)$,
    \begin{align}\label{eq-varrho-rate}
        \varrho_j
        =\cO_P\!\lrp{\frac{\bs L_\Sigma}{\un}\Bigl(\tau^2(1+\tau^2)\,L_{\log}+d\Bigr)}
        =\cO_P\!\lrp{\frac{L_{\log}}{\un}}.
    \end{align}
    The condition needed below is $\varrho_j\sqrt{\bn/\us}\le1$. Since $\sqrt{\bn/\us}\asymp\sqrt{\un}$,
    \[
        \varrho_j\sqrt{\frac{\bn}{\us}}=\cO_P(L_{\rm log}/\sqrt{\un}).
    \]
    For $\mb a\in\bb B_{\varrho_j}(\mb a_j^*)$, we have
    \begin{align*}
        \Sigma_k(\mb a)^{-1/2}(\tx_k-\sqrt{n_k}\mb a)
        &=\Sigma_k(\mb a)^{-1/2}\Sigma_k(\mb a_j^*)^{1/2}\cdot\Sigma_k(\mb a_j^*)^{-1/2}(\tx_k-\sqrt{n_k}\mb a_j^*)\\
        &\qquad+\sqrt{n_k}\,\Sigma_k(\mb a)^{-1/2}(\mb a_j^*-\mb a)
    \end{align*}
    hence
    \begin{align*}
        T_{kj}(\mb a)
        &=\nrm{\Sigma_k(\mb a)^{-1/2}(\tx_k-\sqrt{n_k}\mb a)}{}\\
        &\le\nrm{\Sigma_k(\mb a)^{-1/2}\Sigma_k(\mb a_j^*)^{1/2}}{}\cdot T_{kj}(\mb a_j^*)
        +\sqrt{n_k}\,\nrm{\Sigma_k(\mb a)^{-1/2}}{}\,\nrm{\mb a_j^*-\mb a}{}.
    \end{align*}
    Further, on $\brcs{\varrho_j\sqrt{\bn/\us}\le1}$, we have
    \[
        T_{kj}(\mb a)\le\sqrt{\tau}\,T_{kj}(\mb a_j^*)+1,
    \]
    using $\nrm{\Sigma_k(\mb a)^{-1/2}\Sigma_k(\mb a_j^*)^{1/2}}{}\le\sqrt{\tau}$. Squaring and averaging over $w_{kj}^{(t)}$ gives $T_j^2(\mb a)\le2\tau T_j^2+2$, so \eqref{eq-full-score-residual} yields $\nrm{\mb A(\mb a)^{-1}\b\rho_j(\mb a)}{}\le\bar\eps_j$ throughout the ball. Using $\bar{\b\psi}_j(\mb a_j^*)=\mb a_j^*$ and the Lipschitz bound of Step~3,
    \[
        \nrm{\b\Phi_j(\mb a)-\mb a_j^*}{}
        \le\nrm{\bar{\b\psi}_j(\mb a)-\bar{\b\psi}_j(\mb a_j^*)}{}+\bar\eps_j
        \le L_{\bar\psi,j}\varrho_j+\bar\eps_j=\varrho_j.
    \]
    Thus $\b\Phi_j$ maps the compact convex ball $\bb B_{\varrho_j}(\mb a_j^*)$ into itself and is continuous, so Brouwer's fixed-point theorem yields a fixed point $\mb a_{j0}^*$ in that ball. In particular
    \begin{align}\label{eq-ift-bound}
        \nrm{\mb a_{j0}^*-\mb a_j^*}{}\le\varrho_j,
    \end{align}
    which is \eqref{eq-full-score-distance} by \eqref{eq-varrho-rate}. This completes the proof.
\end{proof}

\begin{proof}[Proof of Proposition \ref{prop-approximation}]
    For each $j\in[\h k]$, let $C_j\in \c H$ such that $\h{\mb a}_j\in C_j$. Fix $\mb u\in \frac{1}{\sqrt{n_k}}\cdot C_j$ and $k\in [K]$, and let 
    \begin{align*}
        \mb z_j(\mb u) := \Sigma_k(\mb u)^{-1/2}(\tx_k - \b{\h a}_j),\quad\mb y_j(\mb u) := \Sigma_k(\mb u)^{-1/2}(\sqrt{n_k}\mb u - \b{\h a}_j).
    \end{align*}
    Let $H_j$ be a positive distribution supported on the corners of $C_j$ such that $H_j(C_j) = \h w_j$ and 
    \begin{align*}
        \int_{C_j}\lrp{1+\langle\mb y_j(\mb u),\mb z_j(\mb u)\rangle}|2\pi\Sigma_k(\mb u)|^{-1/2}\exp\lrp{-\frac{\nrm{\mb z_j(\mb u)}{}^2}{2}}{\r d}H_j(\mb u)\\
        \brcs{=\h w_j\vrp^{(k)}(\tx;\h{\mb a}_j)}=\int_{C_j}\vrp^{(k)}(\tx_k;{\mb u}){\r d}\h G^j(\mb u)
    \end{align*}
    for all $k\le K$ where $\h G^j := \h{w}_j\dlt_{\h{\mb a}_j}$. 
    By definitions $\nrm{\mb y_j}{}\le\sqrt{\frac{d}{\us}}\dlt$ and $\nrm{\mb z_j}{}\le\sqrt{\frac{1}{\us}}D$ hence $\langle\mb z_j,\mb y_j\rangle^2\le1$ by construction. 
    Let $\phi_d(\cdot)$ denote the density of a $\c N(0,I_d)$ distribution. First, we note that
    \begin{align*}
        \vrp^{(k)}(\tx_k;\mb u)&=|\Sigma_k(\mb u)|^{-1/2}\phi_d(\mb z_j-\mb y_j) \\
        &= |\Sigma_k(\mb u)|^{-1/2}\phi_d(\mb z_j)\exp(\langle\mb z_j,\mb y_j\rangle-\nrm{\mb y_j}{}^2/2)
    \end{align*}
    and following \citet[A.27]{jiang2009general}:
    \begin{align*}
        (e^{-\nrm{\mb y_j}{}^2/2}-1)e^{\langle\mb z_j,\mb y_j\rangle}\le e^{-\nrm{\mb y_j}{}^2/2+\langle\mb z_j,\mb y_j\rangle}-(1+\langle\mb z_j,\mb y_j\rangle) \\
        \le\langle \mb z_j,\mb y_j\rangle^2\underbrace{e^{\langle\mb z_j,\mb y_j\rangle-\nrm{\mb y_j}{}^2/2}}_{=:g}, \\
        (\therefore)\quad g(1-e^{\nrm{\mb y_j}{}^2/2})\le g-(1+\langle\mb z_j,\mb z_j\rangle)\le\langle \mb z_j,\mb y_j\rangle^2g
    \end{align*}
    so we arrive at
    \begin{align*}
        \vrp^{(k)}(\tx_k;\mb u)(1-e^{\nrm{\mb y_j}{}^2/2})
        &\le\vrp^{(k)}(\tx_k;\mb u)-(1+\langle\mb z_j,\mb y_j\rangle)|\Sigma_k(\mb u)|^{-1/2}\phi_d(\mb z_j) \\
        &\le\langle \mb z_j,\mb y_j\rangle^2\vrp^{(k)}(\tx_k;\mb u)\le\vrp^{(k)}(\tx_k;\mb u).
    \end{align*}
    Hence, since $\int_{C_j}\vrp\,{\r d}\h G^j = \h w_j\vrp^{(k)}(\tx_k;\h{\mb a}_j)$ is a single-atom evaluation while $\int_{C_j}\vrp\,{\r d}H_j$ averages over corners, the sandwich above gives
    \begin{align*}
        &\h w_j\vrp^{(k)}(\tx_k;\h{\mb a}_j)- \int_{C_j} \vrp^{(k)}(\tx_k;\mb u){\r d}H_j(\mb u) \\ 
        &\le \frac{dD^2\dlt^2}{\us^2}
        \h w_j\vrp^{(k)}(\tx_k;\h{\mb a}_j)+\int_{C_j}\lrp{e^{\|\mb y_j\|^2/2}-1}\vrp^{(k)}(\tx_k;\mb u)H_j(\mb u) \\
        &\le \frac{dD^2\dlt^2}{\us^2}
        \int_{C_j} \vrp^{(k)}(\tx_k;\mb u){\r d}\h G^j(\mb u)
        +\left(e^{d\dlt^2/(2\us)}-1\right)\int_{C_j}\vrp^{(k)}(\tx_k;\mb u){\r d}H_j(\mb u).
    \end{align*}
    Let $H = \sum_{j=1}^{\h k}H_j$ and let $\h G = \sum_{j=1}^{\h k} \h{w}_j\dlt_{\h{\mb a}_j}$.
    Summing the above inequality over $j$, 
    \begin{align*}
        f_{k,\h G}(\tx_k) - f_\HSk(\tx_k) 
        \le \frac{dD^2\dlt^2}{\us^2}f_{k,\h G}(\tx_k)
        +\left(e^{\frac{d\dlt^2}{2\us}}-1\right)f_{k,H}(\tx_k),~i.e., \\
        f_{k,\h G}(\tx_k)\lrp{1-\frac{dD^2\dlt^2}{\us^2}} \le e^{d\dlt^2/(2\us)}f_\HSk(\tx_k).
    \end{align*}
    As $H$ is supported on $\c A$, by optimality: 
    \[
    \prod_{k=1}^K f_{k,\h G^{\c A}}(\tx_k) \ge \prod_{k=1}^K f_\HSk(\tx_k).
    \]
    Combining our findings,  
    \[
    \prod_{k=1}^K f_{k,\h G^{\c A}}(\tx_k) 
    \ge e^{-Kd\dlt^2/(2\us)}\prod_k\left(1-\frac{dD^2\dlt^2}{\us^2}\right)\prod_{k=1}^K f_{k,\h G}(\tx_k).
    \]
    An elementary inequality $1-x\ge e^{-2x}$ (for $x\le 3/4$) gives
    \[
    \prod_{k=1}^K f_{k,\h G^{\c A}}(\tx_k) 
    \ge \exp\left(-\frac{Kd\dlt^2}{2\us} -\frac{2KdD^2\dlt^2}{\us^2}\right)\prod_{k=1}^K f_{k,\h G}(\tx_k)
    \]
    under the given upper limit for $\dlt$.
\end{proof}
\begin{lemma}\label{p1}
    Fix a probability measure $G$ on $\R^d$. We have
    \begin{equation}\label{p1.eq}
        \frac{1}{\bs}\left( \frac{\nrm{\mb q_\GSk(\tx_k)}{}}{f_\GSk(\tx_k)} \right)^2 \le \tr\left(\mb I_d + \frac{\nabla\mb q_\GSk(\tx)}{f_\GSk(\tx_k)} \right) \le \log \frac{(2\pi\us)^{-d}}{f_\GSk(\tx_k)^2}.
    \end{equation}
    Also for every ${\bx}_k \in \R^d$ and a measure $G'$ on $\R^d$, we have
    \begin{align}
        \frac{\nrm{\mb q_{k,G}(\tx_k)}{}}{f_{k,G}(\tx_k)\vee\rho} \le\sqrt{\bs\log \frac{1}{(2\pi\us)^d\rho^2}}
        \label{p1.neq}
    \end{align}    
    for $\rho\le\frac{1}{\sqrt{e}(2\pi\us)^{d/2}}$.
\end{lemma}
\begin{proof}
    If $\b\tht\sim G$ and $\tx_k| \b\tht \sim {\c N}\lrp{\sqrt{n_k}\b\tht, \Sigma_k(\b\tht)}$, then it follows that
    \begin{equation*}
        \frac{\mb q_\GSk(\tx_k)}{f_\GSk(\tx_k)} = \E \left(\sqrt{n_k}\b\tht - \tx_k | \tx_k \right)
    \end{equation*}
    and 
    \begin{equation}
        \frac{\nabla\mb q_\GSk(\tx_k)}{f_\GSk(\tx_k)} = -\mb I_d + \E\bxs{(\sqrt{n_k}\b\tht - \tx_k) (\sqrt{n_k}\b\tht -
        \tx_k)'\Sigma_k(\b\tht)^{-1} | \tx_k}.
    \end{equation}
    It follows that
    \begin{align*}
        \tr\lrp{\mb I_d + \frac{\nabla\mb q_\GSk(\tx_k)}{f_\GSk(\tx_k)}} =\E[\nrm{\Sigma_k(\b\tht)^{-1/2}(\sqrt{n_k}\b\tht-\tx_k)}{}^2|\tx_k]\ge\frac{\nrm{\mb q_\GSk(\tx_k)}{}^2}{\bs f_\GSk(\tx_k)^2}
    \end{align*}
    and
    \begin{align*}
        \exp\lrp{\frac{1}{2} \E[\nrm{\Sigma_k(\b\tht)^{-1/2}(\sqrt{n_k}\b\tht-\tx_k)}{}^2|\tx_k]}&\le \E\lrp{e^{\frac{1}{2}\nrm{\Sigma_k(\b\tht)^{-1/2}(\tx_k-\sqrt{n_k}\b\tht)}{}^2}| \tx_k}\\
        &\qquad=\int\frac{|2\pi\Sigma_k(\mb u)|^{-1/2}}{f_{k,G}(\tx_k)}{\r d}G(\mb u) \le \frac{(2 \pi\us)^{-d/2}}{f_\GSk(\tx_k)}
    \end{align*}
    so that we have 
    \begin{equation*}
        \tr\left(\mb I_d + \frac{\nabla\mb q_\GSk(\tx_k)}{f_\GSk(\tx_k)} \right) \le \log \frac{(2
        \pi\us)^{-d}}{f_\GSk(\tx_k)^2}. 
    \end{equation*}
    To prove \eqref{p1.neq}, note first from \eqref{p1.eq} that 
    \begin{align*}
        \frac{\nrm{\mb q_\GSk(\tx)}{}}{f_\GSk(\tx_k)\vee\rho} &\le \sqrt{\bs\log\frac{(2\pi\us)^{-d}}{f_\GSk(\tx_k)^2}}
        \frac{f_\GSk(\tx_k)}{f_\GSk(\tx_k)\vee\rho} \\ 
        &\qquad=\begin{cases}
            \sqrt{\bs\log\frac{(2\pi\us)^{-d}}{f_\GSk(\tx_k)^2}}\le\sqrt{\bs\log \frac{1}{(2\pi\us)^d\rho^2}}     & \text{if } f_\GSk(\tx_k) > \rho \\ 
            \sqrt{\bs\log \frac{(2\pi\us)^{-d}}{f_\GSk(\tx_k)^2} } \frac{f_\GSk(\tx_k)}{\rho} & \text{o/w.}
        \end{cases}
    \end{align*}
    The function $v \mapsto \sqrt{\log\frac{(2\pi\us)^{-d}}{v^2}}v$ is
    non-decreasing on $(0,\frac{1}{e(2\pi\us)^d}]$ hence when $f_\GSk(\tx_k)^2 \le\rho^2\le \frac{1}{e(2\pi\us)^d}$, the inequality 
    \begin{equation*}
        \sqrt{\bs\log \frac{(2\pi\us)^{-d}}{f_\GSk(\tx_k)^2} } \frac{f_\GSk(\tx_k)}{\rho} \le \sqrt{\bs\log \frac{1}{(2\pi\us)^d\rho^2}}
    \end{equation*}
    holds and this proves \eqref{p1.neq}. 
\end{proof}
The following lemma establishes a Lipschitz-type bound for the heteroskedastic Gaussian kernels that has no analogue in the fixed-covariance literature. Unlike the standard setting where the covariance is fixed, the gradient with respect to the parameter $\mb u$ now requires controlling the Jacobian of the variance function. This bound is the key new ingredient in our metric entropy estimates (Lemma \ref{lem-entropy}).
\begin{lemma}\label{lem:mvt}
    Let $\nabla_{\mb u}$ denote $\frac{\partial}{\partial\mb u}$ and $\mb J:=\begin{bmatrix}
        \frac{\partial \sigma_{k,1}(\mb u)}{\partial u_1} & ... & \frac{\partial \sigma_{k,1}(\mb u)}{\partial u_d} \\
        \vdots & \ddots & \vdots \\
        \frac{\partial \sigma_{k,d}(\mb u)}{\partial u_1} & ... & \frac{\partial \sigma_{k,d}(\mb u)}{\partial u_d}
    \end{bmatrix}$.
    For any points $\mb x,\mb a,\mb b\in\R^d$, we have
    \[\vrt{\vrp^{(k)}(\mb x;\mb a) - \vrp^{(k)}(\mb x;\mb b)}\le_d\nrm{\mb a-\mb b}{}\frac{\sqrt{\bn}+2d^{3/2}\bs'}{\us^{(d+1)/2}}\]
    for $\bs':=\nrm{\mb J}{}$.
\end{lemma}
\begin{proof}
    Let $\phi_d$ denote the density of a $d$-dimensional standard normal distribution. We use mean value theorem: there exists some $\mb u$ which is some linear interpolation of $\mb z_{\mb a}:=\Sigma_k(\mb a)^{-1/2}\lrp{\mb x-\sqrt{n_k}\mb a}$, $\mb z_{\mb b}$ such that
    \begin{align*}
        &\frac{\phi_d(\mb z_{\mb a})}{|\Sigma_k(\mb a)|^{1/2}} - \frac{\phi_d(\mb z_{\mb b})}{|\Sigma_k(\mb b)|^{1/2}}=\nabla_{\mb u}\lrp{\frac{\phi_d(\mb z_{\mb u})}{|\Sigma_k(\mb u)|^{1/2}}}'\lrp{\mb a-\mb b} \\
        &\therefore~\vrt{\vrp^{(k)}(\mb x;\mb a) - \vrp^{(k)}(\mb x;\mb b)}\le\nrm{\mb a-\mb b}{}\sup_{\mb u\in\R^d}\nabla_{\mb u}\lrp{\frac{\phi_d(\mb z_{\mb u})}{|\Sigma_k(\mb u)|^{1/2}}}.
    \end{align*}
    Now, observe that
    \begin{align}\label{eq:mvt1}
        \nabla_{\mb u}\lrp{\frac{\phi_d(\mb z_{\mb u})}{|\Sigma_k(\mb u)|^{1/2}}}=\frac{1}{|\Sigma_k(\mb u)|^{\frac{1}{2}}}\brcs{\nabla_{\mb u}\phi_d(\mb z_{\mb u})-\frac{d}{2}\phi_d(\mb z_{\mb u})\nabla_{\mb u}\log|\Sigma_k(\mb u)|}
    \end{align}
    and using
    \[\nabla_{\mb u} \mb z_{\mb u} = - \sqrt{n_k} \Sigma_k(\mb u)^{-1/2} - \Sigma_k(\mb u)^{-1/2}\diag(\mb z_u)\begin{bmatrix} \nabla_{\mb u'} \sigma_{k,1}(\mb u) \\ \vdots \\ \nabla_{\mb u'} \sigma_{k,d}(\mb u) \end{bmatrix}_{d\times d},\]
    we have
    \begin{align*}
        \nrm{\nabla_{\mb u}\phi_d(\mb z_{\mb u})}{}=\frac{\phi_d(\mb z_{\mb u})}{2}\lrnrm{\nabla_{\mb u}\nrm{\mb z_{\mb u}}{}^2}{}
        \le\frac{\phi_d(\mb z_{\mb u})}{2}\frac{2\nrm{\mb z_{\mb u}}{}}{\us^{1/2}}\lrp{\sqrt{n_k}+\nrm{\mb J}{}\cdot\nrm{\mb z_{\mb u}}{}}.
    \end{align*}
    Also,
    \begin{align*}
        \nrm{\nabla_{\mb u}\log|\Sigma_k(\mb u)|}{}=2\lrnrm{\mb J'\Sigma(\mb u)^{-1/2}\b1_d}{}\le\frac{2\sqrt{d}}{\us^{1/2}}\nrm{\mb J}{}\\
    \end{align*}
    hence
    \begin{align*}
        \nrm{\eqref{eq:mvt1}}{}\le\frac{\phi_d(\mb z_{\mb u})}{\us^{\frac{d+1}{2}}}\brcs{\nrm{\mb z_{\mb u}}{}\lrp{\sqrt{n_k}+\bs'\nrm{\mb z_{\mb u}}{}}+d^{3/2}\bs'}
        \le_d\frac{1}{\us^{(d+1)/2}}\lrp{\frac{\sqrt{\bn}}{\sqrt{e}}+\frac{2\bs'}{e}+d^{3/2}\bs'}\\
        \le\frac{\sqrt{\bn}+2d^{3/2}\bs'}{\us^{(d+1)/2}}
    \end{align*}
    uniformly over $\mb u\in\R^d$.
\end{proof}
\begin{lemma}\label{lem:cvx}
    For a probability measure $G$ on $\R^d$ and $\rho > 0$, let 
    \begin{equation*}
        \t\Dlt_k(G, \rho) := \int \left(1 - \frac{f_\GSk}{f_\GSk\vee\rho}
        \right)^2\frac{\nrm{\mb q_\GSk}{}^2}{f_\GSk}.
    \end{equation*}
    Then,
    \begin{enumerate}
        \item $\t\Dlt_k(G, \rho) \le d\bs$ for all $k\le K$. 
        \item Suppose $G = \sum_{j=1}^m w_j H_j$ for some measures
        $H_1, \dots, H_m$ and weights $w_1, \dots, w_m$. Then 
        \begin{equation}\label{cvx.eq}
            \t\Dlt_k(G, \rho) \le \sum_{j=1}^m w_j \t\Dlt_k \left(H_j, \rho/w_j
            \right). 
        \end{equation}
    \end{enumerate}
\end{lemma}
\begin{proof}[Proof of Lemma \ref{lem:cvx}] 
    Let us suppress the client index $(k)$ for simplicity.
    Note that if $\b\tht \sim G$ and
    $\tx | \b\tht \sim N\lrp{\sqrt{n}\b\tht,\Sigma_k(\b\tht)}$, then
    \begin{equation*}
        \frac{\mb q_{k,G}(\tx)}{f_{k,G}(\tx)}  = \E\lrp{\sqrt{n}\b\tht - \tx | \tx}.
    \end{equation*}
    As a result 
    \begin{align*}
        \t\Dlt_k(G, \rho) \le \int_{\R^d}\nrm{\E\lrp{\sqrt{n}\b\tht - \tx | \tx}}{}^2f_{k,G}(\tx){\r d}\tx 
        =\E[\nrm{\E(\sqrt{n}\b\tht - \tx | \tx)}{}^2] \\
        \le \E[\nrm{\sqrt{n}\b\tht - \tx}{}^2] = \tr\E[\Sigma_k(\b\tht)].
    \end{align*}
    The convexity inequality \eqref{cvx.eq} is a direct consequence of the mixture decomposition in \citet[Lemma F.7]{saha2020nonparametric}, which applies unchanged to our setting since it only involves the pointwise relationship between $f_{k,G}$ and $\mb q_{k,G}$.
\end{proof}
\begin{lemma}\label{lem:dco}
    Fix a probability measure $G$ on $\R^d$ and let $0 < \rho \le(2\pi\bs)^{-d/2}e^{-8/\us}$.
    Let $L(\rho) := \sqrt{\us\log\frac{1}{(2\pi\bs)^d\rho^2}}$.
    Then, for every compact set $S \subseteq \R^d$, we have 
    \begin{align}
        \t\Dlt_k\lrp{G, \rho}\le_{d,\us,\bs} N\lrp{\frac{4}{\sqrt{n_k}L(\rho)},S}\rho L(\rho)^d + G(S^c).
        \label{dco.eq}
    \end{align}
\end{lemma}
\begin{proof}[Proof of Lemma \ref{lem:dco}]
    The bound follows by the same localization argument as \citet[Lemma 4.3]{saha2020nonparametric}: one decomposes $\t\Dlt_k$ into contributions from $\{f_{k,G}\le\rho\}$ and its complement, controls the posterior mean on the former using the Gaussian tail, and covers the support of $G$ with balls of radius $4/\{\sqrt{n_k}L(\rho)\}$. The heteroskedastic covariance enters only through the uniform eigenvalue bounds $\us,\bs$, so the calculation carries over with the same constants.
\end{proof}

\begin{lemma}\label{kfo}
    For every pair of probability measures $G$ and $G_0$ on $\R^d$, $1 \leq i \leq d$, $j\ge1,$ and $k=1,...,K$, we have
    \begin{equation}\label{kfo.eq}
        \begin{split}
        \int_{\R^d} &\left\{\partial_i^j\lrp{\t f_{k,G,i}(\mb x) - \t f_{k,G_0,i}(\mb x)}\right\}^2 {\rm d}\mb x \\
        &\le\frac{\bs}{(2\pi\us)^{d/2}}\inf_{a\ge\sqrt{2j-1}}\brcs{\sqrt{2\pi\us}a^{2j-1}e^{-a^2\us}+a^{2j}h^2\lrp{\t f_{k,G,i},\t f_{k,G_0,i}}}.
        \end{split}
    \end{equation}
    Therefore, for a small enough Hellinger accuracy we have
    \begin{align*}
      \nrm{\mb q_{k,G}-\mb q_{k,G_0}}{\c L_2}^2\le_{d,\us,\bs}\sum_{i=1}^dh^2\lrp{\t f_{k,G,i},\t f_{k,G_0,i}}\vrt{\log \sum_{i=1}^dh^2\lrp{\t f_{k,G,i},\t f_{k,G_0,i}}}
    \end{align*}
    for $\nrm{\b f}{{\c L}_2}^2:=\int\nrm{\b f(\mb x)}{}^2{\r d}\mb x$ $\forall \b f:\R^d\to\R^d$.
\end{lemma}
The second conclusion can be derived for any non-negative, small enough $h(x)$ such that $1/h(x)\ge e^{2j-1}$ and $a=\max\lrp{\sqrt{2j-1},\sqrt{-\log h(x)}}$:
\[a^{2j-1}e^{-a^2}+a^{2j}h(x)\le 2a^{2j}h(x)\le2h(x)\log^j1/h(x).\]
\begin{proof}
    The Plancherel-based approach adapts \citet[Lemma F.2]{saha2020nonparametric} to the variance-weighted densities $\t f_{k,G,i}$. The key modification is that Fourier transforms now act on $\sigma_{k,i}(\b\tht)^2$-weighted kernels, which introduces an additional factor of $\bs$ in the tail bound; the algebraic structure of the iteration is otherwise unchanged.
    It suffices to prove the inequality for the $\t f_{i,\blt}$ functions, as the score components follow by differentiation under the integral sign.
    Fix $a \geq \sqrt{2j-1}$ and assume, without loss of generality,
    that $i = 1$. Let 
    \begin{equation*}
    \bar f_{k,G,i'}(u)  := \int e^{i u x_1} \t f_{k,G,i'}(\mb x) {\rm d}x_1  
    \end{equation*}
    denote the Fourier transform of $\t f_{k,G,i}$ treated as a univariate function. Plancherel's identity gives
    \begin{equation}\label{wi1}
        \begin{split}
        2 \pi&\int \left\{\partial_1^j \lrp{\t f_{k,G,i}(\mb x) - \t f_{k,G_0,i}(\mb x)} \right\}^2 {\rm d}x_1 =
        \int u^{2j} \left|\bar f_{k,G,i}(u)  - \bar f_{k,G_0,i}(u) \right|^2 {\rm d}u  \\
        &\leq a^{2j} \int \left|\bar f_{k,G,i}(u)  - \bar f_{k,G_0,i}(u) \right|^2 {\rm d}u + \int_{|u| > a} u^{2j} \left|\bar f_{k,G,i}(u)  - \bar f_{k,G_0,i}(u) \right|^2 {\rm d}u  \\
        &\qquad= 2\pi a^{2j} \int \left( \t f_{k,G,i}(\mb x) - \t f_{k,G_0,i}(\mb x) \right)^2 {\rm d}x_1 + \int_{|u| > a} u^{2j} \left|\bar f_{k,G,i}(u)  - \bar f_{k,G_0,i}(u) \right|^2 {\rm d}u 
        \end{split}
    \end{equation}
    for every $a > 0$. Also,
    \begin{align*}
        \bar f_{k,G,i}(u) &= \iint e^{i u x_1}\sigma_{k,i}(\b\tht)^2|2\pi\Sigma_k(\b\tht)|^{-1/2}\exp\brcs{-\sum_{j=1}^d\frac{(x_j-\sqrt{n_k}\tht_j)^2}{2\sigma_{k,j}(\b\tht)^2}} {\rm d}x_1{\rm d}G(\b\tht) \\
        &\le \int\mb xs{\frac{\bs\sqrt{2\pi\sigma_{k,1}(\b\tht)^2}}{|2\pi\Sigma_k(\b\tht)|^{1/2}}e^{-\sum_{j\ne1}\frac{(x_j-\sqrt{n_k}\tht_j)^2}{2\sigma_{k,j}(\b\tht)^2}}\int\frac{e^{i u x_1}}{\sqrt{2\pi\sigma_{k,1}(\b\tht)^2}}e^{-\frac{(x_1-\sqrt{n_k}\tht_1)^2}{2\sigma_{k,1}(\b\tht)^2}} {\rm d}x_1}{\rm d}G(\b\tht) \\
        &= \int\frac{\bs(2\pi)^{-(d-1)/2}}{\prod_{j\ne1}\sigma_{k,j}(\b\tht)}\exp\lrp{iu\sqrt{n_k}\tht_1-\frac{u^2\sigma_{k,1}(\b\tht)^2}{2}} \exp\brcs{-\sum_{j \neq 1}\frac{(x_j-\sqrt{n_k}\tht_j)^2}{2\sigma_{k,j}(\b\tht)^2}} {\rm d}G(\b\tht)
    \end{align*}
    so that 
    \begin{equation*}
        \left|\bar f_{k,G,i}(u) \right| \leq e^{-u^2\us/2}\int\frac{(2\pi)^{-(d-1)/2}}{\prod_{j\ne1}\sigma_{k,j}(\b\tht)}\exp\brcs{-\sum_{j \neq 1}\frac{(x_j-\sqrt{n_k}\tht_j)^2}{2\sigma_{k,j}(\b\tht)^2}} {\rm d}G(\b\tht). 
    \end{equation*}
    Similarly, we have
    \begin{align*}
        \bar f_{k,G}(u) &= \iint e^{i u x_1}|2\pi\Sigma_k(\b\tht)|^{-1/2}\exp\brcs{-\sum_{j=1}^d\frac{(x_j-\sqrt{n_k}\tht_j)^2}{2\sigma_{k,j}(\b\tht)^2}} {\rm d}x_1{\rm d}G(\b\tht) \\
        &\le \int\mb xs{\frac{\sqrt{2\pi\sigma_{k,1}(\b\tht)^2}}{|2\pi\Sigma_k(\b\tht)|^{1/2}}e^{-\sum_{j\ne1}\frac{(x_j-\sqrt{n_k}\tht_j)^2}{2\sigma_{k,j}(\b\tht)^2}}\int\frac{e^{i u x_1}}{\sqrt{2\pi\sigma_{k,1}(\b\tht)^2}}e^{-\frac{(x_1-\sqrt{n_k}\tht_1)^2}{2\sigma_{k,1}(\b\tht)^2}} {\rm d}x_1}{\rm d}G(\b\tht) \\
        &= \int\frac{(2\pi)^{-(d-1)/2}}{\prod_{j\ne1}\sigma_{k,j}(\b\tht)}\exp\lrp{iu\sqrt{n_k}\tht_1-\frac{u^2\sigma_{k,1}(\b\tht)^2}{2}} \exp\brcs{-\sum_{j \neq 1}\frac{(x_j-\sqrt{n_k}\tht_j)^2}{2\sigma_{k,j}(\b\tht)^2}} {\rm d}G(\b\tht).
    \end{align*}
    So, the second term in \eqref{wi1} can be bounded from above as 
    \begin{equation*}
        \begin{split}
        & \int_{|u| > a} u^{2j} \left|\bar f_{k,G,i}(u)  - \bar f_{k,G_0,i}(u) \right|^2 {\rm d}u \\ 
        &\le 2\int\frac{\bs(2\pi)^{-(d-1)}}{\prod_{j\ne1}\sigma_{k,j}(\b\tht)^2}\exp\brcs{-\sum_{j \neq 1}\frac{(x_j-\sqrt{n_k}\tht_j)^2}{\sigma_{k,j}(\b\tht)^2}} {\rm d}(G+G_0)(\b\tht) \int_{|u| > a} u^{2j} e^{-u^2\us} {\rm d}u.
        \end{split}
    \end{equation*}
    Noting that
    \begin{align*}
        \iint&\frac{2\bs(2\pi)^{-(d-1)}}{\prod_{j\ne1}\sigma_{k,j}(\b\tht)^2}\exp\brcs{-\sum_{j \neq 1}\frac{(x_j-\sqrt{n_k}\tht_j)^2}{\sigma_{k,j}(\b\tht)^2}} {\rm d}\mb x_{-1}{\rm d}(G+G_0)(\b\tht)\\
        &=\int\frac{2\bs(2\pi)^{(1-d)/2}}{\prod_{j\ne1}\sigma_{k,j}(\b\tht)}{\rm d}(G+G_0)(\b\tht)\le\frac{4\bs(2\pi)^{(1-d)/2}}{\us^{(d-1)/2}},
    \end{align*}
    we deduce that
    \begin{align*}
        \int \left\{\partial_1^j \lrp{\t f_{k,G,i}(\mb x) - \t f_{k,G_0,i}(\mb x)} \right\}^2 {\rm d}\mb x &\le 2\pi a^{2j} \int\lrp{\t f_{k,G,i} - \t f_{k,G_0,i}}^2 \\  
        &\qquad + 4\bs(2\pi/\us)^{(d-1)/2}\int_{|u| > a} u^{2j} e^{-u^2\us} {\rm d}u.
    \end{align*}
    Note that
    \begin{equation*}
        \int_{|u| > a} u^{2j} e^{-u^2\us} {\rm d}u \leq a^{2j-1} e^{-a^2\us}
    \end{equation*}
    for $a \geq \sqrt{2j-1}$ by integration by parts and deduction. Also, by $(a-b)^2=(\sqrt{a}-\sqrt{b})^2(\sqrt{a}+\sqrt{b})^2$,
    \begin{equation*}
        \begin{split}
        \int\lrp{\t f_{k,G,i} - \t f_{k,G_0,i}}^2 
        \le\frac{\bs}{(2\pi\us)^{d/2}}h^2\lrp{\t f_{k,G,i}, \t f_{k,G_0,i}}
        \end{split}
    \end{equation*}
    as $\t f_{k,G}$ is bounded above by $\bs\lrp{2\pi\us}^{-d/2}$ for any $G$.
\end{proof}

\begin{theorem}\label{fcc}
    For any measures $G$ and $G_0$ on $\R^d$ and $0 < \rho \le \sqrt{\frac{1}{(2\pi\bs)^de}}$, the quantity
    \begin{align}\label{eq:fcc1}
        \E_{\tx_k\sim f_\GoSk}\lrnrm{\frac{\mb q_\GoSk(\tx_k)}{f_\GoSk(\tx_k)\vee\rho}-\frac{\mb q_\GSk(\tx_k)}{f_\GSk(\tx_k)\vee\rho}}{}^2
    \end{align}
    is bounded above by
    \begin{align}\label{eq:fcc2}
        \sum_{i=1}^d\max\left\{\lrp{\log\frac{(2\pi\us)^{-d}}{\rho^2}}^{3},\vrt{\log h^2\lrp{\t f_{k,G,i},\t f_{k,G_0,i}}}\right\}h^2\lrp{\t f_{k,G,i}, \t f_{k,G_0,i}} \nonumber\\
        + h^2(f_{k,G},f_\GoSk)\log\frac{(2\pi\us)^{-d}}{\rho^2}
    \end{align}
    up to the multiplication by a constant that depends on $\us,\bs$ only.
\end{theorem}
\begin{proof}
    For functions $\mb u\,;\,\R^d \rightarrow \R^d$, we let
    \begin{equation*}
        \|\mb u \|_{0,k} := \left(\int \nrm{\mb u(\mb x)}{}^2 f_\GoSk(\mb x) {\r d}\mb x
        \right)^{1/2}.
    \end{equation*}
    Set
    \[m_k:=f_\GSk+ f_\GoSk,~~m_{\rho,k}:=f_\GSk \vee \rho + f_\GoSk \vee \rho\] 
    so that,
    \begin{align*}
        &\lrnrm{\frac{\mb q_\GoSk}{f_\GoSk\vee\rho}-\frac{\mb q_\GSk}{f_\GSk\vee\rho}}{0,k} \\
        &=\lrnrm{\frac{\mb q_\GoSk}{f_\GoSk\vee\rho}+\frac{-2\mb q_\GoSk+2\mb q_\GoSk-2\mb q_\GSk+2\mb q_\GSk}{m_{\rho,k}}-\frac{\mb q_\GSk}{f_\GSk\vee\rho}}{0,k} \\
        &\le 2 \max_{H \in \{G, G_0\}} \lrnrm{\frac{\mb q_{k,H} |f_\GSk \vee \rho
        - f_\GoSk \vee \rho|}{(f_{k,H} \vee \rho)m_{\rho,k}}}{0,k} 
        + 2 \lrnrm{\frac{\mb q_\GSk-\mb q_\GoSk}{m_{\rho,k}}}{0,k}
    \end{align*}
    which is the standard $T_1$/$T_2$ decomposition of \citet[Theorem E.1]{saha2020nonparametric}, applied here to the variance-weighted score $\mb q_{k,G}$.
    Write the RHS above by $T_1$ and $T_2$ respectively so that we have
    \[\eqref{eq:fcc1}\lesssim \E(T_1^2)+\E(T_2^2).\]
    We shall now bound $T_1$ and $T_2$ separately.
    \paragraph{\textbf{Bound on $T_1^2$.}}
    For $T_1$, we
    use inequality 
    \eqref{p1.neq} in Lemma \ref{p1} (note that $\rho \le \sqrt{\frac{1}{(2\pi\bs)^de}}$). That is,
    \begin{align*}
        \E(T_1^2) &= \max_{H \in \{G, G_0\}}\int \frac{\nrm{\mb q_{k,H}}{}^2(f_\GSk \vee \rho - f_\GoSk \vee \rho)^2}{\lrp{f_{k,H} \vee \rho}^2m_{\rho,k}^2} f_\GoSk \\
        &\le \bs\log\frac{1}{(2\pi\us)^d\rho^2} \int
        \frac{\left(f_\GSk - f_\GoSk \right)^2}{m_{\rho,k}^2} f_\GoSk \\
        &\qquad= \bs\log\frac{1}{(2\pi\us)^d\rho^2} \int \left(\sqrt{f_\GSk} - \sqrt{f_\GoSk} \right)^2  \frac{\left(\sqrt{f_\GSk} + \sqrt{f_\GoSk} \right)^2}{m_{\rho,k}^2} f_\GoSk \\
        &\le \bs\log\frac{1}{(2\pi\us)^d\rho^2} \int \left(\sqrt{f_\GSk} - \sqrt{f_\GoSk} \right)^2  \frac{2m_kf_\GoSk}{m_{\rho,k}^2}  \\
        &\le \bs\log\frac{1}{(2\pi\us)^d\rho^2} \int
        \left(\sqrt{f_\GSk} - \sqrt{f_\GoSk} \right)^2 = \bs\log\frac{1}{(2\pi\us)^d\rho^2} \cdot h^2\lrp{f_\GSk,f_\GoSk}.
    \end{align*}

    \paragraph{\textbf{Bound on $T_2^2$.}}
    Now, recalling that
    $\mb q_{k,G}(\cdot)=\begin{bmatrix}
    \partial_1\t f_{k,G}(\cdot)\\\vdots\\\partial_d\t f_{k,G}(\cdot)
    \end{bmatrix}$, start by writing 
    \begin{equation} \label{t31}
        \begin{split}
        \E(T_2^2) = 
        \int\frac{\nrm{\mb q_\GSk-\mb q_\GoSk}{}^2}{m_{\rho,k}^2}f_\GoSk
        \le&\int \frac{\nrm{{\mb q_\GSk-\mb q_\GoSk}{}}{}^2}{m_{\rho,k}} \\
        \quad&=\sum_{i=1}^d\int\frac{\brcs{\partial_i\lrp{\t f_{k,G,i}-\t f_{k,G_0,i}}}^2}{m_{\rho,k}} =: \sum_{i=1}^d D_{i,1}^2 
        \end{split}
    \end{equation}
    where, for $1 \leq i \leq d$ and $j \geq 0$, 
    \begin{align*}
        D_{i,j}^2 := \int \frac{\brcs{\partial_i^j\lrp{\t f_{k,G,i}-\t f_{k,G_0,i}}}^2}{m_{\rho,k}} \text{~~with~~}
        \partial_i^j := \frac{\partial^j}{\partial (\cdot)_i^j}.
    \end{align*}
    We bound $D_{i, 1}^2$ via a recursive inequality relating consecutive orders $D_{i,j-1}, D_{i,j}, D_{i,j+1}$ (cf.\ the iteration in \citet[Theorem E.1]{saha2020nonparametric}, applied here to $\t f_{k,G,i}$).
    First, we acknowledge that
    \begin{align}
        D_{i,0}^2 = \int \frac{\lrp{\t f_{k,G,i}-\t f_{k,G_0,i}}^2}{m_{\rho,k}}
        &\leq 2 \int \left(\sqrt{\t f_{k,G,i}} - \sqrt{\t f_{k,G_0,i}}
        \right)^2 \frac{\t f_{k,G,i} + \t f_{k,G_0,i}}{m_{\rho,k}}\nonumber\\
        &\quad\le 2 \int \left(\sqrt{\t f_{k,G,i}} - \sqrt{\t f_{k,G_0,i}}
        \right)^2\frac{\bs m_k}{m_{\rho,k}}\nonumber\\
        &\quad\leq 2\bs h^2\lrp{\t f_{k,G,i}, \t f_{k,G_0,i}}\label{dio}.
    \end{align}
    
    Integration by parts gives, for $j \geq 1$,  
    \begin{align}
        D_{i,j}^2 &= - \int \left[ \partial_i^{j-1} \left(\t f_{k,G,i} - \t f_{k,G_0,i}
        \right) \right]  \left[\partial_i^j \left(\t f_{k,G,i} - \t f_{k,G_0,i} \right)
        \right] \partial_i \left( \frac{1}{m_{\rho,k}} \right) \nonumber \\
        & - \int \frac{ \left[\partial_i^{j-1} \left(\t f_{k,G,i} - \t f_{k,G_0,i}
        \right) \right] \left[\partial_i^{j+1} \left(\t f_{k,G,i} - \t f_{k,G_0,i} \right)
        \right]}{m_{\rho,k}}. \label{ffr}
    \end{align}
    Almost surely, we have
    \begin{align*}
        \left|\partial_i \left(\frac{1}{m_{\rho,k}}
        \right) \right|\le\frac{|\partial_if_{k,G}|+|\partial_if_\GoSk|}{m_{\rho,k}^2}
        \le\frac{|\partial_if_{k,G}|/f_{k,G}\vee\rho+|\partial_if_\GoSk|/f_\GoSk\vee\rho}{m_{\rho,k}} \\
        \le\frac{\nrm{\nabla f_{k,G}}{}/f_{k,G}\vee\rho+\nrm{\nabla f_\GoSk}{}/f_\GoSk\vee\rho}{m_{\rho,k}} \le \frac{2}{m_{\rho, k}}\sqrt{\bs\log\frac{1}{(2\pi\us)^d\rho^2}}
    \end{align*}
    by applying \eqref{p1.neq} in Lemma \ref{p1}. 
    Therefore,
    \begin{align*}
        D_{i,j}^2 &\le 2\sqrt{\bs\log \frac{1}{(2\pi\us)^d\rho^2}} \int \frac{\left| \partial_i^{j-1} \left(\t f_{k,G,i} - \t f_{k,G_0,i}
        \right) \right|  \left|\partial_i^j \left(\t f_{k,G,i} - \t f_{k,G_0,i} \right)
        \right|}{m_{\rho,k}}  \\ 
        &\qquad+ \int \frac{ \left|\partial_i^{j-1} \left(\t f_{k,G,i} - \t f_{k,G_0,i}
        \right) \right| \left|\partial_i^{j+1} \left(\t f_{k,G,i} - \t f_{k,G_0,i} \right)
        \right|}{m_{\rho,k}}. 
    \end{align*}
    Cauchy-Schwarz inequality gives us
    \begin{align*}
        D_{i,j}^2 &\le 2\sqrt{\bs\log \frac{1}{(2\pi\us)^d\rho^2}} \sqrt{\int \frac{\brcs{ \partial_i^{j-1}\lrp{\t f_{k,G,i} - \t f_{k,G_0,i}}}^2}{m_{\rho,k}}} \sqrt{\int
        \frac{\brcs{ \partial_i^j\lrp{\t f_{k,G,i} - \t f_{k,G_0,i}}}^2}{m_{\rho,k}}} \\
        &+ \sqrt{\int \frac{\brcs{ \partial_i^{j-1}\lrp{\t f_{k,G,i} - \t f_{k,G_0,i}}}^2}{m_{\rho,k}}}\sqrt{\int \frac{\brcs{ \partial_i^{j+1}\lrp{\t f_{k,G,i} - \t f_{k,G_0,i}}}^2}{m_{\rho,k}}},
    \end{align*}
    i.e.,
    \begin{equation}\label{rr}
        \frac{D_{i,j}}{D_{i, j-1}} \leq \Upsilon +
        \frac{D_{i, j+1}}{D_{i,j}},\quad j\ge1
    \end{equation}
    where $\Upsilon:=2\sqrt{\bs\log \frac{1}{(2\pi\us)^d\rho^2}}$.
    Fix an integer $j_0 \ge 1$ and a real number $\beta > 0$. 
    We can use the same arguments as in \cite{saha2020nonparametric} to consider cases where $j\le j_0$ satisfies $D_{i,j+1}\le\beta D_{i,j}$ and $D_{i,j+1}>\beta D_{i,j}$. We record them succinctly as follows.
    \begin{enumerate}
        \item[Case 1.] Suppose an integer $1 \leq j \leq j_0$ satisfies $D_{i,j+1} \leq \beta D_{i,j}$. Then applying
        \eqref{rr} recursively for $1, \dots, j$ and 
        \eqref{dio}, we have
        \begin{equation}\label{bb1}
            \begin{split}
            D_{i,1} \le \sqrt{2\bs}\left(j_0 \Upsilon + \beta \right) h\lrp{\t f_{k,G,i}, \t f_{k,G_0,i}}.
            \end{split}
        \end{equation}
        \item[Case 2.] Now suppose that $D_{i,j+1} > \beta D_{i,j}$ for every integer $1 \le j \leq j_0$. We can deduce from \eqref{rr} and by recursive evaluation again:
        \begin{equation*}
            \frac{D_{i,1}}{D_{i,0}} \leq \left(1 +
            \frac{\Upsilon}{\beta} \right)^j \frac{D_{i,j+1}}{D_{i,j}},\quad j\le j_0.
        \end{equation*}
        Geometric averaging for $j = 0, 1, \dots, j_0$ then gives
        \begin{equation*}
            D_{i,1} \leq \left(1 +
            \frac{\Upsilon}{\beta} \right)^{j_0/2}{D_{i,
            j_0+1}}^{\frac{1}{j_0+1}} {D_{i,0}}^{\frac{j_0}{j_0 + 1}}.
        \end{equation*}
        From Lemma \ref{kfo}, we deduce that
        \begin{align}\label{dik}
            D_{i,j}^2&\le\frac{1}{\rho}\int\brcs{\partial_i^j\lrp{\t f_{k,G,i}-\t f_{k,G_0,i}}}^2\nonumber\\
            &\qquad\le_d
            \frac{\bs}{\rho(2\pi\us)^{d/2}}\brcs{\sqrt{2\pi\us}a^{2j-1}e^{-a^2\us}+a^{2j}h^2\lrp{\t f_{k,G,i},\t f_{k,G_0,i}}}
        \end{align}
        for $a\ge\sqrt{2j-1}$.
        Now using \eqref{dio} and the bound \eqref{dik} (with $j = j_0 + 1$),
        we obtain 
        \begin{equation}\label{bb2}
            \begin{split}
            D_{i,1}\le\left(1 +
            \frac{\Upsilon}{\beta} \right)^{\frac{j_0}{2}}&\bxs{\frac{\bs}{\rho(2\pi\us)^{d/2}}\brcs{\sqrt{2\pi\us}a^{2j_0+1}e^{-a^2\us}+a^{2(j_0+1)}h^2\lrp{\t f_{k,G,i},\t f_{k,G_0,i}}}}^{\frac{1/2}{j_0 + 1}} \\
            &\times\brcs{2\bs h^2\lrp{\t f_{k,G,i}, \t f_{k,G_0,i}}}^{\frac{j_0/2}{j_0 + 1}}
            \end{split}
        \end{equation}
        for every $a \geq \sqrt{2j_0 + 1}$.
    \end{enumerate}
    The bound for $D_{i,1}$ is the maximum of cases 1 and 2.
    $\beta$ will be chosen as $\beta = j_0 \Upsilon$ so that
    \[\eqref{bb1}\le3\sqrt{\bs}j_0\Upsilon h\lrp{\t f_{k,G,i}, \t f_{k,G_0,i}}\]
    and 
    \begin{align*}
        \eqref{bb2}\le\sqrt{e}\bxs{\frac{a^{2j_0+1}}{\rho/\bs}\lrp{\frac{\tau}{2\pi}}^{d/2}\brcs{ah^2\lrp{\t f_{k,G,i},\t f_{k,G_0,i}}+\sqrt{2\pi\us}e^{-a^2\us}}}^{\frac{1/2}{j_0 + 1}}
        \brcs{\bs h^2\lrp{\t f_{k,G,i}, \t f_{k,G_0,i}}}^{\frac{j_0/2}{j_0 + 1}}.
    \end{align*}
    Noting that
    \begin{equation*}
        \left( \frac{1}{(2\pi\us)^{d/2}\rho} \right)^{\frac{1/2}{j_0+1}} = \exp\brcs{\Upsilon^2/(2\bs)}^{\frac{1/2}{j_0+1}},
    \end{equation*}
    take $j_0:=\min\brcs{j\in\bb N\,;\,j+1\ge\Upsilon^2/(2\bs)}$.
    Then, the above term is bounded by $\sqrt{e}$.
    Finally, take $a$ as
    \begin{equation*}
        a := \max\brcs{\sqrt{2j_0 + 1}, \sqrt{\frac{-1}{\us}\log h^2\lrp{\t f_{k,G,i},\t f_{k,G_0,i}}},\sqrt{2\pi\us}}
    \end{equation*}
    so that $e^{-a^2\us} \leq h^2\lrp{\t f_{k,G,i},\t f_{k,G_0,i}}$ hence \eqref{bb2} can then be further bounded by
    \begin{align*}
        &ea^{\frac{j_0+1/2}{j_0+1}}\brcs{a\bs h^2\lrp{\t f_{k,G,i},\t f_{k,G_0,i}}+\bs\sqrt{2\pi\us}h^2\lrp{\t f_{k,G,i},\t f_{k,G_0,i}}}^{\frac{1/2}{j_0+1}}\brcs{\bs h^2\lrp{\t f_{k,G,i}, \t f_{k,G_0,i}}}^{\frac{j_0/2}{j_0 + 1}} \\ 
        &\le ea^{\frac{j_0+1/2}{j_0+1}}\lrp{a+\sqrt{2\pi\us}}^{\frac{1/2}{j_0+1}}{\sqrt{\bs}h\lrp{\t f_{k,G,i},\t f_{k,G_0,i}}}
        \le ea\sqrt{\bs}h\lrp{\t f_{k,G,i},\t f_{k,G_0,i}}.
    \end{align*}
    Because $a\ge1$ and $D_{i,1}$ is bounded by $\max$(\eqref{bb1},\eqref{bb2}), we get
    \begin{align*}
        D_{i,1} &\lesssim \sqrt{\bs}\max\lrp{j_0 \Upsilon,a}h\lrp{\t f_{k,G,i},\t f_{k,G_0,i}}.
    \end{align*}
    Our choice of $j_0$ also satisfies $j_0 \le \frac{\Upsilon^2}{2\bs}$ so $j_0\Upsilon\le\Upsilon^3/\bs=\frac{8}{\bs}\lrp{\bs\log\frac{1}{(2\pi\us)^d\rho^2}}^{3/2}$ gives 
    \begin{equation*}
      D_{i,1} \le_{\bs,\us} \max\left\{\lrp{\log\frac{1}{(2\pi\us)^d\rho^2}}^{\frac{3}{2}},\sqrt{\frac{1}{\us}\vrt{\log h^2\lrp{\t f_{k,G,i},\t f_{k,G_0,i}}}}\right\}h\lrp{\t f_{k,G,i}, \t f_{k,G_0,i}}.
    \end{equation*}
    Combining with
    \eqref{t31}, we deduce that 
    \begin{equation*}
        E(T_2^2) \le_{d,\bs,\us} \sum_{i=1}^d\max\left\{\log^3\frac{(2\pi\us)^{-d}}{\rho^2},\vrt{\log h^2\lrp{\t f_{k,G,i},\t f_{k,G_0,i}}}\right\}h^2\lrp{\t f_{k,G,i}, \t f_{k,G_0,i}}.
    \end{equation*}
    Combine the above finding with the bound for $T_1^2$ to finish the proof.
\end{proof}
The following lemma extends the tail moment bound of \citet[Lemma 15]{soloff2025multivariate} to our heteroskedastic kernels.
\begin{lemma}\label{lem-lip}
    For any $\lambda\in(0,1\wedge q]$ and $M>1$, 
    \[\E\left[\mathfrak{d}_S(\tx_k)^{\lambda} 1_{\mathfrak{d}_S(\tx_k)\ge M}\right] \le 2C_dM^{d+\lambda-2}\bs^{1-d/2}e^{-M^2/(8\bs)} + M^\lambda \left(\frac{2\mu_{k,q}(S,G_0)}{M}\right)^q\]
    where $\mu_{k,q}(S,G_0):= \E_{\b\tht\sim G_0}\left[\mathfrak{d}_S(\sqrt{n_k}\b\tht)^q\right]^{1/q}$.
\end{lemma}
\begin{proof} 
    Note that $\tx_k=\sqrt{n_k}\b\tht_k+\Sigma_k(\b\tht_k)^{1/2}\mb z_k$ where $\mb z_k\sim{\c N}_d(0,\mb I_d)$. The proof adapts \citet[Lemma 15]{soloff2025multivariate}; the only modification is that $\Sigma_k(\b\tht_k)^{1/2}$ replaces the fixed scale matrix, which is absorbed by the uniform bound $\bs$.
    Since distance $\mathfrak{d}_S$ is $1$-Lipschitz,
    \begin{align*}
    \E&\left[\mathfrak{d}_S(\tx_k)^{\lambda} 1_{\mathfrak{d}_S(\tx_k)\ge M}\right]\\
    &\le \E\left[(2\|\Sigma_k(\b\tht_k)^{1/2}\mb z_k\|)^{\lambda} 1_{2\|\Sigma_k(\b\tht_k)^{1/2}\mb z_k\|\ge M}\right] + \E\left[(2\mathfrak{d}_S(\sqrt{n_k}\b\tht_k))^{\lambda} 1_{2\mathfrak{d}_S(\sqrt{n_k}\b\tht_k)\ge M}\right].
    \end{align*}
    The first term is bounded above by
    \begin{align*}
        2M^{\lambda-1}\E&\left[\lrnrm{\Sigma_k(\b\tht_k)^{1/2}\mb z_k}{} 1_{\|\Sigma_k(\b\tht_k)^{1/2}\mb z_k\|_\ge M/2}\right] \\
        &\le 2M^{\lambda-1}\sqrt{\bs}\E\bxs{\|\mb z_k\| 1_{\brcs{\|\mb z_k\| \ge \frac{M}{2\bs^{1/2}}}}}
        \le 2C_dM^{\lambda-1}\sqrt{\bs}\lrp{\frac{M}{\bs^{1/2}}}^{d-1}e^{-M^2/(8\bs)}.
    \end{align*}
    Since $\lambda < q$, applying H\"older yields
    \[
    \E\left[(2\mathfrak{d}_S(\sqrt{n_k}\b\tht_k))^{\lambda} 1_{2\mathfrak{d}_S(\sqrt{n_k}\b\tht_k)\ge M}\right]
    \le M^\lambda \left(\frac{2\mu_{k,q}(S,G_0)}{M}\right)^q 
    \]
    where $\mu_{k,q}(S,G_0)$ denotes the $q$-th moment of $\mathfrak{d}_S(\sqrt{n_k}\b\tht_k)$.
\end{proof}

\section{Control of Metric Entropy}
To bound the metric entropy of the induced marginal classes, we adopt the moment-matching techniques.
\begin{lemma} (Moment matching I) 
    Let $a>0$ and $G,H\in {\c P}(\R^d)$. Suppose $A\subset \R^d$ is such that \[\bb{B}_{a}(\bx_k)\subseteq A\subseteq\bb{B}_{ca}(\bx_k),\quad c\ge 1\] 
    for all $k\le K$ and
    \begin{align}\label{eq:mmatch}
        \int_A \frac{\prod_{l=1}^d\tht_l^{\alpha_l}}{\prod_{j=1}^d\sigma_{k,j}(\b\tht)^{2\beta_j+1}}\mathrm{d}(G-H)(\b\tht) = 0, \quad k=1,\dots,K
    \end{align}
    for all multi-indices $\alpha, \beta \in \bb N_{\ge0}^d$ such that $\nrm{\alpha}{1}\le2m+1, \nrm{\beta}{1} \le m$.
    Then,
    \begin{align*}
        \max_{1\le k\le K} \left|f_\GSk(\tx_k) - f_\HSk(\tx_k)\right|
        &\le \frac{2}{(2\pi\us)^{d/2}}\brcs{\left(\frac{ec^2\bn a^2}{2\us(m+1)}\right)^{m+1}  + e^{-\un a^2/(2\bs)}}
    \end{align*}
    and
    \[\max_{1\le k\le K}\lrnrm{\mb q_{k,G}(\tx_k)-\mb q_{k,H}(\tx_k)}{}
    \le \frac{2\sqrt{\bn}a}{(2\pi\us)^{d/2}} \brcs{c\left( \frac{ec^2\bn a^2}{2\us(m+1)} \right)^{m+1} + \sqrt{\frac{\bar{s}}{e\bn a^2}} + e^{-\un a^2/(2\bar{s})}}.\]
\end{lemma}
\begin{proof} 
    For each $k\in [K]$, we have
    \begin{align*}
        f_\GSk(\tx_k) - f_\HSk(\tx_k) 
        = \int_{A\cup A^c}\vrp^{(k)}(\tx_k;\b\tht) {\r d}(G - H)(\b\tht).
    \end{align*}
    When $\b\tht\in A^c$, $\|\bx_k-\b\tht\|\ge a$ hence $\nrm{\tx_k-\sqrt{n_k}\b\tht_k}{}\ge\sqrt{n_k}a$, so 
    \[\int_{A^c}\vrp^{(k)}(\tx_k;\b\tht){\r d}(G-H)(\b\tht) \le \frac{2e^{-n_k a^2/(2\bs)}}{(2\pi \us)^{d/2}}.\]
    Next, write $\mb z_{k,\b\tht}:=\Sigma_k(\b\tht)^{-1/2}\lrp{\bx-\b\tht}$.
    The remainder term of the polynomial expansion of
    \[|2\pi\Sigma_k(\b\tht)|^{1/2}\vrp^{(k)}(\tx_k;\b\tht)=\exp\lrp{-n_k\|\mb z_{k,\b\tht}\|^2/2}\]
    with degree $m$, $P(\mb z_{k,\b\tht})=\sum_{j\le m}\frac{(-n_k\|\mb z_{k,\b\tht}\|^2)^j}{2^jj!}$, satisfies \citep[Lemma D.2]{saha2020nonparametric}:
    \[|R(\mb z_{k,\b\tht})|\le\frac{(n_k\|\mb z_{k,\b\tht}\|^2/2)^{m+1}}{(m+1)!}\le\brcs{\frac{en_k\|\mb z_{k,\b\tht}\|^2}{2(m+1)}}^{m+1}\le \brcs{\frac{en_k\|\bx-\b\tht\|^2}{2\us(m+1)}}^{m+1}.\]
    Now,
    \[\vrp^{(k)}(\tx_k;\b\tht)=|2\pi\Sigma_k(\b\tht)|^{-1/2}\brcs{P(\mb z_{k,\b\tht})+R(\mb z_{k,\b\tht})}\]
    and since $\int_A|\Sigma_k(\b\tht)|^{-1/2}P\lrp{\mb z_{k,\b\tht}} {\r d}(G-H)(\b\tht)=\int_A|\Sigma_k(\b\tht)|^{-1/2}\tht_iP\lrp{\mb z_{k,\b\tht}} {\r d}(G-H)(\b\tht)=0$ for $i=1,...,d$,
    \begin{align*}
        \left|\int_A\vrp^{(k)}(\tx_k;\b\tht) {\r d}(G-H)(\b\tht)\right|
        &\le \left|\int_A|2\pi\Sigma_k(\b\tht)|^{-1/2}R(\mb z_{k,\b\tht}) {\r d}(G-H)(\b\tht)\right| \\
        &\le \frac{2}{(2\pi\us)^{d/2}} \brcs{\frac{en_k c^2a^2}{2\us(m+1)}}^{m+1}.
    \end{align*}
    
    Next, recall that $\mb q_{k,G}(\tx_k)=\int(\sqrt{n_k}\b\tht-\tx_k)\vrp^{(k)}(\tx_k;\b\tht){\r d}G(\b\tht)$. 
    We have
    \[\int_{A^c}\nrm{\sqrt{n_k}\b\tht-\tx_k}{}\vrp^{(k)}(\tx_k;\b\tht)|{\r d}(G-H)(\b\tht)| \le \sup_{u\ge\sqrt{n_k}a}\frac{2ue^{-{u^2}/\lrp{2\bs}}}{(2\pi\us)^{d/2}}.\]
    Case-by-case analysis of the last term depends on whether the starting point $\sqrt{n_k}a$ lies before or after the critical point $\sqrt{\bs}$. In short, we have
    \[\sup_{u\ge\sqrt{n_k}a}\frac{ue^{-{u^2}/\lrp{2\bs}}}{(2\pi\us)^{d/2}}=\frac{\sqrt{\bs/e}}{(2\pi\us)^{d/2}}1_{\brcs{a\le\sqrt{\bs/n_k}}}+\frac{\sqrt{n_k}ae^{-n_k a^2/(2\bs)}}{(2\pi\us)^{d/2}}1_{\brcs{a>\sqrt{\bs/n_k}}}.\]
    Finally,
    \begin{align*}
        &\lrnrm{\int_A|2\pi\Sigma_k(\b\tht)|^{-1/2}(\tx_k-\sqrt{n_k}\b\tht)\vrp^{(k)}(\tx_k;\b\tht){\r d}(G-H)(\b\tht)}{}\\
        &\le\lrnrm{\int_A|2\pi\Sigma_k(\b\tht)|^{-1/2}(\tx_k-\sqrt{n_k}\b\tht)R(\mb z_{k,\b\tht}) {\r d}(G-H)(\b\tht)}{} \\
        &\le\frac{2\sqrt{n_k}ca}{(2\pi\us)^{d/2}}\brcs{\frac{en_k c^2a^2}{2\us(m+1)}}^{m+1}.
    \end{align*}
\end{proof}
\begin{lemma}\label{lem-moment-match-ii} 
    (Moment matching II) For any $G\in{\c P}(\R^d)$, there exists a discrete distribution $H$ supported on $S^{a}$ with at most
    \[
    l:= (\lfloor27\bn a^2/\us\rfloor+2)^d N(a, S^{a}) + 1,~~a>0
    \]
    atoms such that 
    \begin{align*}
        \|f_{\blt,G} - f_{\blt,H}\|_{\infty, S^a}
        \le \frac{4e^{-\un a^2/(2\bs)}}{(2\pi\us)^{d/2}}&,\quad
        \nrm{\mb q_{\blt,G}-\mb q_{\blt,H}}{\infty,S^a}\le\frac{6a}{(2\pi\us)^{d/2}}\brcs{\sqrt{\frac{\bs}{e}}+2\sqrt{\bn}e^{-\un a^2/(2\bs)}}.
    \end{align*}
\end{lemma}
\begin{proof} 
    The covering argument follows \citet[Lemma D.3]{saha2020nonparametric}; the moment-matching condition \eqref{eq:mmatch} now involves the variance functions $\sigma_{k,j}(\b\tht)$, and the remainder bounds must absorb the heteroskedastic kernel via the uniform eigenvalue bounds.
    Let $\mathring{S}^{a} := \cup_{\bx \in S} \bb{B}^o_{a}(\bx)$ be the open $a$-neighborhood of $S$, where $\bb{B}^o_a(\bx)$ denotes an open ball of radius $a$ centered at $\bx$. Let $L := N(a, \mathring{S}^{a})$ be its covering number, which satisfies $L \le N(a, S^{a})$. Given a collection of closed balls $B_1, \dots, B_L$ of radius $a$ that cover $\mathring{S}^{a}$, we can construct a disjoint partition of $\mathring{S}^{a}$ by defining $E_1 := B_1 \cap \mathring{S}^{a}$ and $E_i := \left(B_i \setminus \cup_{j < i} B_j\right) \cap \mathring{S}^{a}$ for $i = 2, \dots, L$. By construction, $\cup_{i=1}^L E_i = \mathring{S}^{a}$.
    
    Suppose that a probability measure $H$ is chosen so that $G$ and $H$ satisfy \eqref{eq:mmatch} on each set $E_i$ for $i = 1, \dots, L$. So, after setting $c=3$ and $m=\lfloor13.5\bn a^2/\us\rfloor$, we obtain
    \begin{align*}
        \max_{1\le k\le K} \left|f_\GSk(\tx_k) - f_\HSk(\tx_k)\right|
        &\le \frac{2}{(2\pi\us)^{d/2}}\brcs{\left(\frac{9e\bn a^2}{2\us(m+1)}\right)^{m+1}  + e^{-\un a^2/(2\bs)}}
    \end{align*}
    and
    \[\max_{1\le k\le K}\lrnrm{\mb q_{k,G}(\tx_k)-\mb q_{k,H}(\tx_k)}{}
    \le \frac{6\sqrt{\bn}a}{(2\pi\us)^{d/2}} \brcs{\left( \frac{9e\bn a^2}{2\us(m+1)} \right)^{m+1} + \sqrt{\frac{\bar{s}}{e\bn a^2}} + e^{-\un a^2/(2\bar{s})}}.\]
    Since $m+1\ge13.5\bn a^2/\us$,
    \[\brcs{\frac{9e\bn a^2}{2\us(m+1)}}^{m+1}\le\lrp{\frac{e}{3}}^{m+1}\le e^{-(m+1)/12}\le\exp\lrp{-\frac{13.5\bn a^2}{12\us}}\le e^{-\bn a^2/(2\us)}.\]
    Then, the result for the first inequality follows.
    This also gives
    \begin{align*}
        \max_{1\le k\le K}\lrnrm{\mb q_{k,G}(\tx_k)-\mb q_{k,H}(\tx_k)}{}
        \le \frac{6\sqrt{\bn}a}{(2\pi\us)^{d/2}}\brcs{e^{-\bn a^2/(2\us)}+\sqrt{\frac{\bs}{e\bn a^2}}+e^{-\un a^2/(2\bs)}} \\
        \le \frac{6a}{(2\pi\us)^{d/2}}\brcs{\sqrt{\frac{\bs}{ea^2}}+2\sqrt{\bn}e^{-\un a^2/(2\bs)}}.
    \end{align*}
\end{proof}
Finally, we control the covering number of $S^a$ under a pseudo metric.
\begin{lemma}\label{lem-entropy} 
    Assume that $\bn/\un$ tends to some positive constant as $\un\to\infty$.
    There exists a positive constant $\un_0>0$ such that for any $\un>\un_0$ and every compact set $S\subset\R^d$, $M > 0$ and $\eta\le\us^{-d/2}\wedge e^{-1},$ we have
    \begin{align}
        \log N(\eta,\bb{F},\|\cdot\|_{\infty,S^a})\vee N(\eta,\bb{F}_q,\|\cdot\|_{\infty,S^a})
        \le_{d} N(a,S^a)\log^{d+1}\frac{c_{d,\us,\bs,\bs'}}{\eta}
    \end{align}
    for $a=\sqrt{\frac{2\bs}{\un}\log\frac{1}{\eta}}$.
\end{lemma}
\begin{proof}
Let $G\in {\c P}(\R^d)$, and obtain a discrete distribution $H$ supported on $S^a$ with at most $l$ as in Lemma \ref{lem-moment-match-ii}.
\paragraph{First inequality.}
We write $\c{C}$ to denote a minimal $\zeta$-net of $S^a$, and let $H'$ be supported on the elements of $\c{C}$ that are closest to atoms of $H$. Writing $H= \sum_{j=1}^l w_j \dlt_{\mb a_j}$ and $H'= \sum_{j=1}^l w_j \dlt_{\mb b_j}$, we have
\begin{align*}
    \|f_{\blt,H} - f_{\blt,H'}\|_{\infty, S^a}
    &= \max_{k\in[K]}\sup_{\mb x\in S^a} |f_{k,H}(\mb x) - f_{k,H'}(\mb x)|  \\
    &\le \max_{k\in[K]}\sup_{\mb x\in S^a} \sum_{j}w_j\vrt{\vrp^{(k)}(\mb x;\mb a_j) - \vrp^{(k)}(\mb x;\mb b_j)}\\
    &\le_d \frac{\sqrt{\bn}+2d^{3/2}\bs'}{\us^{(d+1)/2}}\zeta.
\end{align*}
In the third line, we used Lemma \ref{lem:mvt} and $\nrm{\mb a_j-\mb b_j}{}\le\zeta$.
Let $\c{D}$ denote a minimal $\xi$-net of $\Dlt^{l-1}=\{\mb w \in \R_+^l\,;\,\sum_{j=1}^l w_j=1\}$ in the $\ell_1$ norm, and approximate the weights $w$ by their closest element $v\in \c{D}$. Writing $H'' = \sum_{j=1}^l v_j\dlt_{\mb b_j}$,
\begin{align*}
    \|f_{\blt,H'} - f_{\blt,H''}\|_{\infty, S^a}
    &= \max_{k\in[K]}\sup_{\mb x\in S^a} |f_{k,H'}(\mb x) - f_{k,H''}(\mb x)|  \\
    &\le \max_{k\in[K]}\sup_{\mb x\in S^a} \sum_{j}|w_j - v_j| |\vrp^{(k)}(\mb x;\mb b_j)| \le_d \us^{-d/2} \xi.
\end{align*}
Applying the triangle inequality to the past three displays,
\begin{align*}
    \|f_{\blt,G} - f_{\blt,H''}\|_{\infty, S^a}
    \le_{d,\us,\bs'} e^{-\un a^2/(2\bs)}+(\sqrt{\bn}+1)\zeta+\xi.
\end{align*}
Letting $\xi := \eta$, $\zeta := \frac{\xi}{\sqrt{\bn}+1}$, and selecting $a=\sqrt{\frac{2\bs}{\un}\log\frac{1}{\xi}}$ yields the convergence rate $\|f_{\blt,G} - f_{\blt,H''}\|_{\infty, S^a}\le C_{d,\us,\bs'}\eta$. To take $a$ as such, we let $\eta<\frac{1}{e}$ so that $\log(1/\xi)>1$.

The number of possible $H''$ is 
\[
|\c{C}|\cdot|\c{D}| = N(\xi, \Dlt^{l-1}) \binom{N(\zeta, S^a)}{l}
\le \bxs{\left(1+\frac{2}{\xi}\right)\frac{e N(\zeta, S^a)}{l}}^l
\]
by Stirling's lower bound and standard arguments.
From the previous Lemma,
\[l=\lrp{\left\lfloor\frac{27\bn a^2}{\us}\right\rfloor+2}^d N(a, S^a) + 1>\lrp{\frac{54\tau\bn}{\un}+1}^d N(a, S^a) + 1\ge_d\lrp{\frac{\tau\bn}{\un}}^dN(a, S^a)\]
with $a,\xi,\eta$ selected as above, so using the arguments in \cite{saha2020nonparametric},
\begin{align*}
    \frac{N(\zeta, S^a)}{l}
    \le \brcs{\frac{\un}{\tau\bn}\left(1+\frac{a}{\zeta}\right)}^d
    \le \lrp{\frac{\un}{\bn}+2\sqrt{\frac{\bs\un}{\bn\eta^3}}}^d\le\lrp{1+2\sqrt{\frac{\bs}{\eta^3}}}^d.
\end{align*}
So, as $1/\eta>e$,
\begin{align*}
    \log N(C_{d,\us,\bs'}\eta,\bb{F},&\|\cdot\|_{\infty,S})=\log\bxs{\left(1+\frac{2}{\xi}\right)\frac{e N(\zeta, S^a)}{l}}^l\le l\brcs{\log\frac{3}{C_{d,\us,\bs'}\eta}\frac{eN(\zeta, S^a)}{l}} \\
    &\le_dN(a,S^a)\brcs{\frac{\bn}{\un}\log(1/\eta)}^d\log\brcs{\frac{3}{C_{d,\us,\bs'}\eta}\lrp{1+2\sqrt{\frac{1}{C_{d,\us,\bs'}\eta^3}}}^d} \\
    &\le N(a,S^a)\lrp{\frac{\bn}{\un}}^d\log^d\frac{1}{\eta}\log\frac{c_{d,\us,\bs'}}{\eta^{1+3d/2}} \\
    &\le_dN(a,S^a)\log^{d+1}\lrp{\frac{c_{d,\us,\bs'}}{\eta}}.
\end{align*}

\paragraph{Second inequality.}
We use the same covering argument to prove the second inequality as well.
For $\c{C}$, a minimal $\zeta$-net of $S^a$ and $H'$, the approximate of $H$ with closest element from $\c{C}$, $H= \sum_{j=1}^l w_j \dlt_{\mb a_j}$ and $H'= \sum_{j=1}^l w_j \dlt_{\mb b_j}$, we have
\begin{align*}
    \|&\mb q_{\blt,H} - \mb q_{\blt,H'}\|_{\infty, S^a}
    \le\max_{k\in[K]}\sup_{\mb x\in S^a} \sum_{j}w_j\nrm{(\sqrt{n_k}\mb a_j-\mb x)\vrp^{(k)}(\mb x;\mb a_j) - (\sqrt{n_k}\mb b_j-\mb x)\vrp^{(k)}(\mb x;\mb b_j)}{}\\
    &\le\max_{k\in[K]}\sup_{\mb x\in S^a,\mb u_j\in I(\mb a_j,\mb b_j)} \sum_{j}w_j\nrm{\sqrt{n_k}(\mb a_j-\mb b_j)\vrp^{(k)}(\mb x;\mb a_j) + (\sqrt{n_k}\mb b_j-\mb x)\langle\mb b_j-\mb a_j,\nabla_{\mb u_j}\vrp^{(k)}(\mb x;\mb u_j)\rangle}{} \\
    &\le\zeta\max_{k\in[K]}\brcs{\frac{\sqrt{n_k}}{\us^{d/2}}+\sup_{\mb x\in S^a}\sup_{\mb u_j\in I(\mb a_j,\mb b_j)}(\nrm{\mb x-\sqrt{n_k}\mb u_j}{}+\sqrt{n_k}\nrm{\mb u_j-\mb b_j}{})\cdot\nrm{\nabla_{\mb u_j}\vrp^{(k)}(\mb x;\mb u_j)}{}} \\
    &\le\zeta\max_{k\in[K]}\brcs{\frac{\sqrt{n_k}}{\us^{d/2}}+\frac{\sqrt{n_k}\zeta(\sqrt{n_k}+2d^{1.5}\bs')}{\us^{(d+1)/2}}+\sup_{\mb x\in S^a}\sup_{\mb u_j\in I(\mb a_j,\mb b_j)}\nrm{\mb x-\sqrt{n_k}\mb u_j}{}\cdot\nrm{\nabla_{\mb u_j}\vrp^{(k)}(\mb x;\mb u_j)}{}}
\end{align*}
where $I(\mb a_j,\mb b_j)$ denotes a set of linear interpolations between $\mb a_j,\mb b_j$.
In the third line, we used \eqref{eq:mvt1} and Lemma \ref{lem:mvt}. Plus, for $\mb z_{\mb u_j}:=\Sigma_k(\mb u_j)^{-1/2}(\mb x-\sqrt{n_k}\mb u_j)$, we have
\begin{align*}
    \sup_{\mb x\in S^a}\sup_{\mb u_j\in I(\mb a_j,\mb b_j)}&\nrm{\Sigma_k(\mb u_j)^{1/2}\mb z_{\mb u_j}}{}\cdot\lrnrm{\nabla_{\mb u_j}\frac{\phi_d(\mb z_{\mb u_j})}{|\Sigma_k(\mb u_j)|^{1/2}}}{}\\
    &\le\frac{\bs^{1/2}}{\us^{(d+1)/2}}\sup_{u\ge0}u\brcs{u\lrp{\sqrt{\bn}+\bs'u}+d^{1.5}\bs'}\phi_1(u) \\
    &\le\frac{2\bs^{1/2}}{\sqrt{e}\us^{(d+1)/2}}\lrp{\sqrt{\bn}+\bs'd^{1.5}}.
\end{align*}
So, we obtain
\[\|\mb q_{\blt,H} - \mb q_{\blt,H'}\|_{\infty, S^a}\le\frac{\sqrt{\bn}\zeta}{\us^{d/2}}\brcs{1+\frac{(\zeta+2\bs^{1/2})(1+2d^{1.5}\bs'/\sqrt{\bn})}{\us^{1/2}}}.\]
Now, $\c{D}$ denotes a minimal $\xi$-net of $\Dlt^{l-1}=\{\mb w \in \R_+^l\,;\,\sum_{j=1}^l w_j=1\}$ in the $\ell_1$ norm, and approximate the weights $w$ by their closest element $v\in \c{D}$. Writing $H'' = \sum_{j=1}^l v_j\dlt_{\mb b_j}$,
\begin{align*}
    \|\mb q_{\blt,H'} - \mb q_{\blt,H''}\|_{\infty, S^a}
    &=\max_{k\in[K]}\sup_{\mb x\in S^a}\sum_{j}(w_j-v_j)\nrm{(\sqrt{n_k}\mb b_j-\mb x)\vrp^{(k)}(\mb x;\mb b_j)}{} \\
    &\le\sum_{j}|w_j - v_j|\frac{\bs^{1/2}}{\us^{d/2}}\sup_{u\ge0}u\phi_1(u)\le_d \bs^{1/2}\us^{-d/2} \xi.
\end{align*}
Finally, aggregating the inequalities for large enough $\bn$,
\begin{align*}
    \|\mb q_{\blt,G} - \mb q_{\blt,H''}\|_{\infty, S^a}
    \le_{d,\us,\bs,\bs'} \sqrt{\bn}\bxs{ae^{-\un a^2/(2\bs)}+\frac{1+\xi}{\sqrt{\bn}}+\zeta\lrp{\zeta+1}}.
\end{align*}
Now we set $\xi=\eta$, $\zeta=\eta/\sqrt{\bn}$, and $\eta<1/e$. Then, $a=\sqrt{\frac{2\bs}{\un}\log\frac{1}{\xi}}$ yields the convergence rate $\|\mb q_{\blt,G} - \mb q_{\blt,H''}\|_{\infty, S^a}\le C_{d,\us,\bs,\bs'}\sqrt{\bn}(a+1/\sqrt{\bn})\eta=:\t\eta$ (so $\t\eta\le_\bs\sqrt{\bn}a\eta$).

Next, we had
\[
|\c{C}|\cdot|\c{D}|\le \bxs{\left(1+\frac{2}{\xi}\right)\frac{e N(\zeta, S^a)}{l}}^l,\quad l\ge_d\lrp{\frac{\tau\bn}{\un}}^dN(a, S^a)
\]
and
\begin{align*}
    \frac{N(\zeta, S^a)}{l}
    \le \brcs{\frac{\un}{\tau\bn}\left(1+\frac{a}{\zeta}\right)}^d
    \le \lrp{1+\sqrt{\frac{2\un\bs}{\bn\xi^3}}}^d\underset{\un\to\infty}{\le}\lrp{\sqrt{\frac{C_{d,\bs}}{\eta^3}}}^d.
\end{align*}
So, as $1/\eta>e$,
\begin{align*}
    \log N(\t\eta,\bb{F}_q,&\|\cdot\|_{\infty,S})\le l\brcs{\log\frac{3}{\xi}\frac{eN(\zeta, S^a)}{l}} \\
    &\le_dN(a,S^a)\brcs{\log(1/\xi)}^d\log\lrp{\frac{C_{d,\bs}'}{\eta^{1+3d/2}}}
    \le_dN(a,S^a)\log^{d+1}\frac{C_{d,\bs}'}{\eta}
\end{align*}
hence
\[\log N(\eta,\bb{F}_q,\|\cdot\|_{\infty,S})\le_dN(a,S^a)\log^{d+1}\frac{\sqrt{\bn}aC_{d,\us,\bs,\bs'}'}{\eta}\le N(a,S^a)\log^{d+1}\frac{C_{d,\us,\bs,\bs'}'}{\eta}.\]
\end{proof}

\newpage

\section{Additional algorithms, figures, and tables}

\begin{algorithm}[t]
\DontPrintSemicolon
\KwIn{Local MLEs $\{\mb x_k\}_{k=1}^K$; client covariances $\{\h\Sigma_k\}_{k=1}^K$; component number $L$; outer iterations $T$; learning rate $\eta\in(0,1]$}
\KwOut{Personalized estimates $\{\h{\b\tht}_k\}_{k=1}^K$}
\BlankLine
Fit initial spherical GMM on $\{\h{\b\tht}_k^{(0)}\}$ with $L$ components\;
\For{$t = 0,\ldots,T-1$}{
    Compute responsibilities $r_{k\ell}^{(t)}=\Pr(\ell|\h{\b\tht}_k^{(t)})$ and component scales $\tau_\ell^{2,(t)}$ from the current GMM\;
    Extract diagonal observation precisions:
    $\mb d_k := \diag(\h\Sigma_k)^{-1}$\;
    \For{$k = 1,\ldots,K$}{
        $\alpha_k^{(t)}\gets\sum_{\ell=1}^{L}\frac{r_{k\ell}^{(t)}}{\tau_\ell^{2,(t)}},\qquad
        \mb b_k^{(t)}\gets\sum_{\ell=1}^{L}\frac{r_{k\ell}^{(t)}}{\tau_\ell^{2,(t)}}\,\h\mu_\ell^{(t)}$\;
        MAP solution ($\oslash$ denotes elementwise division):
        $\t{\b\tht}_k^{(t)}
        \gets
        \bigl(\mb d_k\odot \mb x_k + \mb b_k^{(t)}\bigr)\oslash\bigl(\mb d_k + \alpha_k^{(t)}\b1_d\bigr)$\;
        Damped update:
        $\h{\b\tht}_k^{(t+1)}\gets\h{\b\tht}_k^{(t)}+\eta\bigl(\t{\b\tht}_k^{(t)}-\h{\b\tht}_k^{(t)}\bigr)$\;
    }
    Refit spherical GMM on $\{\h{\b\tht}_k^{(t+1)}\}$\;
}
\Return $\{\h{\b\tht}_k^{(T)}\}_{k=1}^K$
\caption{AdaMix Baseline \citep{ozkara2023statistical}}
\label{alg:adamix}
\end{algorithm}

\begin{algorithm}[ht!]
\caption{VANEB-Full: Empirical Bayes Personalization with Full Network Weights}
\label{alg:vaneb_fullnn_alg2e}

\KwIn{Clients $k=1,\dots,K$; rounds $T$; initial model $\bx^{(0)}$; atoms $m$; local epochs $E$; damping $\lambda$}
\KwOut{Personalized models $\{\bx_k^{v,(t)}\}$}

\For{$t=0,\dots,T-1$}{
    Server broadcasts $\bx^{(t)}$\;

    \ForEach{client $k$ \textbf{in parallel}}{
        Train locally for $E$ epochs $\Rightarrow \bx_k$\;
        Compute $\h{\mb F}_k = \frac{1}{n_k}\sum_{i=1}^{n_k} \mb g_{ki} \mb g_{ki}^\top$ using local gradients $\mb g_{ki}$\;
        Set $\h{\Sigma}_k = \diag\big((\h{\mb F}_k + \lambda\mb I)^{-1}\big)$\;
        Send $(\bx_k,\h{\Sigma}_k,n_k)$ to server\;
    }

    Initialize atoms $\{\mb a_j\}_{j=1}^m$ and weights $\pi_j=1/m$\;
    Set $\Sigma_k(\mb a_j) \gets \h{\Sigma}_k$ for all $k,j$,  or learn function through polynomial regression\;

    \For{$s=0,\dots,S-1$}{
        \tcp{E-step}
        \ForEach{$k,j$}{
            $r_{kj} \propto \pi_j \, \vrp(\bx_k; \mb a_j, \Sigma_k(\mb a_j)/n_k)$\;
        }
        Normalize $\{r_{kj}\}_j$\;

        \tcp{M-step}
        \For{$j=1,\dots,m$}{
            $\pi_j \gets \frac{1}{K}\sum_{k=1}^K r_{kj}$\;
            $\mb a_j \gets \left(\sum_{k=1}^K r_{kj} n_k \Sigma_k(\mb a_j)^{-1}\right)^{-1}
            \sum_{k=1}^K r_{kj} n_k \Sigma_k(\mb a_j)^{-1}\bx_k$\;
        }
        Update $\Sigma_k(\mb a_j)$ for all $k,j$\;
    }
    \ForEach{client $k$}{
        $\bx_k^v \gets \sum_{j=1}^m r_{kj} \mb a_j$\;
    }
    $\bx^{(t+1)} \gets \sum_{k=1}^K \frac{n_k}{\sum_\ell n_\ell}\bx_k^v$\;
}
\Return $\{\bx_k^{v,(t)}\}$\;
\end{algorithm}

\section{Remark on AdaMix}
Algorithm \ref{alg:adamix} adapts the AdaMix idea of \citet{ozkara2023statistical} to the summary-statistics setting. There are several distinct aspects that we clarify here.
\begin{itemize}[leftmargin=*]
    \item The original AdaMix operates on local datasets and runs gradient steps on a full personalized objective that includes the local training loss. In Algorithm \ref{alg:adamix} clients have already transmitted low-dimensional MLEs $\{\bx_k\}$, so the MAP fusion replaces the gradient loop entirely.
    \item Because the observation model is Gaussian given (estimated) covariance, the MAP step admits the closed-form coordinate-wise solution.
    This is an instance of the standard Gaussian conjugate posterior \citep{efron2014two}.
    \item The original assumes a single scalar observation variance per client. Algorithm \ref{alg:adamix}'s MAP step uses the diagonal of the client-reported Fisher covariance estimate \citep{jhunjhunwala2024fedfisher}.
    \item The original fits a full Gaussian mixture with component-specific covariances.
    We use a spherical GMM (isotropic per-component variance), which is more regularized and avoids ill-conditioning when $K$ is moderate. Spherical mixtures have been advocated as a practical prior for federated shrinkage in \citet{kotelevskii2022fedpop}.
\end{itemize}
The damped update is a proximal relaxation that prevents the iterates from chasing the moving GMM fit; the step size $\eta$ plays the same stabilizing role as the learning rate in the gradient-based version.

\end{document}